\documentclass[journal]{IEEEtran}
\usepackage{amsmath}
\usepackage{amsthm}
\usepackage{amsfonts}
\usepackage{amssymb}
\usepackage{longtable}
\usepackage{pifont}
\usepackage{graphicx}
\usepackage{caption}
\usepackage{subcaption}
\usepackage[colorlinks,citecolor=blue,urlcolor=blue]{hyperref}
\usepackage[ruled,vlined]{algorithm2e}
\newtheorem{thm}{Theorem}[section]
\newtheorem{defn}{Definition}[section]
\newtheorem{claim}{Step}
\newtheorem{claimproof}{Proof of Step}
\newtheorem{lemma}{Lemma}[section]
\newtheorem{rem}{Remark}
\newcommand{\cmark}{\ding{51}}%
\newcommand{\xmark}{\ding{55}}%

\ifCLASSINFOpdf
\else
\fi
\begin{document}
%
\title{A Strongly Consistent Sparse $k$-means Clustering with Direct $l_1$ Penalization on Variable Weights}
%
%
%

\author{Saptarshi~Chakraborty, 
        Swagatam~Das\\
         {Indian~Statistical~Institute,~ Kolkata,~India}
\thanks{S. Chakraborty is with the Indian Statistical Institute, Kolkata, India, 700108 e-mail: saptarshichakraborty27@gmail.com.}
\thanks{S. Das is with the Electronics and Communication Sciences Unit, Indian Statistical Institute, Kolkata, India, 700108 e-mail: swagatamdas19@yahoo.co.in, swagatam.das@isical.ac.in.}}
\maketitle
\begin{abstract}
We propose the Lasso Weighted $k$-means ($LW$-$k$-means) algorithm as a simple yet efficient sparse clustering procedure for high-dimensional data where the number of features ($p$) can be much larger compared to the number of observations ($n$). In the $LW$-$k$-means algorithm, we introduce a lasso-based penalty term, directly on the feature weights to incorporate feature selection in the framework of sparse clustering. $LW$-$k$-means does not make any distributional assumption of the given dataset and thus, induces a non-parametric method for feature selection. We also analytically investigate the convergence of the underlying optimization procedure in $LW$-$k$-means and establish the strong consistency of our algorithm. $LW$-$k$-means is tested on several real-life and synthetic datasets and through detailed experimental analysis, we find that the performance of the method is highly competitive against some state-of-the-art procedures for clustering and feature selection, not only in terms of clustering accuracy but also with respect to computational time. 
\end{abstract}


\section{Introduction}
\IEEEPARstart{C}{lustering}  is one of the major steps in exploratory data mining and statistical data analysis. It refers to the task of distributing a collection of patterns or data points into more than one non-empty groups or clusters in such a manner that the patterns belonging to the same group may be more identical to each other than those from the other groups \cite{1427769, wong2015short}. The patterns are usually represented by a vector of variables or observations that are also commonly known as features in the pattern recognition community. The notion of a cluster, as well as the number of clusters in a particular data set, can be ambiguous and subjective. However, most of the popular clustering techniques comply with the human conception of clusters and capture a dense patch of points in the feature space as a cluster. Center-based partitional clustering algorithms identify each cluster in terms of a single point called a \textit{centroid} or a \textit{cluster center}, which may or may not be a member of the given dataset. $k$-means \cite{Mac, Jain:2010:DCY:1755267.1755654} is arguably the most popular clustering algorithm in this category. This algorithm separates the data points into $k$ disjoint clusters ($k$ is to be specified beforehand, though) by locally minimizing the total intra-cluster spread i.e. the sum of squares of the distances from each point to the candidate centroids. Obviously, the algorithm starts with a set of randomly initialized candidate centroids from the feature space of the data and attempts to refine them towards the best representatives of each cluster over the iterations by using a local heuristic procedure. $k$-means may be viewed as a special case of the more general model-based clustering \cite{McNicholas2016,Fraley98howmany, McLachlan:2014:NCG:2787261.2787262} where the set of $k$ centroids can be considered as a model from which the data is generated. Generating a data point in this model consists of first selecting a centroid at random and then adding some noise. For a Gaussian distribution of the noise, this procedure will result into hyper-spherical clusters usually. \par

 With the advancement of sensors and hardware technology, it has now become very easy to acquire a vast amount of real data described over several variables or features, thus giving rise to high-dimensional data. For example, images can contain billions of pixels, text and web documents can have several thousand words, microarray datasets can consist of expression levels of thousands of genes. {\it Curse of dimensionality} \cite{bellman1957dynamic} is a term often coined to describe some fundamental problems associated with the high-dimensional data where the number of features $p$ far exceeds the number of observations $n$ ($p \gg n$). With the increase of dimensions, the difference between the distances of the nearest and furthest neighbors of a point fades out, thus making the notion of clusters almost meaningless \cite{beyer1999nearest}. In addition to the problem above, many researchers also concur on the fact that especially for high dimensional data, the meaningful clusters may be present only in subspaces formed with a specific subset of the features available \cite{tsai2008developing,liu2005toward,chen2012feature,de2012minkowski}. 
  Different features can exhibit different degrees of relevance to the underlying groups in a practical data with a high possibility. Generally, the machine learning algorithms employ various strategies to select or discard a number of features to deal with this situation. Using all the available features for cluster analysis (and in general for any pattern recognition task) can make the final clustering solutions less accurate when a considerable number of features are not relevant to some clusters \cite{chan2004optimization}. To add to the difficulty further, the problem of selection of an optimal feature subset with respect to some criteria is known to be NP-hard \cite{blum1989training}. Also, even the degree of contribution of the relevant features can vary differently to the task of demarcating various groups in the data. Feature weighting is often thought of as a generalization of the widely used feature selection procedures \cite{wettschereck1997review,modha2003feature,tsai2008developing,DBLP:journals/corr/Amorim16}. An implicit assumption of the feature selection methods is that all the selected features are equally relevant to the learning task in hand, whereas, feature weighting algorithms do not make such assumption as each of the selected features may have a different degree of relevance to a cluster in the data. To our knowledge, Synthesized Clustering (SYNCLUS) \cite{desarbo1984synthesized}  is the first $k$-means extension to allow feature weights. SYNCLUS partitions the available features into a number of groups and uses different weights for these groups during a conventional $k$-means clustering process. The convex $k$-means algorithm \cite{modha2003feature} is an interesting approach to feature weighting by integrating multiple, heterogeneous feature spaces into the $k$-means framework. Another extension of $k$-Means to support feature weights was introduced in \cite{chan2004optimization}.  Huang \textit{et al.} \cite{huang2005automated}, introduced the celebrated Weighted $k$-means algorithm ($WK$-means) which introduces a new step for updating the feature weights in $k$ means by using a closed-form formula for the weights derived from the current partition. $WK$-means was later extended to support fuzzy clustering \cite{li2006novel} and cluster-dependent weights \cite{huang2008weighting}. Entropy Weighted $k$-means \cite{jing2007entropy}, improved $k$-prototypes \cite{huang1998extensions}, Minkowski Weighted $k$-means \cite{de2012minkowski}, Feature Weight Self-Adjustment $k$-Means \cite{tsai2008developing}, Feature Group Weighted $k$-means ($FG$-$k$-means) \cite{chen2012feature} are among the notable works in this area. A detailed account of these algorithms and their extensions can be found in \cite{DBLP:journals/corr/Amorim16}.\par
    Traditional approaches for feature selection can be broadly categorized into \textit{filter} and \textit{wrapper}-based approaches \cite{dy2008unsupervised,kohavi1997wrappers}. Filter methods use some kind of proxy measure ( just for example,  mutual information, Pearson product-moment correlation coefficient, Relief-based algorithms etc.) to score the selected feature subset during the pre-processing phase of the data. On the other hand, the wrapper approaches employ a predictive learning model to evaluate the candidate feature subsets. Although wrapper methods tend to be more accurate than those following a filter-based approach \cite{dy2008unsupervised}, nevertheless, they incur high computational costs due to the need of executing both a feature selection module and a clustering module several times on the possible feature subsets. \par
  
    Real-world datasets can come with a large number of noise variables, i.e., variables that do not change from cluster to cluster, also implying that the natural groups occurring in the data differ with respect to a small number of variables. Just as an example, only a small fraction of genes (relevant features) contribute to the occurrence of a certain biological activity, while the others in a large fraction, can be irrelevant (noisy features). A good clustering method is expected to identify the relevant features, thus avoiding the derogatory effect of the noisy and irrelevant ones. It is not hard to see that if an algorithm can impose positive weights on the relevant features while assigning exactly zero weights on the noisy ones, the negative influence from the latter class of features can be nullified. Sparse clustering methods closely follow such intuition and aim at partitioning the observations by using only an adaptively selected subset of the available features.
    \subsection{Relation to Prior Works}
    
    Introducing sparsity in clustering is a well studied field of unsupervised learning. Friedman and Meulman \cite{cosa2004} proposed a sparse clustering procedure, called Clustering Objects on Subsets of Attributes (COSA), which in its simplified form, allows different feature weights within a cluster and closely relate to a weighted form of the $k$-means algorithm. Witten and Tibshirani \cite{witten2010framework} observed that COSA hardly results in a truly sparse clustering since, for a positive value of the tuning parameter involved, all the weights retain non-zero value. As a betterment, they proposed the sparse $k$-means algorithm by using the $l_1$ and $l_2$ penalization to incorporate feature selection. The $l_1$ penalty on the weights result in sparsity (making weights of some of the (irrelevant) features $0$) for a small value of a parameter which is tuned by using the Gap Statistic \cite{doi:10.1111/1467-9868.00293}. On the other hand, the $l_2$ penalty is equally important as it causes more than one components of the weight vector to retain non-zero value. Despite its effectiveness, the statistical properties of the sparse $k$-means algorithm including its consistency are yet to be investigated. Unlike the fields of sparse classification and regression, only a few notable extensions on sparse $k$-means emerged subsequently. A regularized version of sparse $k$ means for clustering high dimensional data was proposed in \cite{sun2012}, where the authors also established its asymptotic consistency. Arias-Castro and Pu \cite{ARIASCASTRO2017217} proposed a simple hill climbing approach to optimize the clustering objective in the framework of the sparse $k$ means algorithm. 

A very competitive approach for high dimensional clustering, different from the framework of sparse clustering was taken in \cite{jin2016influential} based on the so-called Influential Feature-based Principal Component Analysis aided with a Higher Criticality based Thresholding (IF-PCA-HCT). 
    This method first selects a small fraction of features with the largest Kolmogorov-Smirnov (KS) scores and then determines the first $k-1$ left singular vectors of the post-selection normalized data matrix. Subsequently, it estimates the clusters by using a classical $k$-means algorithm on these singular vectors. According to \cite{jin2016influential}, the only parameter that needs to be tuned in IF-PCA-HCT is the threshold for the feature selection step. The authors recommended a data-driven rule to set the threshold on the basis of the notion of Higher Criticism (HC) that uses the order statistics of the feature $z$-scores \cite{donoho2008higher}.\par

    Another similar approach known as the IF-PCA algorithm was proposed by Jin {\it et al.} \cite{jin2017phase}. For a threshold $t$ This method clusters the dataset by using the classical PCA to all features whose $l_2$ norm is larger that $t$.
  \begin{table*}
	\centering
	\caption{Some Well Known Algorithms on Feature Weighting and Feature Selection}
	\begin{tabular}{ccccc}
		\label{symbols1}\\
		\hline
		Algorithm/Reference & Feature Weighting & Feature Selection & Model Assumptions & Consistency Proof \\
		\hline
		$k$-means \cite{Mac}& \xmark & \xmark & \xmark & \checkmark\\
		$W$-$k$-means \cite{huang2005automated} & \checkmark & \xmark & \xmark & \checkmark\\
		Pan and Shen \cite{pan2007penalized} &  \xmark & \cmark & \cmark (Mixture model assumption) & \xmark\\
		Sparse-$k$-means \cite{witten2010framework} & \checkmark & \cmark & \xmark & \xmark\\
		IF-HCT-PCA \cite{jin2016influential} & \xmark & \cmark & \cmark (Normality assuption on the irrelevant features) & \cmark\\
		IF-PCA \cite{jin2017phase} &  \xmark & \cmark & \cmark (Normality assuption on the irrelevant features) & \cmark\\
		
		$LW$-$k$-means (The Proposed Method) & \cmark & \cmark & \xmark & \cmark\\
		\hline
	\end{tabular}
	\label{tab history}
\end{table*}
 Pan and Shen \cite{pan2007penalized} proposed the Penalized model-based clustering. This method proposes am EM algorithm to obtain feature selection. Although this method is quite effective, it assumes the likelihood of the data, which can lead to erroneous results if the assumed likelihood is not well suited for the data. This is also the case for IF-HCT-PCA\cite{jin2016influential} and IF-PCA \cite{jin2017phase} as they both assume a Gaussian mixture model for the data. As it can be seen from Section \ref{objective section} that the proposed method does not suffer from this drawback. In contrast to the Sparse $k$-means algorithm \cite{witten2010framework} which uses only $l_1$ and $l_2$ terms in the objective function, our proposed method uses only an $l_1$ penalization and also a $\beta$ exponent in the weight terms, which can lead to more efficient feature selection as seen in Section \ref{fs}. In addition, no obvious relation between the Saprse $k$-means and $LW$-$k$-means is apparent. \par    
 Some theoretical works on sparse clustering can be found in \cite{azizyan2013minimax,jin2017phase,verzelen2017detection}. A minimax theory for highdimensional Gaussian mixture models was proposed by  Azizyan {\it et al.} \cite{azizyan2013minimax}, where the authors derived some precise information theoretic bounds on the clustering accuracy and sample complexity of learning a mixture of two isotropic Gaussians in high dimensions under small mean separation. The minimax rates for the problems of testing and of variable selection under sparsity assumptions on the difference in means were derived in \cite{verzelen2017detection}. 
 The strong consistency of the Reduced $k$-Means (RKM) algorithm \cite{de1994k} the  under i.i.d sampling was recently established by Terada \cite{terada2014strong}.
Following the methods of \cite{pollard1981strong} and \cite{terada2014strong}, the strong consistency of the factorial $k$-means algorithm \cite{vichi2001factorial} was also proved in \cite{terada2015strong}. Gallegos and Ritter \cite{GALLEGOS201314} extended the Pollard's proof of strong consistency \cite{pollard1981strong} for an affine invariant $k$-parameters clustering algorithm. Nikulin \cite{JMLR:v16:nikulin15a} presented proof for the strong consistency of the divisive information-theoretic feature clustering model in probabilistic space with Kullback-Leibler (KL) divergence. Recently, the strong consistency of the Weighted $k$-means algorithm for nearmetric spaces under i.i.d. sampling was proved by Chakraborty and Das \cite{chakraborty2019}. The proof of strong consistency presented in this paper is slightly trickier than the aforementioned papers as we need to choose $\alpha$ and $\lambda$ suitably such that the sum of the weights are bounded almost surely and at least one weight is bounded away from $0$ almost surely.\par
In Table \ref{tab history}, we highlight some of the works in this field along with their important aspects in terms of feature weighting, feature selection, model assumptions and proof of consistency of the algorithms and try to put our proposed algorithm in the context.
    \begin{figure*}[htb]
    \centering
    \begin{subfigure}[b]{0.45\textwidth}
        \includegraphics[width=\textwidth]{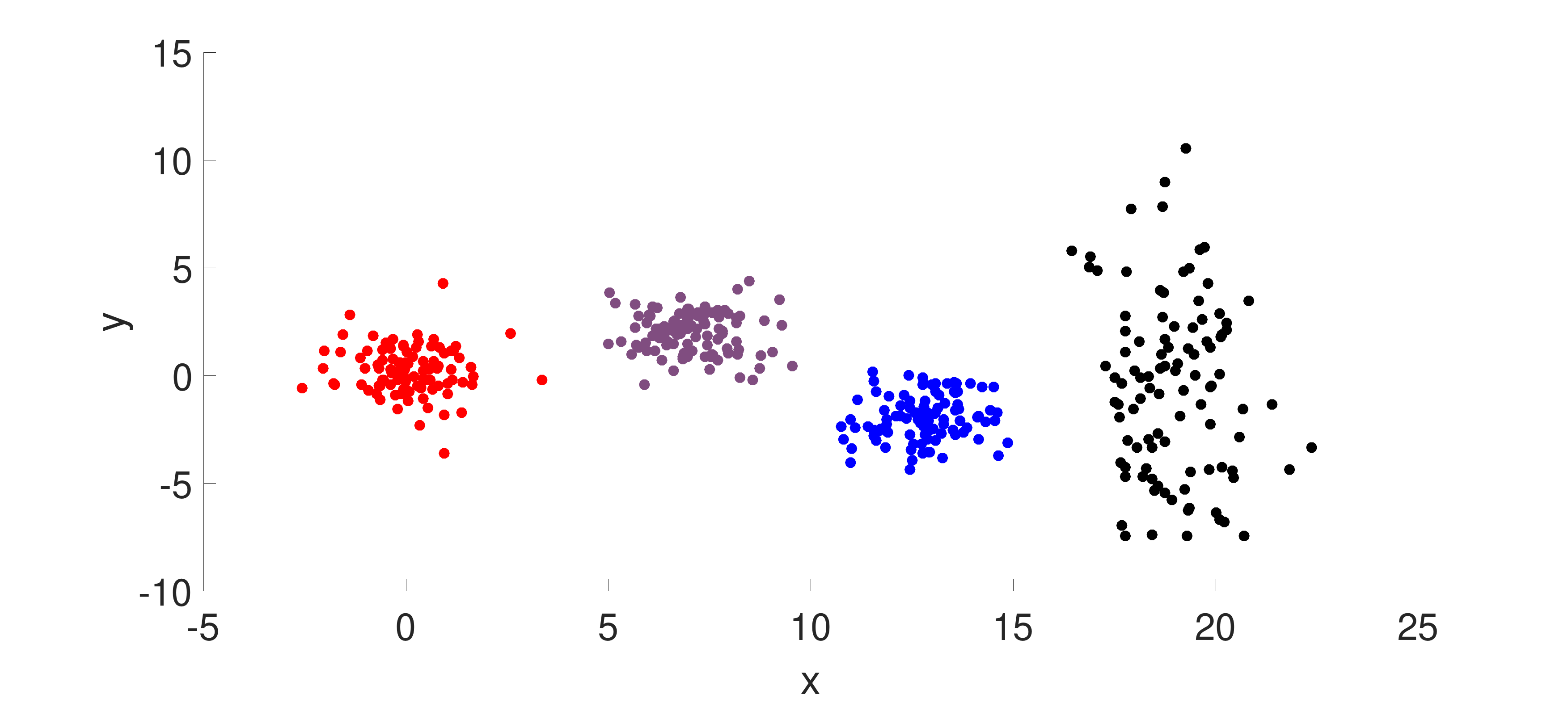}
        \caption{Ground truth}
        \label{cl}
    \end{subfigure}
    ~ 
    \begin{subfigure}[b]{0.45\textwidth}
        \includegraphics[width=\textwidth]{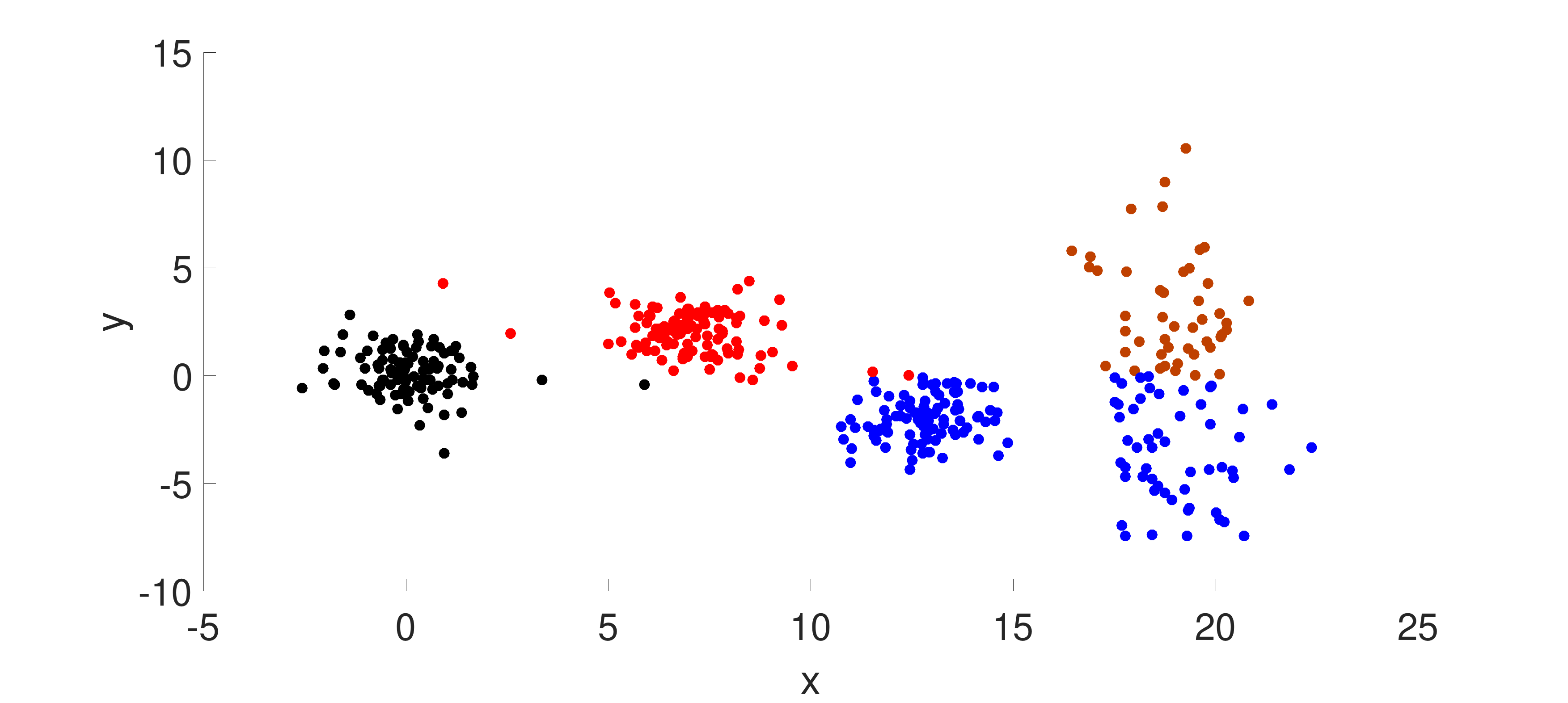}
        \caption{IF-HCT-PCA}
        \label{ifpca}
    \end{subfigure}
    ~ 
    \begin{subfigure}[b]{0.45\textwidth}
        \includegraphics[width=\textwidth]{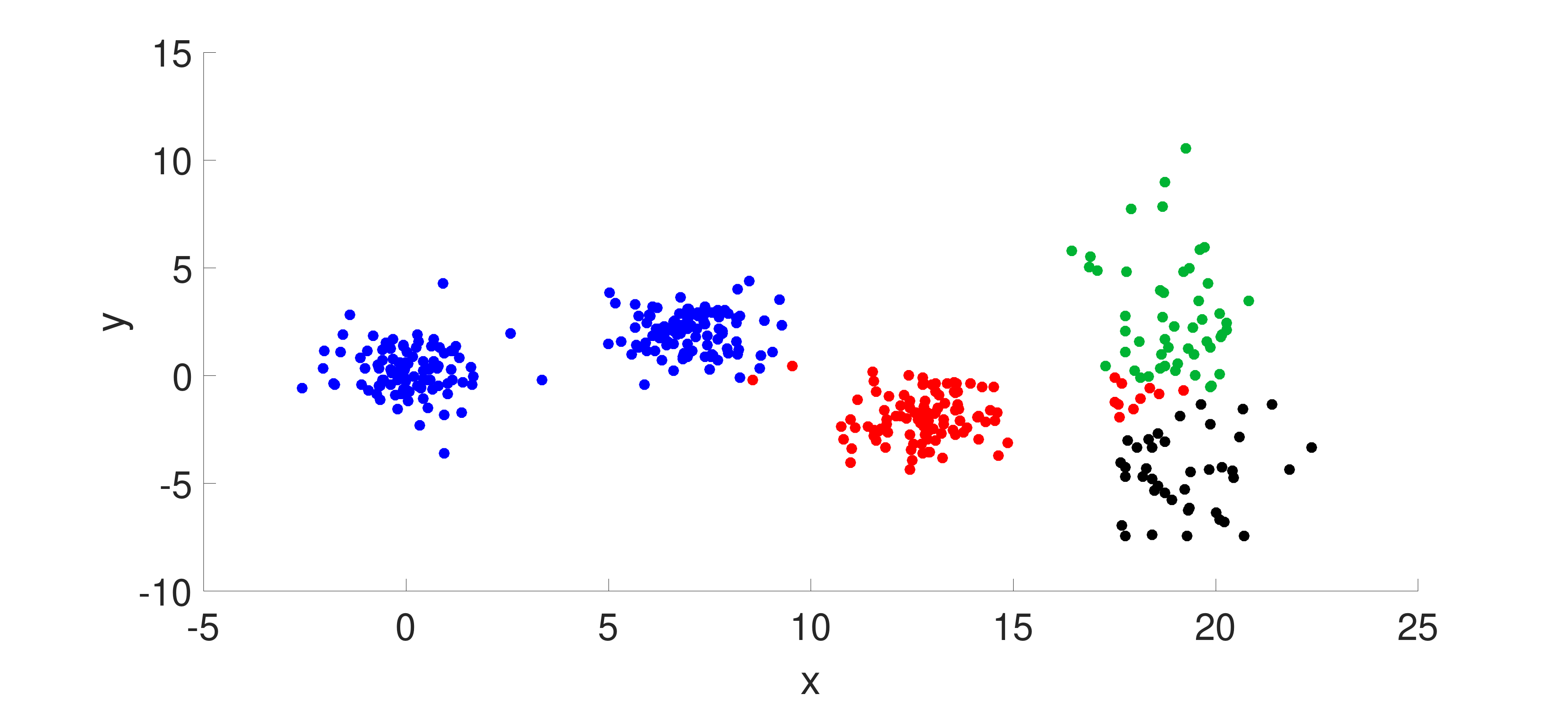}
        \caption{$k$-means}
    \end{subfigure}
    ~ 
    \begin{subfigure}[b]{0.45\textwidth}
        \includegraphics[width=\textwidth]{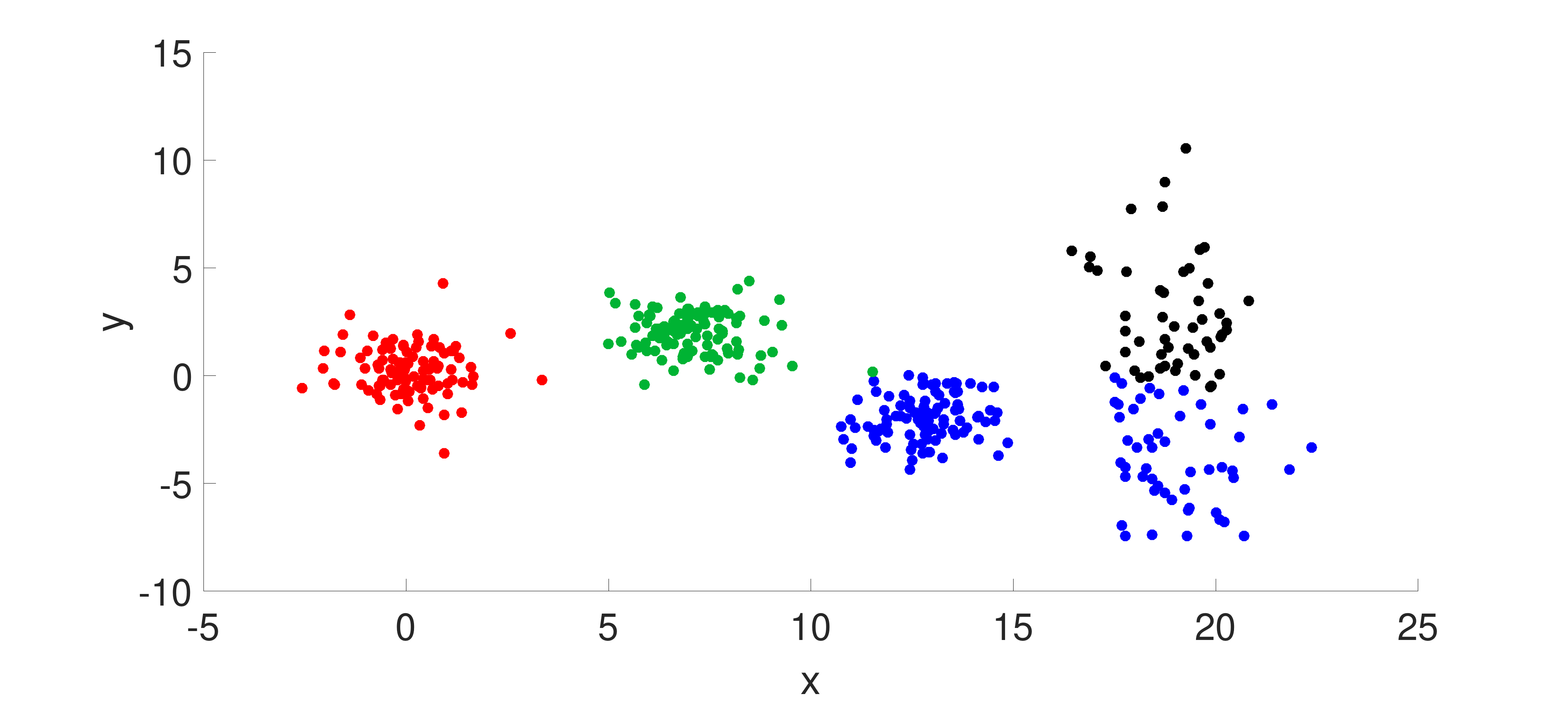}
        \caption{Sparse $k$-means}
    \end{subfigure}
    ~ 
    \begin{subfigure}[b]{0.45\textwidth}
        \includegraphics[width=\textwidth]{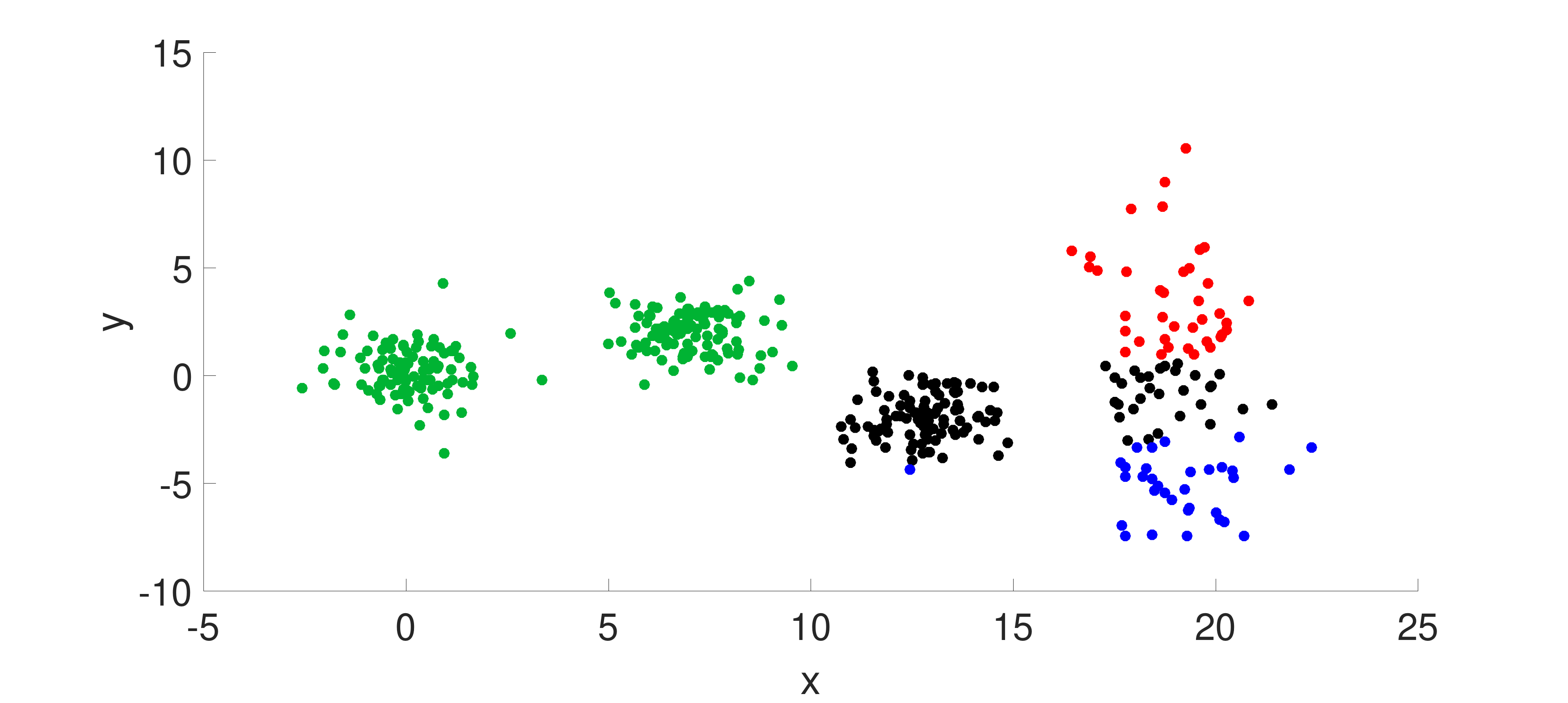}
        \caption{$WK$-means}
    \end{subfigure}
    ~ 
    \begin{subfigure}[b]{0.45\textwidth}
        \includegraphics[width=\textwidth]{images/cl.pdf}
        \caption{ $LW$-$k$-means}
        \label{ifpca}
    \end{subfigure}
    \caption{Ground truth Clustering and Partitioning by different algorithms for $data1$ dataset.}%
    \label{fff}%
\end{figure*}

    \subsection{Summary of Our Contributions}
    We propose a simple sparse clustering framework based on the feature-weighted $k$ means algorithm,  where a Lasso penalty is imposed directly on the feature weights and a closed form solution can be reached for updating the weights. The proposed algorithm, which we will refer to as Lasso Weighted $k$ means ($LW$-$k$-means), does not require the assumption of normality of the irrelevant features as required for the IF-HCT-PCA algorithm \cite{jin2016influential}. We formulate the $LW$-$k$-means as an optimization procedure on an objective function and  derive a block coordinate descent type algorithm \cite{tseng2001convergence} to optimize the objective function in section \ref{opti}. We also prove that the proposed algorithm converges after a finite number of iteration in Theorem \ref{convergence}. We establish the strong consistency of the proposed $LW$-$k$-means algorithm in Theorem \ref{strongth}. Conditions ensuring almost sure convergence of the estimator of $LW$-$k$-means with unboundedly increasing sample size are investigated in section \ref{strong1}. With a detailed experimental analysis, we demonstrate the competitiveness of the proposed algorithm against the baseline $k$-means and  $WK$-means algorithms along with the state-of-the-art sparse $k$-means and IF-HCT-PCA algorithms by using several synthetic as well as challenging real-world datasets with a large number of attributes. Through our experimental results, we observe that not only the $LW$-$k$-means outperforms the other state-of-the-art algorithms, but it does so with considerably less computational time. In section \ref{simulation studies}, we report a simulation study to get an idea about the distribution of the obtained feature weights. The outcomes of the study show that $LW$-$k$-means perfectly identifies the irrelevant features in certain datasets which may deceive some of the state-of-the-art clustering algorithms.  \par
    
       \begin{table}
\centering
\caption{Comparison between $LW$-$k$-means and IF-HCT-PCA}
\label{con}
\begin{tabular}{crrc}
\hline

Algorithm & \multicolumn{2}{c}{Feature Weights}   & Average CER\\
 & $x$ & $y$ \\
 \hline
 $k$-means & 1 & 1 & 0.2657\\\hline
 $WK$-means & 0.5657  & 0.4343 &0.1265\\\hline
 IF-HCT-PCA & 1& 1&0.1475 \\\hline
 Sparse $k$-means & 0.9446&0.3281 &  0.1275\\\hline
 $LW$-$k$-means &0.7587 & 0&0 

\\ \hline
\end{tabular}
\end{table}

    \subsection{A Motivating Example}
Before proceeding further, we take a motivating example to illustrate the efficacy of the $LW$-$k$-means procedure (detailed in Section \ref{lwkalgo})  w.r.t the other peer clustering algorithms by considering a sample toy dataset. In Fig. \ref{cl}, we show the scatter plot of a synthetic dataset $data1$ (the dataset is available at \url{https://github.com/SaptarshiC98/lwk-means}). It is clear that only the $x$-variable contains the cluster structure of the data while the $y$-variable does not. We run five algorithms ($k$-means, $WK$-means, sparse $k$-means, IF-HCT-PCA, and $LW$-$k$-means) on the dataset independently 20 times and report the average CER (Classification Error Rate: proportional to instances misclassified over the whole set of instances) in Table \ref{con}. We also note the average feature weights for each algorithm. From Table \ref{con}, we see that only the $LW$-$k$-means assigns a zero feature weight to feature $y$ and also that it achieves an average CER of 0. The presence of an elongated cluster (colored in black in Fig. \ref{con}) affects the clustering procedure of all the algorithms except $LW$-$k$-means. This elongated cluster, which is non-identically distributed in comparison to the other clusters, increases the Within Sum of Squares (WSS) of the $y$ values, thus increasing its weight. It can be easily seen that for this toy example, the other peer algorithms erroneously detect the $y$ feature to be important for clustering and thus leads to inaccurate clustering. This phenomenon is illustrated in Fig. \ref{con}. 

\section{Background}
\subsection{Some Preliminary Concepts}
In this section we will discuss briefly about the notion of consistency of an estimator. Before we begin, let us recall the defination of convergence in probability and almost surely.
\begin{defn}
Let $(\Omega, \mathcal{F},P)$ be a probability space. A sequence of random variables $\{X_n\}_{n \geq 1}$ is said to converge almost surely (a.s. $[P]$) to a random variable $X$ (in the same probability space),
written as $$X_n \xrightarrow{a.s.} X$$ 
if $P(\{\omega \in \Omega: X_n(\omega) \to X(\omega)\})=1$. 
\end{defn}

\begin{defn}
Let $(\Omega, \mathcal{F},P)$ be a probability space. A sequence of random variables $\{X_n\}_{n \geq 1}$ is said to converge in probability to a random variable $X$ (in the same probability space),
written as $$X_n \xrightarrow{P} X$$ 
if $\forall \epsilon>0$, $\lim_{n \to \infty } P(|X_n-X|> \epsilon)=0$. 
\end{defn}

\subsection{The Setup and Notations}

Before we start, we discuss the meaning of some symbols used throughout the paper in Table \ref{symbols1}.

\begin{table}[htb]
\centering
\caption{Symbols and Their meanings }
\begin{tabular}{cc}
\label{symbols1}
{\it Symbol} & {\it Meaning} \\
\hline
$\mathbb{R}$           & The set of all real numbers            \\
$\mathbb{R}_+$           & The set of all non-negative real numbers            \\
$\mathbb{R}_k^p$ & $\{A\subset \mathbb{R}^p|A \text{ contains k or fewer points }\}$\\
$\mathbb{N}$             & The set of all natural numbers     \\
$\mathcal{S}$           & The set $\{2n : n \in \mathbb{N}\}$  \\
$\mathcal{U}$  & The cluster assignment matrix\\
$\mathcal{Z}$ & The centroid matrix whose rows denote the centroids\\ 
$\mathcal{W}$ & Vector of all the feature weights\\
$\mathcal{N}(\mu,\sigma^2)$ & Normal distribution with mean $\mu$ and variance $\sigma^2$\\
$Unif(a,b)$ & Uniform distribution on the interval $(a,b)$\\
$\chi^2_{d}$ & $\chi^2$ distribution with $d$ degrees of freedom\\
$\mathbf{A}'$ & Transpose of the matrix $\mathbf{A}$\\
$\mathbf{1}$ & Vector $(1,\dots,1)'$ of length n\\
i.i.d & Independent and Identically Distributed\\
i.o. & Infinitely Often\\
a.s. & Almost Surely\\
CER & Classification Error Rate\\
\hline
\end{tabular}
\end{table}
Let $\mathcal{X}=\{\mathbf{x}_1, \mathbf{x}_2, \dots , \mathbf{x}_n\}\subset \mathbb{R}^p$ be a set of $n$ data points which needs to be partitioned into $k$ disjoint and non-empty clusters. Let us also impose $2 \leq k \leq n$ and assume that $k$ is known. 
Let us now recall the definition of a consistent and strongly consistent estimator.
\begin{defn}
An estimator $T_n=T_n(X_1,\dots,X_n)$ is said to be consistent for a parameter $\theta$ if $T_n \xrightarrow{P} \theta$.
\end{defn}

\begin{defn}
An estimator $T_n=T_n(X_1,\dots,X_n)$ is said to be strongly consistent for a parameter $\theta$ if $T_n \xrightarrow{a.s.} \theta$.
\end{defn}
A detailed exposure on consistency can be found in \cite{lehmann2006theory}.

\subsection{$k$-means Algorithm}
The conventional $k$-means clustering problem can be formally stated as a minimization of the following objective function: 

\begin{equation}
\label{obj}
P_{k-means}(\mathcal{U}, \mathcal{Z})=\sum_{i=1}^{n}\sum_{j=1}^{k} \sum_{l=1}^{p} u_{i, l} d(x_{i, j}, z_{j, l}),
\end{equation}
where $\mathcal{U}$ is an $n\times k$ cluster assignment matrix (also called partition matrix), $u_{i,j}$ is binary and $u_{i,j}$ = 1 means data point $\mathbf{x}_i$ belongs to cluster $C_j$. $\mathcal{Z}=[\mathbf{z}_1', \mathbf{z}_2', \dots , \mathbf{z}_k']'$ is a matrix, whose rows represent the  $k$ cluster centers, and $d(\text{ } , \text{ })$ is the distance metric of choice to measure the dissimilarity between two data points. For the widely popular squared Euclidean distance, $d(x_{i, l}, z_{j, l})$=$(x_{i, l}-z_{j, l})^2$. Local minimization of the $k$-means objective function is, most commonly carried out by using a two-step alternating optimization procedure, called the Lloyd's heuristic and recently a performance guarantee of the method in well clusterable situations was established in \cite{Ostrovsky:2013:ELM:2395116.2395117}.
\subsection{$WK$-means Algorithm}
In the well-known Weighted $k$-means ($W$-$k$-means) algorithm by Huang \textit{et al.} \cite{huang2005automated}, the feature weights are also updated along with the cluster centers and the partition matrix within a $k$-means framework. In \cite{huang2005automated}, the authors modified the objective function of $k$-means in the following way to achieve an automated learning of the feature weights: 
\begin{equation}
\label{obj}
P_{Wk-means}(\mathcal{U}, \mathcal{Z}, \mathcal{W})=\sum_{i=1}^{n} \sum_{j=1}^{k} \sum_{l=1}^{p} u_{i, j} w_l^\beta d(x_{i, l}, z_{j, l}),
\end{equation}
where $\mathcal{W}=[w_1, w_2, \dots , w_p]$ is the vector of weights for the $p$ variables, $\sum_{l=1}^{l}w_l$ = 1, and $\beta$ is the exponent of the weights. Huang {\it et al.} \cite{huang2005automated} formulated an alternative optimization based procedure to minimize the objective function with respect to $\mathcal{U}$, $\mathcal{Z}$ and $\mathcal{W}$. The additional step introduced in the $k$-means loop to update the weights use the following closed form upgrade rule: $w_l=\frac{1}{\sum_{t=1}^{p}(\frac{D_l}{D_t})^{\frac{1}{\beta-1}}}$, where $D_l=\sum_{i=1}^{n}\sum_{j=1}^{k}u_{i,j} d(x_{i,l},z_{j,l})$.
\subsection{Sparse $k$-means Algorithm}
Witten and Tibshirani \cite{witten2010framework} proposed the sparse $k$-means clustering algorithm for feature selection during clustering of high-dimensional data. The sparse $k$-means objective function can be formalized in the following way:
\begin{equation}
\label{sparse}
\begin{aligned}
& P_{Sparse \text{ } k-means}(\mathcal{U},\mathcal{W}) \\
&=\sum_{l=1}^p \bigg( w_l\frac{1}{n}\sum_{i=1}^n\sum_{i'=1}^nd(x_{i,l},x_{i',l})\\
&-\sum_{j=1}^k\frac{1}{n_j}\sum_{i=1}^n\sum_{i'=1}^nu_{i,j}u_{i',j}d(x_{i,l},x_{i',l})\bigg).      
\end{aligned}    
\end{equation}

This objective function is optimized w.r.t. $\mathcal{U}$ and $\mathcal{W}$ subject to the constraints, 
$$\|\mathcal{W}\|_2^2 \leq 1\text{,  }\|\mathcal{W}\|_1\leq s \text{ and } w_j \geq 0\textbf{ } \forall j\in \{1,\dots,p\}.$$

\section{The $LW$-$k$-means Objective }
\label{objective section}
The $LW$-$k$-means algorithm is formulated as a minimization problem of the $LW$-$k$-means objective function given by, 
\begin{equation}
\label{eq lwk}
\begin{split}
    &P_{LW-kmeans}(\mathcal{U},\mathcal{Z},\mathcal{W})\\
  &=\frac{1}{n}\sum_{i=1}^n\sum_{j=1}^k\sum_{l=1}^p (w_l^\beta+\frac{\lambda}{p^2} |w_l|)u_{i,j}d(x_{i,l},z_{j,l})-\alpha\sum_{l=1}^pw_l,
\end{split}
\end{equation}

where, $\lambda >0$, $\alpha>0$ and $\beta \in \mathcal{S}$ are fixed parameters chosen by the user.
This objective function is to be minimized w.r.t $\mathcal{U},\mathcal{Z}$, and $\mathcal{W}$ subject to the constraints,
\begin{subequations}
\begin{equation} \label{c1}
 \sum_{j=1}^{k}u_{i,j}=1,
\end{equation}
\begin{equation} \label{c2}
u_{i,j} \in \{1,0\} \forall i \in \{1,\dots n\}, \forall j \in \{1,\dots ,k\},
\end{equation}
\begin{equation} \label{c3}
\mathcal{Z}=[\mathbf{z}_1',\dots ,\mathbf{z}_k']' \text{ } \mathbf{z}_j \in \mathbb{R}^p \text{, } \forall j \in \{1,\dots ,k\},
\end{equation}
\begin{equation} \label{c4}
\mathcal{W}=[w_1,\dots ,w_p]' \text{ } w_l \in \mathbb{R}_+ \text{, } \forall l \in \{1,\dots ,p\}.
\end{equation}

\end{subequations}
In what follows, we discuss the key concept behind the choice of the objective function (\ref{eq lwk}). It is well known that though the $WK$-means algorithm \cite{huang2005automated} is very effective for automated feature weighing, it cannot perform feature selection automatically. Our motivation for introducing the $LW$-$k$-means is to modify the $WK$-means objective function in such a way that it can perform feature selection automatically.  
If we fix $\mathcal{U}$ and $\mathcal{Z}$ and  consider equation (\ref{obj}) only as a function of $\mathcal{W}$, we get,
\begin{equation}
    \label{huang2}
    P(\mathcal{W})=\frac{1}{n}\sum_{i=1}^n\sum_{j=1}^k\sum_{l=1}^p w_l^\beta u_{i,j}d(x_{i,l},z_{j,l})=\frac{1}{n}\sum_{l=1}^pD_l w_l^\beta,
\end{equation}
where, $D_l=\sum_{i=1}^n\sum_{j=1}^ku_{i,j}d(x_{i,l},z_{j,l})$.
The objective function \ref{huang2} is minimized subject to the constraint $\sum_{l=1}^p w_l=1$. This optimization problem is pictorially presented in Fig. \ref{fig:gull}. The blue lines in the figure represent the contour of the objective function. The red line represents the constraint $\sum_{l=1}^p w_l=1$. The point that minimizes the objective function \ref{huang2}, is the point where the red line touches the contours of the objective function. It is clear from the picture and also from the weight update formula in \cite{huang2005automated}, that $w_l$ is strictly positive unless $D_l=0$. Thus, the $WK$-means will assign a weight, however small it may be, to the irrelevant features but will never assign a zero to it. Thus, $WK$-means fails to perform feature selection, where some features need to be completely discarded.\par
Let us try to overcome this difficulty by adding a penalty term. If we add a penalty term $\frac{1}{n}\lambda \sum_{l=1}^p|w_l|$, that will equally penalize all the $w_l$'s regardless of whether the feature is distinguishing or not. Instead of doing that we use the penalty term $\frac{1}{n}\frac{\lambda}{p^2} \sum_{l=1}^pD_l |w_l|$, which will punish those $w_l$'s for which $D_l$'s are larger. Here $p^2$ is just a normalizing constant. Thus if we use this penalty term, the objective function becomes

{\small
\begin{equation}
    \label{huang3}
    P(\mathcal{W})=\frac{1}{n}\sum_{l=1}^pD_l w_l^\beta+\frac{1}{n}\frac{\lambda}{p^2} \sum_{l=1}^pD_l |w_l|=\sum_{l=1}^p (w_l^\beta+\frac{\lambda}{p^2}|w_l|) D_l. 
\end{equation}
}%
Apart from having the objective function \ref{huang3}, we do not want that the sum of the weights should deviate too much from $1$. Thus we substract a penalty term $\alpha(\sum_{l=1}^pw_l-1)$, where, $\alpha>0$ and the objective function becomes,
\begin{equation}
    \label{huang4}
    P(\mathcal{W})
    =\frac{1}{n}\sum_{l=1}^p (w_l^\beta+\frac{\lambda}{p^2}|w_l|) D_l-\alpha\sum_{l=1}^pw_l+\alpha.
\end{equation}
Since, $\alpha>0$ is a constant, minimizing \ref{huang4} is same as minimizing
\begin{equation}
    \label{huang5}
    P(\mathcal{W})=\frac{1}{n}\sum_{l=1}^p (w_l^\beta+\frac{\lambda}{p^2}|w_l|) D_l-\alpha\sum_{l=1}^pw_l
\end{equation}
w.r.t $\mathcal{W}$. Since $\lambda$ is a constant, we change the objective function to \ref{eq lwk}.
In Fig. \ref{fig:tiger}, we show the contour plot of the objective function \ref{eq lwk} with $p=2$, $\beta=2$, $D_1=1$, $D_2=3$, $\lambda=1$, $\alpha=0.66$. Clearly the minimization of the objective function occurs on the $x$-axis. Hence the $LW$-$k$-means algorithm can set some of the feature weights to $0$ and thus can perform feature selection.
\begin{rm}
The $w_l^\beta$ term provides an additional degree of non-linearity to the $LW$-$k$-means objective function. Also notice that for $\beta=2$, though sparse $k$-means objective function (\ref{sparse}) and the $LW$-$k$-means objective function uses the same term, they are not similar at all. The optimal value for the weight for a given set of cluster centroids for sparse $k$-means algorithm does not have a closed form expression but for $LW$-$k$-means, we can find a closed form expression (section \ref{opti}), which can be used for hypothesis testing purposes for model based clustering. \par
In addition, we note the difference between the Regularized $k$-means \cite{sun2012regularized} and $LW$-$k$-means. The former uses a penalization on the centroids of each cluster for feature selection but the later uses the whole dataset for the same purpose. Since a cluster centroid determined by the underlying $k$-means procedure may not be the actual representative of a whole cluster, using penalization only on the cluster centroids may lead to improper feature selection due to grater loss of information about the naturally occurring groups in the data. 
\end{rm}

    
\begin{figure}
    \centering
    \begin{subfigure}[b]{0.24\textwidth}
        \includegraphics[width=\textwidth]{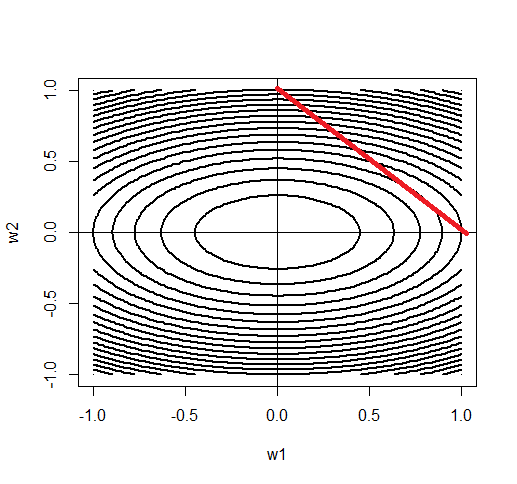}
        \caption{Optimization in $WK$-means}
        \label{fig:gull}
    \end{subfigure}
    \begin{subfigure}[b]{0.24\textwidth}
        \includegraphics[width=\textwidth]{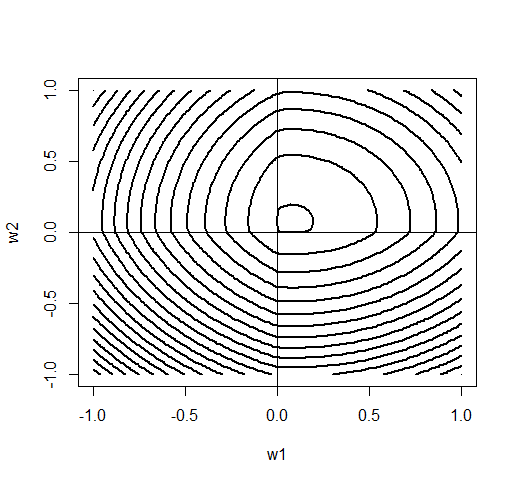}
        \caption{Optimization in $LW$-$k$-means}
        \label{fig:tiger}
    \end{subfigure}
    \caption{Contour plot of the objective functions for $WK$-means and $LW$-$k$-means. }\label{fig:animals}
\end{figure}
\section{The Lasso Weighted $k$-means Algorithm and its COnvergence}
 \label{opti}

 We can minimize \ref{eq lwk} by solving the following three minimization problems. 
\begin{itemize}
    \item \textit{Problem $\mathbf{P_1}$:} Fix $\mathcal{Z}=\mathcal{Z}_0$, $\mathcal{W}=\mathcal{W}_0$, minimize $P(\mathcal{U},\mathcal{Z}_0,\mathcal{W}_0)$ w.r.t $\mathcal{U}$ subject to the constraints \ref{c1} and \ref{c2}.
    \item \textit{Problem $\mathbf{P_2}$:} Fix $\mathcal{U}=\mathcal{U}_0$, $\mathcal{W}=\mathcal{W}_0$, minimize $P(\mathcal{U}_0,\mathcal{Z},\mathcal{W}_0)$ w.r.t $\mathcal{Z}$.
    \item \textit{Problem $\mathbf{P_3}$:} Fix $\mathcal{Z}=\mathcal{Z}_0$, $\mathcal{U}=\mathcal{U}_0$, minimize $P(\mathcal{U}_0,\mathcal{Z}_0,\mathcal{W})$ w.r.t $\mathcal{W}$.
\end{itemize}
It is easily seen that \textit{Problem $\mathbf{P_1}$} can be solved by assigning
\[\hspace{-1cm}
    u_{i ,j}=
    \begin{cases}
    1, & \hspace{-7cm}\text{if } \sum_{l=1}^p (w_l^\beta+\frac{\lambda}{p^2} |w_l|) d(x_{i, l}, z_{j, l}), \\ \hspace{.3cm}\leq \sum_{l=1}^p (w_l^\beta+\frac{\lambda}{p^2} |w_l|) d(x_{i, l}, z_{t, l}) , 1\leq t \leq k,\\
    0,               & \hspace{-7cm}\text{otherwise.}
    \end{cases}
    \]
\textit{Problem $\mathbf{P_2}$} can also be easily solved by assigning,
$$z_{i, j}=\frac{\sum_{i=1}^{n}u_{i, l} x_{i, j}}{\sum_{i=1}^{n}u_{i, l}}.$$
Let, $D_l=\sum_{i=1}^n\sum_{j=1}^ku_{i,j}d(x_{i,l},z_{j,l})$. Hence \textit{Problem $\mathbf{P_3}$} can be stated in the following way.
Let $D_l^0$ denote the value of $D_l$ at $\mathcal{Z}=\mathcal{Z}_0$ and $\mathcal{U}=\mathcal{U}_0$. We note that the objective function can now be written as,

\begin{equation}
\label{new obj}
 P(\mathcal{W})=\frac{1}{n}\sum_{l=1}^p (w_l^\beta+\frac{\lambda}{p^2} |w_l|)D_l^0-\alpha\sum_{l=1}^pw_l. 
\end{equation}
 Now, for solving \textit{Problem $\mathbf{P_3}$}, we note the following. 
\begin{thm}
\label{convex}
The objective function $P(\mathcal{W})$ in \ref{new obj} is convex in $\mathcal{W}$.
\end{thm}
\begin{proof}
See Appendix \ref{convex proof}.
\end{proof}
Now let us solve \textit{Problem $\mathbf{P_3}$} for the case $p=1$. For this, we construct an equivalent problem as follows.\par
\begin{thm}
\label{reformulation}
Suppose $w \in \mathbb{R}$, $D>0$, $\alpha \geq 0$, $\lambda\geq 0$, $\beta \in \mathcal{S}$ be scalars. Consider the following single-dimensional optimization problem $P_1^*$,
\begin{equation}
\label{eq1}
min_{w} \frac{1}{n}w^{\beta} D-\alpha w +\frac{\lambda}{np^2} |w| D.  
\end{equation}
Let $w_1^*$ be a solution to \ref{eq1}. Consider another single-dimensional optimization problem $P_2^*$
\begin{equation}
\label{eq2}
min_{w}\frac{1}{n} w^{\beta} D-\alpha w +\frac{\lambda}{np^2} t D,   
\end{equation} 
subject to 
\begin{equation}
\label{eq3}
t-w \geq 0,
\end{equation}
\begin{equation}
\label{eq4}
t+w \geq 0.
\end{equation}
Suppose $(w^*,t^*)$ be a solution of problem $P_2^*$. Then, $w^*=w_1^*.$
\end{thm}
\begin{proof}
See Appendix \ref{reformulation proof}.
\end{proof}
Before we solve problem $P_2^*$, we consider the following definition.
\begin{defn}
\label{S}
For scalars $x$ and $y \geq 0$, the function $S(\cdot,\cdot)$ is defined as, 
\[
    S(x,y)=
    \begin{cases}
    x-y, & \text{if } x>y,\\
    x+y,               & \text{if } x<-y,\\
    0, & \text{otherwise.}
    \end{cases}
    \]
\end{defn}
We now solve problem $P_2^*$ of Theorem \ref{reformulation} by using Theorem \ref{solve alt}.
\begin{thm}
\label{solve alt}
Consider the 1-D optimization problem $P_2^*$ of Theorem \ref{reformulation}. Let $D>0$ and $(w^*,t^*)$ be a solution to problem $P_2^*$. Then $w^*$ is given by,
$$w^*=\Bigg[\frac{1}{\beta}S(\frac{n\alpha}{D},\frac{\lambda}{p^2})\Bigg]^{\frac{1}{\beta-1}}.$$
\end{thm}
\begin{proof}
See Appendix \ref{solve alt proof}.
\end{proof}

In Theorem \ref{reformulation}, we showed the equivalence of problems $P_1^*$ and $P_2^*$ and in Theorem \ref{solve alt}, we solved problem $P_2^*$. Hence combining the results of Theorems \ref{reformulation} and \ref{solve alt}, we have the following theorem.
\begin{thm}
\label{solve 1 d}
Suppose $w \in \mathbb{R}$, $D>0$, $\alpha \geq 0$, $\lambda \geq 0$, $\beta\in \mathcal{S}$ be scalars. Consider the following single-dimensional optimization problem,
\begin{equation}
\label{2.4.1}
min_{w} \frac{1}{n} w^{\beta} D-\alpha w +\frac{\lambda}{np^2} |w| D .
\end{equation}
 Then a solution to this problem exists, is unique and is given by,
 $$w^*=\Bigg[\frac{1}{\beta}S(\frac{n\alpha}{D},\frac{\lambda}{p^2})\Bigg]^{\frac{1}{\beta-1}}.$$
\end{thm}
\begin{proof}
The result follows trivially from Theorems \ref{reformulation} and \ref{solve alt}.
\end{proof}

We are now ready to prove Theorem \ref{solve p3}, which essentially gives us the solution to \textit{Problem $\mathbf{P_3}$}.
\begin{thm}
\label{solve p3}
Let $\lambda \geq 0$, $\alpha \geq 0$, $D_l>0$ for all $d \in \{1,\dots ,p\}$ be scalars. Also let, $\beta \in \mathcal{S}$ and $p \in \mathbb{N}$. If $\mathcal{W} \in \mathbb{R}^p$ , then solution to the problem

$$ minimize_{\mathcal{W} \in \mathbb{R}^p} \text{ } P(\mathcal{W})=\frac{1}{n}\sum_{l=1}^p (w_l^\beta+\frac{\lambda}{p^2} |w_l|)D_l^0-\alpha\sum_{l=1}^pw_l$$  
exists, is unique and is given by,
$$w_l^*=\Bigg[\frac{1}{\beta}S(\frac{n\alpha}{D_l},\frac{\lambda}{p^2})\Bigg]^{\frac{1}{\beta-1}} \text{ } \forall
l \in\{1,\dots ,p\}.$$
\begin{proof}
See Appendix \ref{solve p3 proof}.
\end{proof}

\end{thm}
\label{lwkalgo}

Algorithm \ref{alg1} gives a formal description of the $LW$-$k$-means algorithm.

\begin{algorithm}
\SetAlgoLined
 \KwData{$\mathcal{X}$, $k$, $\frac{\lambda}{p^2}$, $\epsilon$}
 \KwResult{$\mathcal{U}$, $\mathcal{Z}$, $\mathcal{W}$}
 \textbf{initialization}: 
Randomly pick $k$ datapoints $\mathbf{x}_1,\dots,\mathbf{x}_k$ from $\{\mathbf{x}_1,\dots,\mathbf{x}_n\}.$ \\
Set $\mathcal{Z}=[\mathbf{x}_1,\dots,\mathbf{x}_k]'$\\
   $\mathcal{W}=[\frac{1}{p},\dots ,\frac{1}{p}]$.\\
   $P_1=0$\\
   $P_2$= A very large value\\
   \While{$|P_1-P_2|>\epsilon$}{
     $P_1=\frac{1}{n}\sum_{i=1}^n\sum_{j=1}^k\sum_{l=1}^p (w_l^\beta+\frac{\lambda}{p^2}|w_l|)u_{i,j}d(x_{i,l},z_{j,l})-\alpha\sum_{l=1}^pw_l,$\\
     Update $\mathcal{Z}$ by 
    $$z_{i, j}=\frac{\sum_{i=1}^{n}u_{i, l} x_{i, j}}{\sum_{i=1}^{n}u_{i, l}}$$
     Update $\mathcal{W}$ by\\
              \[
    w_l=
    \begin{cases}
    0& \text{if } D_l=0,\\
 \Bigg[\frac{1}{\beta}S(\frac{n\alpha}{D_l},\frac{\lambda}{p^2})\Bigg]^{\frac{1}{\beta-1}} & \text{otherwise,}
    \end{cases}
    \]\\
    where $D_l=\sum_{i=1}^{n}\sum_{j=1}^{k}u_{i, j}d(x_{i, l}, z_{j,l})$.\\
   
        Update $\mathbf{U}$ by 
    \[\small
    u_{i, j}=
    \begin{cases}
    1, & \text{if } \sum_{l=1}^p (w_l^\beta+\frac{\lambda}{p^2} |w_l|) d(x_{i, l}, z_{j, l})\\ 
    &  \leq \sum_{l=1}^p (w_l^\beta+\frac{\lambda}{p^2} |w_l|) d(x_{i, l}, z_{t, l}) , 1\leq t \leq k\\
    0,               & \text{otherwise.} 
    \end{cases}
    \]
    $P_2=\frac{1}{n}\sum_{i=1}^n\sum_{j=1}^k\sum_{l=1}^p (w_l^\beta+\frac{\lambda}{p^2} |w_l|)u_{i,j}d(x_{i,l},z_{j,l})-\alpha\sum_{l=1}^pw_l$
}

\caption{The $LW$-$k$-means Algorithm}
\label{alg1}
\end{algorithm}

We now prove the convergence of the iterative steps in the $LW$-$k$-means algorithm. This result is proved in the following theorem.
The proof of convergence of the $LW$-$k$-means can be directly derived from \cite{tseng2001convergence}. We only state the result in Theorem \ref{convergence}. The proof of this result is given in Appendix \ref{convergence proof}.
\begin{thm}
\label{convergence}
The $LW$-$k$-means algorithm converges after a finite number of iterations.
\end{thm}

\section{Strong Consistency of the $LW$-$k$-means Algorithm}
\label{strong}
In this section, we will prove a strong consistency result pertaining to the $LW$-$k$-means algorithm. Our proof of strong consistency result is slightly trickier than that of Pollard \cite{pollard1981strong} in the sense that we have to deal with the weight terms which may not be bounded. We first prove the existence of an $\alpha$, which depends on the datasets itself (Theorem \ref{bddwt}), such that $\exists \lambda_0$ for which, we can find an $l$ such that $w_l^{(n)}>C(P)>0$ $\forall 0<\lambda<\lambda_0$ (Theorem \ref{lambda}). In Theorem \ref{strongth}, we prove the main result pertaining to the strong consistency of proposed algorithm.
Throughout this section, we will assume that $d(x,y)=(x-y)^2$, i.e. the distance used is the squared Euclidean distance. We will also assume that the underlying distribution has a finite second moment.
\subsection{The Strong Consistency Theorem}
\label{strong1}
In this section we prove the strong consistency of the proposed method for the following setup. Let $\mathbf{X}_1$,\dots ,$\mathbf{X}_n$ be independent random variables with a common distribution $P$ on $\mathbb{R}^p$.
\begin{rem}
\normalfont
$P$ can be thought of a mixture distribution in the context of clustering but this assuption is not necessary for the proof.
\end{rem}
Let $P_n$ denote the empirical measure based on $\mathbf{X}_1$,\dots ,$\mathbf{X}_n$. 
For each measure $Q$ on $\mathbb{R}^p$, each $\mathcal{W} \in \mathbb{R}^p$ and each finite subset $A$ of $\mathbb{R}^p$, define
\begin{align*}
    \Phi(\mathcal{W},A,Q)&:=\int min_{a \in A} \sum_{l=1}^p (w_l^\beta+\frac{\lambda}{p^2}|w_l|)(x_l-a_l)^2 Q(dx)\\
    &-\alpha(Q)\sum_{l=1}^p w_l
\end{align*}

and 
$$m_k(Q):=inf\{\Phi(\mathcal{W},A,Q)| A \in \mathbb{R}_k^p, \mathcal{W} \in \mathbb{R}^p\}.$$
Here $\alpha(Q)$ is a functional. $\alpha(Q)$  and $\lambda$ are chosen as in Theorems \ref{bddwt} and \ref{lambda}.
For a given $k$, let $A_n$ and $\mathcal{W}_n$ denote the optimal sample clusters and weights respectively, i.e. $\Phi(\mathcal{W}_n,A_n,P_n)=m_k(P_n)$. The optimal population cluster centroids and weights are denoted by $\bar{A}(k)$ and $\bar{\mathcal{W}}(k)$ respectively and they satisfy the relation, $\Phi(\bar{\mathcal{W}}(k),\bar{A}(k),P)=m_k(P)$. Our aim is to show $A_n \xrightarrow{a.s.} \bar{A}(k)$ and $\mathcal{W}_n \xrightarrow{a.s.} \bar{\mathcal{W}}(k)$.\par

\begin{thm}
\label{bddwt}
There exists at least one fuctional $\alpha(\cdot)$ such that $\mathbf{1}'\mathcal{W}_n \leq 1$. Moreover $\alpha(P_n) \xrightarrow{a.s.} \alpha(P)$.
\end{thm}
\begin{proof}
See appendix \ref{bddwt proof}.
\end{proof}

\begin{rem}
\normalfont
One can choose $\alpha$ as follows.
\begin{itemize}
    \item Run the $k$-means algorithm on the entire dataset. Let $U$ and $Z$ be the correspong cluster assignment matrix and the set of centroids respectively.
    \item Choose $\alpha_n(P_n)=\frac{1}{\Bigg(\sum_{l=1}^p [\frac{n}{\beta D_l}]^{\frac{1}{\beta-1}}\Bigg)^{\beta-1}}.$
\end{itemize}
\end{rem}

\begin{rem}
\normalfont
Note that if one chooses $\alpha$ to be constant, then $\sum_{l=1}^p w_l^{(n)}$ will be bounded above by $\alpha^{\frac{1}{\beta-1}} \sum_{l=1}^p [\frac{n}{\beta D_l}]^{\frac{1}{\beta-1}}$, which converges almost surely to a constant by \cite{pollard1981strong}. In what follows, we only require the $w_l^{(n)}$ terms to be almost surely bounded by a positive constant. That requirement is also satisfied if we choose any positive constant $\alpha>0$.
\end{rem}

\begin{thm}
\label{d'}
Let $\mathcal{U}^*_n$ and $D^*_l$ have the same meaning as in the proof of Theorem \ref{bddwt}. Let $\bar{x_l}=\int x_l P_n(d\mathbf{x})$ denote the mean of the $j^{th}$ feature and $D_l=\sum_{i=1}^n d(x_{i,l},\bar{x_l})$. Then $\exists d' \in \{1,\dots,p\}$ such that $D^*_{d'} \leq D_{d'}$. 
\end{thm}
\begin{proof}
We prove the theorem using contradiction. Assuming the contrary, suppose, $D^*_l > D_l$ $\forall d \in \{1,\dots,p\}$. Then,
\begin{align*}
m_k(P_n)&=\frac{1}{n}\sum_{l=1}^p({w^{(n)}_l}^\beta+\frac{\lambda}{p^2}|w_l|) D^*_l-\alpha \sum_{l=1}^p w^{(n)}_l\\
& >\frac{1}{n}\sum_{l=1}^p({w^{(n)}_l}^{\beta}+\frac{\lambda}{p^2}|w_l|) D_l-\alpha \sum_{l=1}^p {w^{(n)}}_l,    
\end{align*}

which is a contradiction since $m_k(Q):=inf\{\Phi(\mathcal{W},A,Q)| A \in \mathbb{R}_k^p, \mathcal{W} \in \mathbb{R}^p\}.$

\end{proof}
\begin{rem}
\normalfont
The following theorem illustrate that if $\lambda$ is chosen inside the range $(0,\lambda_0)$, at least one feature weight is bounded below by a positive constant almost surely. This positive constant depends only on the underlying distribution and is thus denoted by $C(P)$. ALso note that $\lambda_0$ depends on the underlying distribution of the datapoints.
\end{rem}
\begin{thm}
\label{lambda}
There exists a constant $\lambda_0>0$ and $d' \in \{1,\dots,p\}$ such that $\forall 0<\lambda<\lambda_0$, $w^{(n)}_{d'} \geq c(P) >0$ almost surely.
\end{thm}
\begin{proof}
Let $\bar{x_l}=\int x_l P_n(d\mathbf{x})$ denote the mean of the $j^{th}$ feature. 
Let $D_l=\sum_{i=1}^n d(x_{i,l},\bar{x_l})$. $\mathcal{U}^*_n$ and $D^*_l$ have the same meaning as in the proof of Theorem \ref{bddwt}. Choose $d'$ as in Theorem \ref{d'}. Thus, $D^*_{d'} \leq D_{d'}$. Thus, $\frac{\alpha(P_n)n}{D^*_{d'}} \geq \frac{\alpha(P_n)n}{D_{d'}}$. By the assumption of finite second moment, $\frac{1}{n}D_{d'} \xrightarrow{a.s.} \sigma_{d'^2}$, where, $\sigma^2_{d'}=(E[(\mathbf{X}-E(\mathbf{X}))(\mathbf{X}-E(\mathbf{X})')])_{d'd'}$ is the population variance of the $d'-th$ feature. Here $\mathbf{X}$ is any random variable having distribution $P.$ Again, $\alpha(P_n) \xrightarrow{a.s.} \alpha(P)$ (by Theorem \ref{bddwt}). Since, $S(x,y)$ is a continuous function in $x$, $w^{(n)}_{d'}=\Bigg[\frac{1}{\beta}S\Bigg(\frac{n \alpha(P_n)}{D^*_{d'}},\frac{\lambda}{p^2}\Bigg)\Bigg]^{\frac{1}{\beta-1}}\geq\Bigg[\frac{1}{\beta}S\Bigg(\frac{n \alpha(P_n)}{D_{d'}},\frac{\lambda}{p^2}\Bigg)\Bigg]^{\frac{1}{\beta-1}}\xrightarrow{a.s.}\Bigg[\frac{1}{\beta}S\Bigg(\frac{\alpha(P)}{\sigma_{d'}^2},\frac{\lambda}{p^2}\Bigg)\Bigg]^{\frac{1}{\beta-1}}$.
We can choose $\lambda_0=\frac{1}{4}\frac{\alpha(P)p^2}{\sigma_{d'}^2}$ and $C(P)=\frac{1}{2}\Bigg[\frac{1}{\beta}S\Bigg(\frac{\alpha(P)}{\sigma_{d'}^2},\frac{\lambda}{p^2}\Bigg)\Bigg]^{\frac{1}{\beta-1}}$. Thus $C(P)>0$ $\forall 0<\lambda<\lambda_0$.
\end{proof}
We are now ready to prove the main result of this section, i.e. the consistency theorem. The theorem essentially implies that if $\alpha$ and $\lambda$ are suitably chosen, the set of optimal cluster centroids and the optimal weights tends to the popoulation optima in an almost sure sense. 

\begin{thm}
\label{strongth}
Suppose that $\int \|\mathbf{x}\|^2P(dx)< \infty$ and for each $j=1,\dots,k$, there is a unique set $\bar{A}(j)$ and a unique $\bar{\mathcal{W}}(j)\in \mathbb{R}^p$ such that $\Phi(\bar{\mathcal{W}}(j),\bar{A}(j),P)=m_j(P)$ and $\alpha$ and $\lambda$ are chosen accoring to Theorem \ref{bddwt} and \ref{lambda} respectively. Then $A_n\xrightarrow{a.s.}\bar{A}(k)$ and $\mathcal{W}_n\xrightarrow{a.s.}\bar{\mathcal{W}}(k)$. Moreover, $\Phi(\mathcal{W}_n,A_n,P_n)\xrightarrow{a.s.}\Phi(\bar{\mathcal{W}}(k),\bar{A}(k),P)$. 
\end{thm}

\begin{proof}

We will prove the theorem using the following steps.

\begin{claim}

\label{one}
There exists $M>0$ such that $B(M)$ contains at least one point of $A_n$  almost surely, i.e. there exists $M>0$  such that $$P(\cup_{n=1}^\infty \cap_{m=n}^\infty \{\omega:A_m(\omega) \cap B(M) \neq \emptyset \})=1.$$
\end{claim}
\begin{claimproof}
\normalfont

Let $r>0$ be such that $B(r)$ has a positive $P$-measure. By our assumptions, $\Phi(\mathcal{W}_n,A_n,P_n) \leq \Phi(\mathbf{1}_n,A_0,P_n)$ for any set $A_0$ containing atmost $k$ points. Choose $A_0=\{\mathbf{0}\}$. Then, 
$$\Phi(\mathbf{1}_n,A_0,P_n)=\Bigg(1+\frac{\lambda}{p^2}\Bigg)\int \|\mathbf{x}\|_2^2 P_n(d\mathbf{x})-p\alpha(P_n).$$
Thus, 
$$\Phi(\mathbf{1}_n,A_0,P_n) \xrightarrow{a.s.} \Bigg(1+\frac{\lambda}{p^2}\Bigg)\int \|\mathbf{x}\|_2^2 P(d\mathbf{x})-p\alpha(P).$$
Let $\Omega'=\{\omega \in \Omega: \forall n \in \mathbb{N}, \exists m \geq n \text{ s.t. } A_m \cap B(M) = \emptyset\}$. By the Axiom of Choice \cite{jech2008axiom}, for any $\omega \in \Omega'$, there exists a sequence $\{n_h\}_{h \in \mathbb{N}}$ such that $n_i < n_j$ for $i<j$ and $A_{n_h} \cap B_M(\mathbf{x}_0) = \emptyset$.
Now, for this sequence,
\begin{align*}
& \limsup_h\Phi(\mathcal{W}_{n_h},A_{n_h},P_{n_h}) \\ & \geq \lim_h [C(P)^\beta+\frac{\lambda}{p^2}C(P)] (M-r)^2 P_{n_h}(B(r))\\
& -\liminf_{h} \alpha(P_{n_h}) \sum_{l=1}^pw_l^{(n_h)}.    
\end{align*}

Thus, 
$\limsup_{n_h}\Phi(\mathcal{W}_{n_h},A_{n_h},P_{n_h}) \geq [C(P)^\beta+\frac{\lambda}{p^2}C(P)] (M-r)^2 P(B(r)) -\alpha(P)$ almost surely. We choose $M$ large enough such that $[C(P)^\beta+\frac{\lambda}{p^2}C(P)] (M-r)^2 P(B(r)) -\alpha(P)> \Bigg(1+\frac{\lambda}{p^2}\Bigg)\int \|\mathbf{x}\|_2^2 P(d\mathbf{x})-p\alpha(P)$. This would make $\Phi(\mathcal{W}_n,A_n,P_n) > \Phi(\mathbf{1}_n,A_0,P_n) $ i.o., which is a contradiction.
\end{claimproof}
\begin{claim}

For $n$ large enough, $B(5M)$ contains all points of $A_n$ almost surely, i.e. 
$$P(\cup_{n=1}^\infty \cap_{m=n}^\infty \{\omega:A_m(\omega) \subset B(5M)  \})=1.$$
\end{claim}
\begin{claimproof}
\normalfont
We use induction for the proof of this step. We have seen from Step \ref{one}, the conclusions of this claim is valid. We assume this claim is valid for optimal allocation of $1,\dots,k-1$ cluster centroids. \par
Suppose $A_n$ contains at least one point outside $B(5M)$. Now if we delete this cluster centroid, at worst, the center $\mathbf{a}_1$, which is known to lie inside $B(M)$ might have to accept points that were previously assigned to cluster centroids outside $B(5M)$. These sample points must have been at a distance at least $2M$ from the origin, otherwise, they would have been closer to the centroid $a_1$, than to any other centroid outside $B(5M)$. Hence, the extra contribution to $\Phi(\cdot,\cdot,P_n)$, due to deleting the centroids outside $B(5M)$ is atmost 
\begingroup
\allowdisplaybreaks
\begin{equation}
\begin{aligned}
 & \int_{\|x\| \geq 2M} \sum_{l=1}^p ({w^{(n)}_l}^\beta+\frac{\lambda}{p^2}) (x_l-a_{1l})^2 P_n(dx) -\alpha(P_n)\sum_{l=1}^p w^{(n)}_l \\
&    \leq (1+\frac{\lambda}{p^2}) \int_{\|x\| \geq 2M} \sum_{l=1}^p  (x_l-a_{1l})^2 P_n(dx)-\alpha(P_n)p \\
& \leq 2(1+\frac{\lambda}{p^2}) \int_{\|x\| \geq 2M}  (\|x\|^2+\|a_1\|^2) P_n(dx) \\
& \leq 4(1+\frac{\lambda}{p^2}) \int_{\|x\| \geq 2M}  \|x\|^2 P_n(dx).
\end{aligned}
\end{equation}
\endgroup
Let $A_n^*$ be obtained by deleting the centroids outside $B(5M)$ of $A_n$. Since $A_n^*$ has atmost $k-1$ points, we have $\Phi(\mathcal{W}_n, A_n^*,P_n) \geq \Phi(\mathcal{V}_n, B_n,P_n)$, where $\mathcal{V}_n$ and $B_n$ denote the optimal set of weights and optimal set of cluster centroids for $k-1$ centers respectively.
Let $\Omega''=\{\omega \in \Omega: \forall n \in \mathbb{N}, \exists m\geq n , A_m(\omega) \not \subset B(5M)\}$. Now by Axiom of Choice, for any $\omega \in \Omega''$, there exists a sequence $\{n_h\}_{h \in \mathbb{N}}$ such that $n_i < n_j$ for $i<j$ and $A_{n_h} \not \subset B(5M).$

\begin{equation}
\label{5M}
    \begin{aligned}
    & m_{k-1}(P)\\
    & \leq \liminf_h\Phi(\mathcal{W}_{n_h}, A^*_{n_h},P_{n_h}) \text{ a.s.} \\
               & \leq \limsup_{h} [\Phi(\mathcal{W}_{n_h}, A^*_{n_h},P_{n_h}) \\
               &+4(1+\frac{\lambda}{p^2}) \int_{\|x\| \geq 2M}  \|x\|^2 P_{n_h}(dx)]\\
               & \leq \limsup_{h} \Phi(\mathcal{W}, A,P_{n_h})+ 4(1+\frac{\lambda}{p^2}) \int_{\|x\| \geq 2M}  \|x\|^2 P(dx) ,
    \end{aligned}
\end{equation}
for any $A$ having $k$ or fewer points and for any $\mathcal{W}\in \mathbb{R}^p$. Choose $A=\bar{A}(k)$ and $\mathcal{W}=\bar{\mathcal{W}}(k)$. Choose $\epsilon>0$ such that $m_k(P)+\epsilon<m_{k-1}(P)$. Choose $M$ large enough such that $4(1+\frac{\lambda}{p^2}) \int_{\|x\| \geq 2M}  \|x\|^2 P(dx) < \epsilon$. Thus, the last bound of Eqn \ref{5M} is less than $\Phi(\bar{\mathcal{W}}(k),\bar{A}(k),P)+\epsilon=m_k(P)+\epsilon>m_{k-1}(P)$, which is a contradiction. 
\end{claimproof}
Hence, for $n$ large enough, it suffices to search for $A_n$ among the class of sets, $\xi_k:=\{A \subset B(5M)| A \text{ contains } k\text{ or fewer points}\}$. For the final requirement on $M$, we assume that $M$ is large enough so that $\xi_k$ contains $\bar{A}(k)$. Under the topology induced  by the Hausdroff metric, $\xi_k$ is compact. Let $\Gamma_k=[0,b] \times \dots [0,b]$  ($p$ times), where $b$ is such that $b>1$ and $\bar{\mathcal{W}}(k)_l<b$ $\forall d \in \{1,\dots, p\}$. As proved in Theorem \ref{cont}, the map $(\mathcal{W},A)\to \Phi(\mathcal{W},A,P)$ is continuous on $\Gamma_k \times \xi_k $. The function $\Phi(\cdot, \cdot, P)$ has the property that given any neighbourhood $\mathcal{N}$ of $(\bar{\mathcal{W}}(k),\bar{A}(k))$ (depending on $\eta$) $\Phi(\mathcal{W},A,P)\geq\Phi(\bar{\mathcal{W}}(k),\bar{A}(k),P)+\eta$, for every $(\mathcal{W},A) \in \Gamma_k \times \xi_k \setminus  \mathcal{N}$. \par
Now by uniform SLLN (Theorem \ref{slln_ultimate}), we have,
$$sup_{(\mathcal{W},A) \in \Gamma_k \times \xi_k} |\Phi(\mathcal{W},A,P_n)-\Phi(\mathcal{W},A,P)| \xrightarrow{ a.s.}0.$$
We need to show that $(\mathcal{W}_n,A_n)$ eventually lies inside $\mathcal{N}$. It is enough to show that $\Phi(\mathcal{W}_n,A_n,P)<\Phi(\bar{\mathcal{W}}(k),\bar{A}(k),P)+\eta$, eventually. This follows from $$\Phi(\mathcal{W}_n,A_n,P_n)\leq \Phi(\bar{\mathcal{W}}(k),\bar{A}(k),P_n),$$
$$\Phi(\mathcal{W}_n,A_n,P_n)-\Phi(\mathcal{W}_n,A_n,P)\xrightarrow{a.s.}0,$$
and 
$$\Phi(\bar{\mathcal{W}}(k),\bar{A}(k),P_n)-\Phi(\bar{\mathcal{W}}(k),\bar{A}(k),P)\xrightarrow{a.s.}0.$$
Similarly for $n$ large enough,
\begin{align*}
\Phi(\mathcal{W}_n,A_n,P_n) & =inf\{\Phi(\mathcal{W},A,P_n)|\mathcal{W} \in \Gamma_k,A \in \xi_k\}\\
& \xrightarrow{a.s.}inf\{\Phi(\mathcal{W},A,P)|\mathcal{W} \in \Gamma_k,A \in \xi_k\}\\
& =m_k(P).    
\end{align*}
\end{proof}

\subsection{Uniform SLLN and continuity of $\Phi(\cdot,\cdot,P)$}
In this section, we prove a uniform SLLN for the function $\Phi(\cdot,\cdot,P)$ in Theorem \ref{slln}.
\begin{thm}
\label{slln}
Let $\mathcal{G}$ denote the family of all $P$-integrable functions of the form $g_{\mathcal{W},A}(x):=min_{\mathbf{a}\in A}\sum_{l=1}^p(w_l^\beta+\frac{\lambda}{p^2}|w_l|)(x_{l}-a_l)^2$, where $A\in \xi_k$ and $\mathcal{W}\in \Gamma_k$. Then $sup_{g \in \mathcal{G}}|\int g P_n-\int g P|\xrightarrow{a.s}0$.
\end{thm}
\begin{proof}
It is enough to show that for every $\epsilon>0$, $\exists$ a finite class of functions $\mathcal{G}_{\epsilon}$, such that for each $g\in \mathcal{G}$, there exists functions $\Dot{g}$, $\Bar{g}\in \mathcal{G}_{\epsilon}$ such that $\Dot{g} \leq g \leq \Bar{g}$ and $\int(\Bar{g}-\Dot{g})P(dx)<\epsilon$.\par
Let $D_{\delta_1}$ be a finite subset of $B(5M)$ such that every point of $B(5M)$ lies within a $\delta_1$ distance of at least one point of $D_{\delta_1}$. Also let $D_{\delta_2}$ be a finite subset of $\Gamma_k$ such that every point of $\Gamma_k$ is within a $\delta_2$ distance of at least one point of $D_{\delta_2}$. $\delta_1$ and $\delta_2$ will be chosen later. Let, $\xi_{k,\delta_1}=\{A \in \xi_k|A\subset D_{\delta_1}\}$ and $\Gamma_{k,\delta_2}=\{\mathcal{W}\in \Gamma_k|\mathcal{W}\subset D_{\delta_2}\}$. Take $\mathcal{G}_{\epsilon}$ to be the class of functions of the form $$min_{\mathbf{a}\in A'}\sum_{l=1}^p((w_l\pm \delta_2)^\beta+\frac{\lambda}{p^2}|w_l\pm \delta_2|)(x_{l}-a_l\pm \delta_1)^2,$$
where $A'$ ranges over $\xi_{k,\delta_1}$ and $\mathcal{W}$ ranges over $\Gamma_{k,\delta_2}$.\par
Given $A=\{\mathbf{a}_1,\dots,\mathbf{a}_k\} \in \xi_k $, there exists $A^{0}=\{\mathbf{a}^{(0)}_1,\dots,\mathbf{a}^{(0)}_k\}\in \xi_{k,\delta_1} $, such that $H(A,A')<\delta_1$ (choose $\mathbf{a}^{(0)}_i\in D_{\delta_1}$ such that $\|\mathbf{a}_i-\mathbf{a}^{(0)}_i\|<\delta_1$). Also note that given $\mathcal{W} \in \Gamma_k$, there exists $\mathcal{W}^{(0)} \in \Gamma_{k,\delta_2}$. For given $g_{\mathcal{W},A} \in \mathcal{G}$, take, 
\begin{align*}
\Dot{g}_{\mathcal{W},A}&:=min_{\mathbf{a}\in A^{(0)}}\sum_{l=1}^p((max\{w^{(0)}_l-\delta_2,0\})^\beta\\
&+\frac{\lambda}{p^2}|max\{w^{(0)}_l-\delta_2,0\}|)(max\{x_{l}-a_l -\delta_1,0\})^2    
\end{align*}
 and
\begin{align*}
\Bar{g}_{\mathcal{W},A} & :=min_{\mathbf{a}\in A^{(0)}}\sum_{l=1}^p((max\{w^{(0)}_l+\delta_2,0\})^\beta\\
&+\frac{\lambda}{p^2}|max\{w^{(0)}_l+\delta_2,0\}|)(max\{x_{l}-a_l +\delta_1,0\})^2.    
\end{align*}

Clearly, $\Dot{g}_{\mathcal{W},A}\leq g_{\mathcal{W},A}\leq \Bar{g}_{\mathcal{W},A} $. Now by taking $R>5M$, We have, 

\begingroup
\allowdisplaybreaks
{\footnotesize
\begin{equation}
    \begin{aligned}
    & \int  (\Bar{g}_{\mathcal{W},A}-\Dot{g}_{\mathcal{W},A})P(dx)  \\
    &\leq \sum_{i=1}^k \int \Bigg[ \sum_{l=1}^p((max\{w^{(0)}_l+\delta_2,0\})^\beta\\
    &+\frac{\lambda}{p^2}|max\{w^{(0)}_l+\delta_2,0\}|)(max\{x_{l}-a_l +\delta_1,0\})^2\\
    &-\sum_{l=1}^p((max\{w^{(0)}_l-\delta_2,0\})^\beta\\
    &+\frac{\lambda}{p^2}|max\{w^{(0)}_l-\delta_2,0\}|)(max\{x_{l}-a_l -\delta_1,0\})^2\Bigg]P(dx)\\
    & \leq kp \text{ } sup_{|x|>5M} \text{ } sup_{|a|<5M} \text{ } sup_{|w|<L} \Bigg[((max\{w+\delta_2,0\})^\beta\\
    &+\frac{\lambda}{p^2}|max\{w+\delta_2,0\}|)(max\{x-a -\delta_1,0\})^2 \\
    &-((max\{w-\delta_2,0\})^\beta\\
    &+\frac{\lambda}{p^2}|max\{w-\delta_2,0\}|)(max\{x-a -\delta_1,0\})^2\Bigg]\\
    &+2(1+\frac{\lambda}{p^2}) \int_{\|x\| \geq R}  \|x\|^2 P(dx).
    \end{aligned}
\end{equation}
}%
\endgroup
The second term can be made smaller than $\epsilon/2$ if $R$ is made large enough. Now appealing to the uniform continuity of the function $((max\{w,0\})^\beta+\frac{\lambda}{p^2}|max\{w,0\}|)(max\{x,0\})^2$ on bounded sets, we can find $\delta_1$ and $\delta_2$ small enough such that the first term is less than $\epsilon/2$. Hence the result.
\end{proof}

\begin{thm}
\label{slln_ultimate}
Let $\mathcal{G}$ denote the family of all $P$-integrable functions of the form $g_{\mathcal{W},A}(x):=min_{\mathbf{a}\in A}\sum_{l=1}^p(w_l^\beta+\frac{\lambda}{p^2}|w_l|)(x_{l}-a_l)^2$, where $A\in \xi_k$ and $\mathcal{W}\in \Gamma_k$. Let $g_{\mathcal{W},A,P_n}(x)=min_{\mathbf{a}\in A}\sum_{l=1}^p(w_l^\beta+\frac{\lambda}{p^2}|w_l|)(x_{l}-a_l)^2-\alpha(P_n)\sum_{l=1}^pw_l$. Then the following holds:
\begin{enumerate}
    \item $\int g_{\mathcal{W},A,P_n}(x)P_n(x)dx=\Phi(\mathcal{W},A,P_n)$.
    \item $sup_{\mathcal{W},A}|\int g_{\mathcal{W},A,P_n} P_n-\int g_{\mathcal{W},A,P} P|\xrightarrow{a.s}0$.
\end{enumerate}

\end{thm}
\begin{proof}
Part $(1)$ follows trivially. We only prove part $(2)$. Clearly,
\begin{equation}
\label{sup}
\begin{aligned}
    &|\int g_{\mathcal{W},A,P_n} P_n-\int g_{\mathcal{W},A,P} P| \\
    & \leq |\int g_{\mathcal{W},A} P_n-\int g_{\mathcal{W},A} P|+\Bigg|\sum_{l=1}^pw_l\Bigg||\alpha(P_n)-\alpha(P)|\\
    &\leq |\int g_{\mathcal{W},A} P_n-\int g_{\mathcal{W},A} P|+b|\alpha(P_n)-\alpha(P)|.
\end{aligned}
\end{equation}
Hence,
{\footnotesize
\begin{equation}
\begin{aligned}
\label{sup1}
    &sup_{\mathcal{W},A}|\int g_{\mathcal{W},A,P_n} P_n-\int g_{\mathcal{W},A,P} P| \\
    &\leq sup_{\mathcal{W},A}|\int g_{\mathcal{W},A} P_n-\int g_{\mathcal{W},A} P|+sup_{\mathcal{W},A}b|\alpha(P_n)-\alpha(P)|\\
    &= sup_{\mathcal{W},A}|\int g_{\mathcal{W},A} P_n-\int g_{\mathcal{W},A} P|+b|\alpha(P_n)-\alpha(P)|\\
    &\xrightarrow{a.s.}0.
\end{aligned}
\end{equation}
}%
The last almost sure convergence of Eqn \ref{sup1} is true since the first term converges to $0$ a.s. (Theorem \ref{slln}) and the second term converges to $0$ a.s. (Theorem \ref{bddwt}).
\end{proof}

 Before proceeding any further let us first define two function classes. 
\begin{itemize}
    \item Let $f_A(\mathbf{w})=\Phi(\mathbf{w},A,P)$, and let us define $\mathcal{F}_1=\{f_A: A \in \xi_k\}$.        
    \item Let $f_\mathbf{w}(A)=\Phi(\mathbf{w},A,P)$, and let us define $\mathcal{F}_2=\{f_\mathbf{w}: \mathbf{w} \in \mathcal{S}\}$.
\end{itemize}
In Lemmas \ref{lem1} and \ref{lem2}, we show that the families $\mathcal{F}_1$ and $\mathcal{F}_2$ are both equicontinuous \cite{rudin2006real}.
\begin{lemma}
\label{lem1}
The family of functions $\mathcal{F}_1$ is equicontinuous. 
\end{lemma}
\begin{proof}
See Appendix \ref{lem1 proof}
\end{proof}
\begin{lemma}
\label{lem2}
The family of functions $\mathcal{F}_2$ is equicontinuous. 
\end{lemma}

\begin{proof}
See Appendix \ref{lem2 proof}.
\end{proof}
Before we state the next theorem, note that, the map $(\mathcal{W},A) \to \Phi(\mathcal{W},A,P)$ is from $\Gamma_k \times \xi_k \to \mathbb{R}$. $\Gamma_k \times \xi_k$ is a metric space with the metric $$d_0((\mathcal{W}_1,A_1),(\mathcal{W}_2,A_2)):=\|\mathcal{W}_1-\mathcal{W}_2\|_2^2+H(A_1,A_2),$$ where, $H(\cdot,\cdot)$ is the Hausdorff metric.

\begin{thm}\label{cont}
The map $(\mathcal{W},A) \to \Phi(\mathcal{W},A,P)$ is continuous on $\Gamma_k \times \xi_k$.
\end{thm}
\begin{proof}
Fix $(\mathcal{W}_0,A_0) \in \Gamma_k \times \xi_k$. From triangle inequality, we get,
{\footnotesize
\begin{align*}
& |\Phi(\mathcal{W},A,P)-\Phi(\mathcal{W}_0,A_0,P)| \\
&\leq |\Phi(\mathcal{W},A,P)-\Phi(\mathcal{W},A_0,P)|+|\Phi(\mathcal{W},A_0,P)-\Phi(\mathcal{W}_0,A_0,P)|.    
\end{align*}
}%
The first term can be made smaller than $\epsilon/2$ if $A$ is chosen close enough to $A_0$ (in Hausdorff sense). This follows from Lemma \ref{lem1}. The second term can also be made smaller than $\epsilon/2$ if $\mathcal{W}$ is chosen close enough to $\mathcal{W}_0$ (in Euclidean sense). This follows from Lemma \ref{lem2}. Hence the result. 
\end{proof}

\section{Experimental Results}
In this section, we present the experimental results on various real-life and synthetic datasets. All the experiments were undertaken on an \texttt{HP} laptop with Intel(R) Core(TM) i3-5010U 2.10 GHz processor, 4GB RAM, 64-bit Windows 8.1 operating system. 
The datasets and codes used in the experiments are publicly available from \url{https://github.com/SaptarshiC98/lwk-means}.
\begin{figure}
    \centering
    \includegraphics[width=0.55\textwidth]{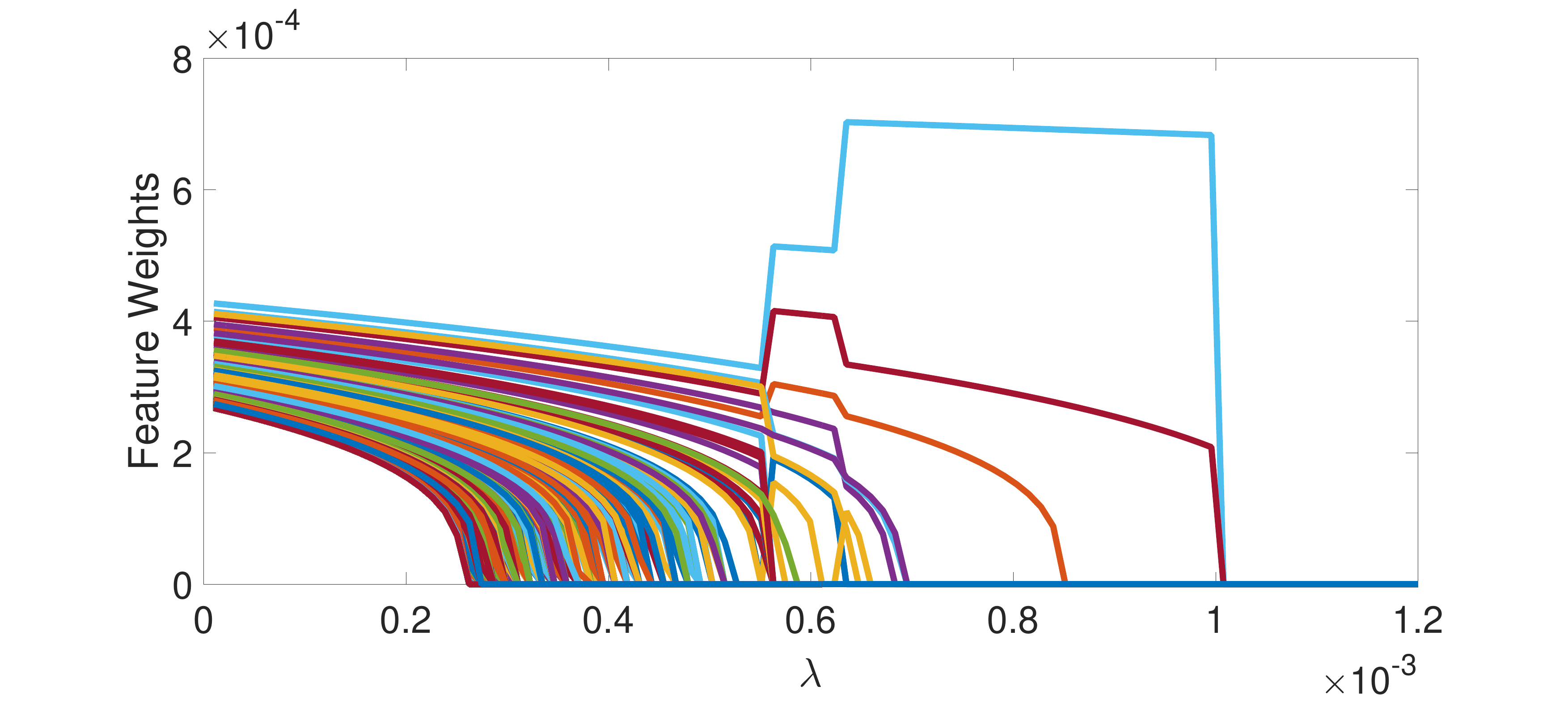}
    \caption{Regularization paths for the Leukemia dataset.}
    \label{fig:leu}
\end{figure}
\subsection{Regularization Paths}
In this section we discuss the concept of regularization paths in the context of $LW$-$k$-means. The term regularization path was first introduced in the context of lasso \cite{tibshirani1996regression}. We introduce two new concepts called {\it mean regularization path} and {\it median regularization path} in the context of $LW$-$k$-means. Suppose we have a sequence $\{\lambda_i\}_{i=1}^{n}$ of length $n$ of $\lambda $ values. After setting $\lambda=\lambda_i$, we run the $LW$-$k$-means algorithm $t$ times (say). Hence we have a set of $t$ weights, $\mathcal{W}_1,\dots ,\mathcal{W}_t$. Hence we can take the estimates of the average weight to be the mean of these $t$ vectors. Let this estimate be  $\mathcal{W}^*_i$. Thus, for each value $\lambda_i$, we get the mean weights $\mathcal{W}^*_i$. This sequence of $\mathcal{W}^*_i$'s, $\{\mathcal{W}^*_i\}_{i=1}^{n}$  is defined to be the mean regularization path. Similarly one can define the median regularization path by taking the median of the weights instead of the mean. 
\subsection{Case Studies in Microarray Datasets}
A typical microarray dataset has several thousands of genes and fewer than $100$ samples. We use the Leukemia and Lymphoma  datasets to illustrate the effectiveness of the $LW$-$k$-means algorithm. We do not include $k$-means and IF-HCT-PCA in the following examples since both the algorithms does not perform feature weighting.
\subsubsection{Example 1}
The Leukemia dataset consists of $3571$ gene expressions and $72$ samples. The dataset was collected by Golub {\it et al.} \cite{golub1999molecular}. We run the $LW$-$k$-means algorithm $100$ times  for each value of $\lambda$ and note the average value of the different feature  weights. We also note the average CER for different $\lambda$ values.
\begin{figure}
    \centering
    \begin{subfigure}{0.49\textwidth}
        \includegraphics[width=\textwidth]{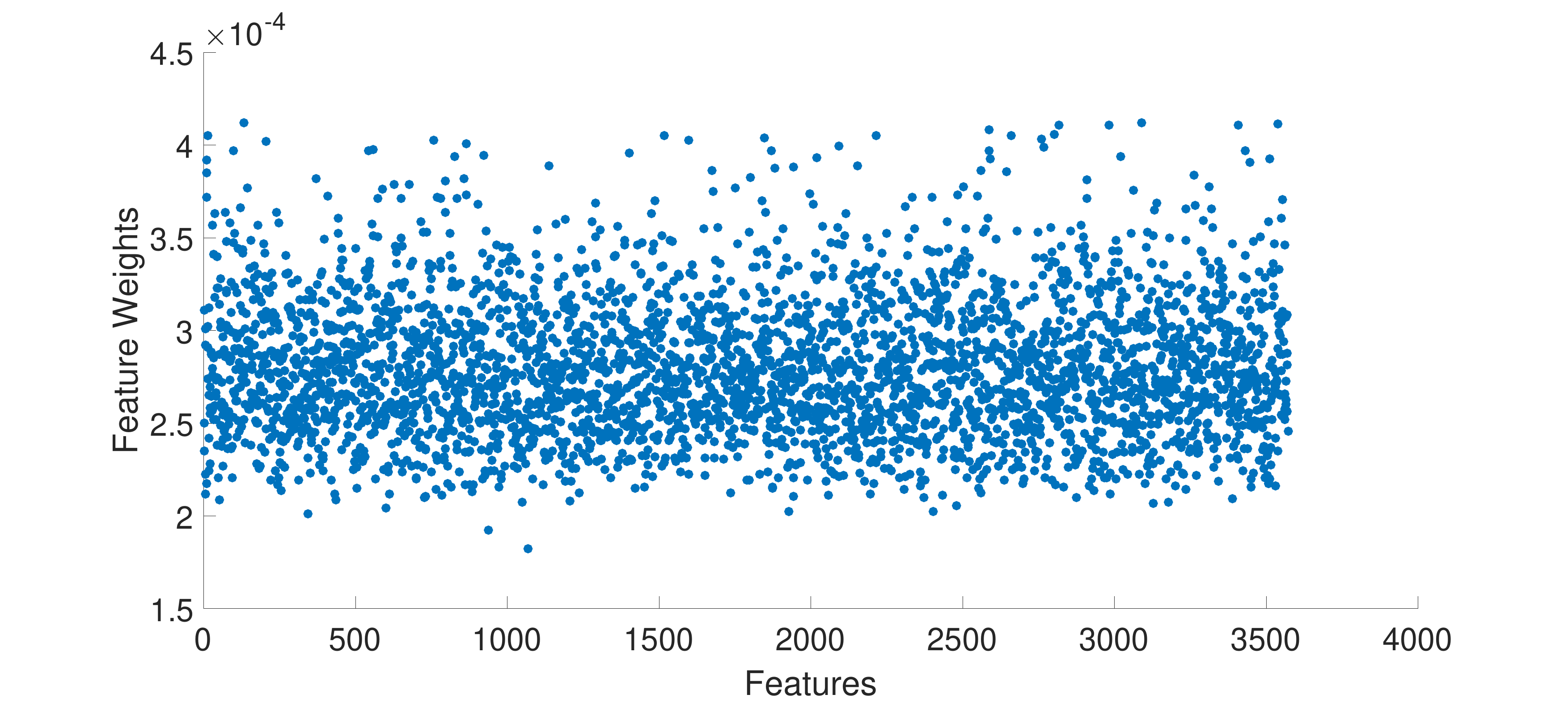}
        \caption{$WK$-means weights}
        \label{fig:luk_huang}
    \end{subfigure}
    \begin{subfigure}{0.49\textwidth}
        \includegraphics[width=\textwidth]{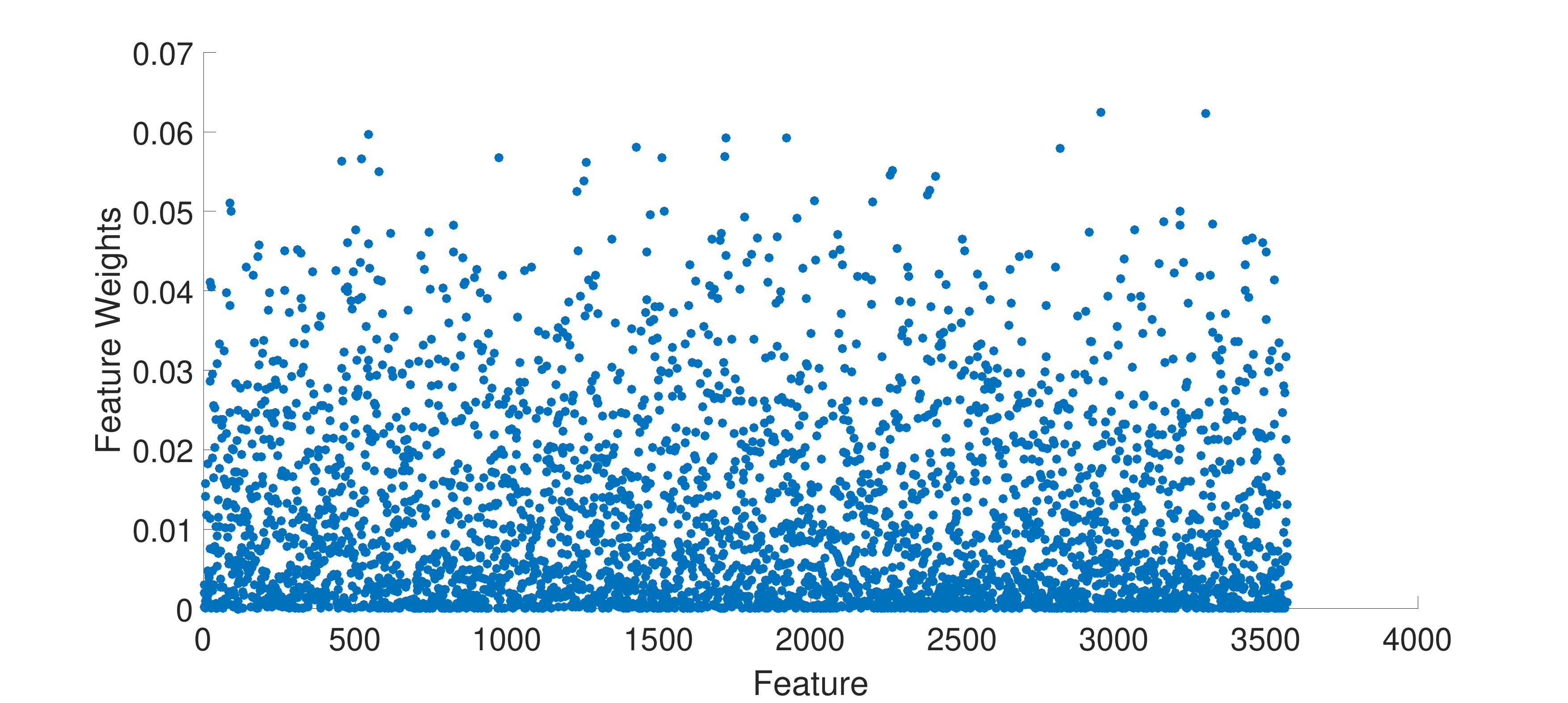}
        \caption{Sparse $k$-means}
        \label{fig:luk_sp}
    \end{subfigure}
    \begin{subfigure}{0.49\textwidth}
        \includegraphics[width=\textwidth]{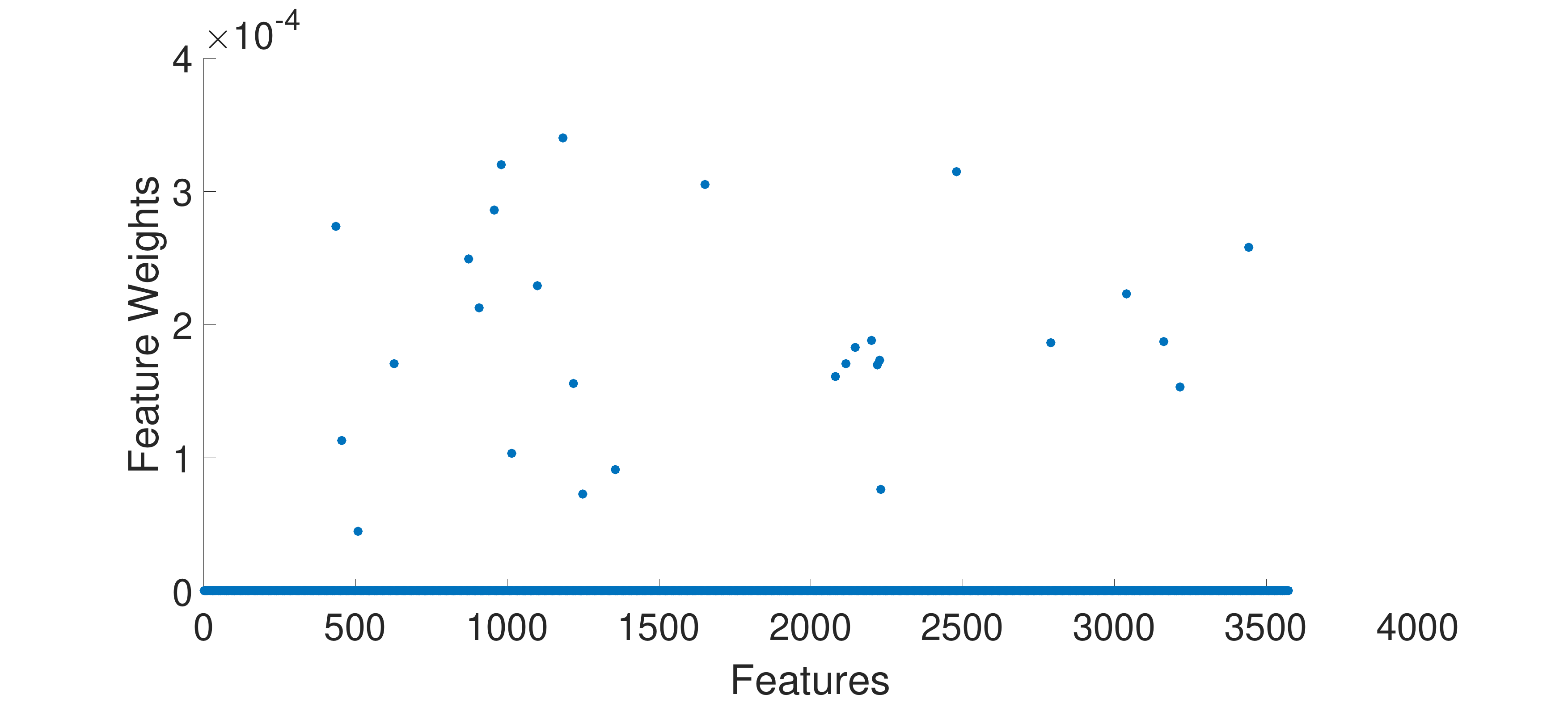}
        \caption{$LW$-$k$-means weights}
        \label{fig:luk_lwk}
    \end{subfigure}
    \caption{Average Weights assigned to different features by the $WK$-means and $LW$-$k$-means algorithm for lymphoma dataset. $LW$-$k$-means assigns zero feature weights to many of the features whereas $WK$-means and sparse $k$-means does not. 
    }
    \label{fig:luk}
\end{figure}
\begin{figure}[htb]
    \centering
    \includegraphics[width=0.4\textwidth]{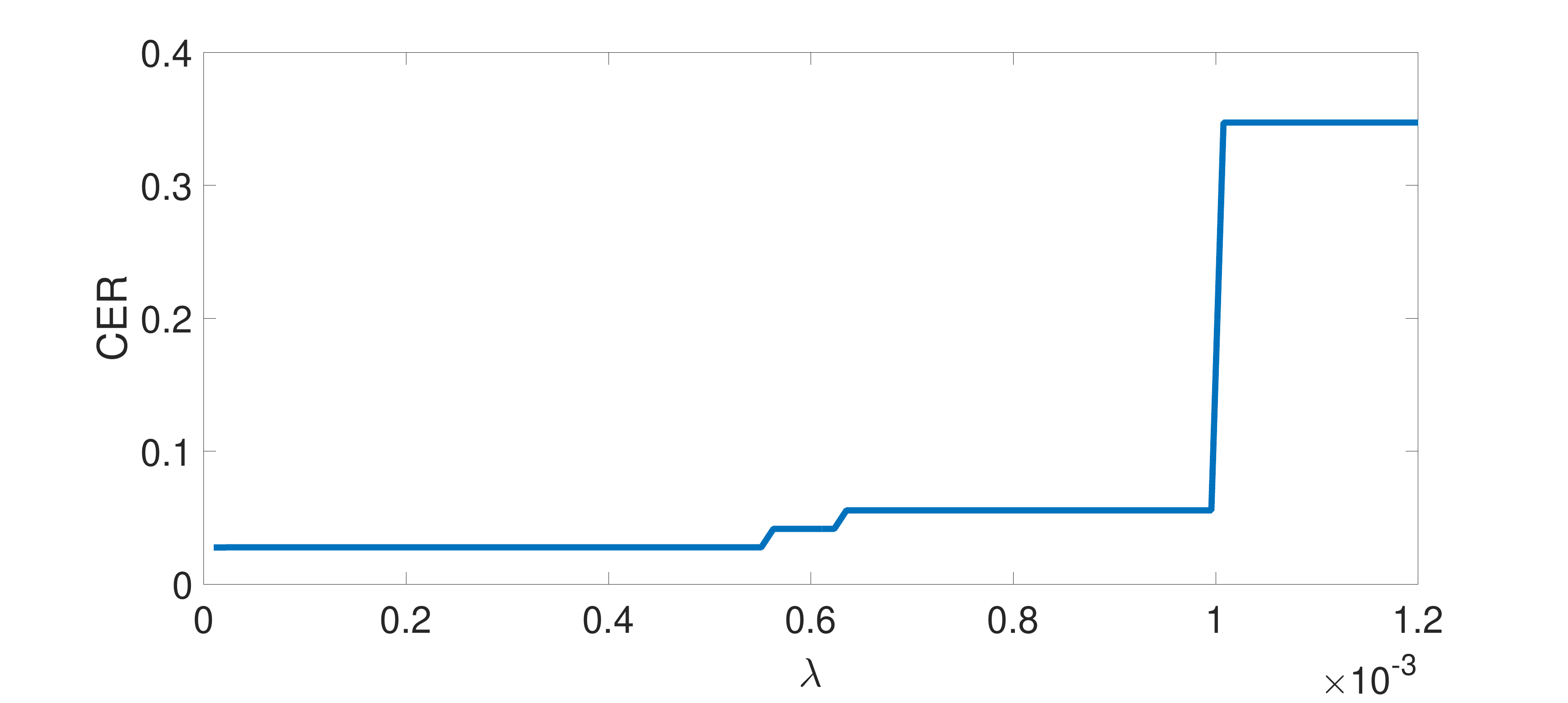}
    \caption{Average CER for different values of $\lambda$ for Leukemia dataset.}
    \label{fig:leu_r}
\end{figure}

In Fig. \ref{fig:leu}, we show the regularization paths for the Leukemia dataset.  In Fig. \ref{fig:leu_r}, we plot the average misclassification error rate for the same dataset. It is evident from Fig. \ref{fig:leu_r}, that as we decrease $\lambda$ the average CER drops down abruptly around $\lambda=0.52 \times 10^{-3}$. From Fig. \ref{fig:leu}, we observe that only few features are selected (on an average, 10 for $\lambda=0.6 \times 10^{-3}$) when $\lambda>0.5 \times 10^{-3}$. Possibly these features do not completely reveal the cluster structure of the dataset. As $\lambda$ is decreased, the CER remains more or less stable.
We also run the $WK$-means and sparse $k$-means algorithms 100 times (we performed the experiment 100 times to get a more consistent view of the feature weight) on the Leukemia dataset and compute the median of the weights for different features. In Fig. \ref{fig:luk_huang} and \ref{fig:luk_sp}, we plot these feature weights against the corresponding features for $WK$-means and sparse $k$-means respectively. It can be easily seen that $WK$-means and sparse $k$-means do not assign zero weight to all the features. In Fig. \ref{fig:luk_lwk}, we plot corresponding average (median) feature weights assigned by the $LW$-$k$-means algorithm. It can be easily observed that $LW$-$k$-means assigns zero feature weights to many of the features.\par 
\subsubsection{Example 2}
The Lymphoma dataset consists of $4026$ gene expressions and $62$ samples. The dataset was collected by  Alizadeh {\it et al.} \cite{alizadeh2000distinct}.  We run the $LW$-$k$-means algorithm $100$ times for each value of $\lambda$ and note both the mean and median values of the different feature weights. We also note the both the mean and median CER's for different $\lambda$ values.

\begin{figure}
    \centering
    \begin{subfigure}[b]{0.48\textwidth}
        \includegraphics[width=\textwidth]{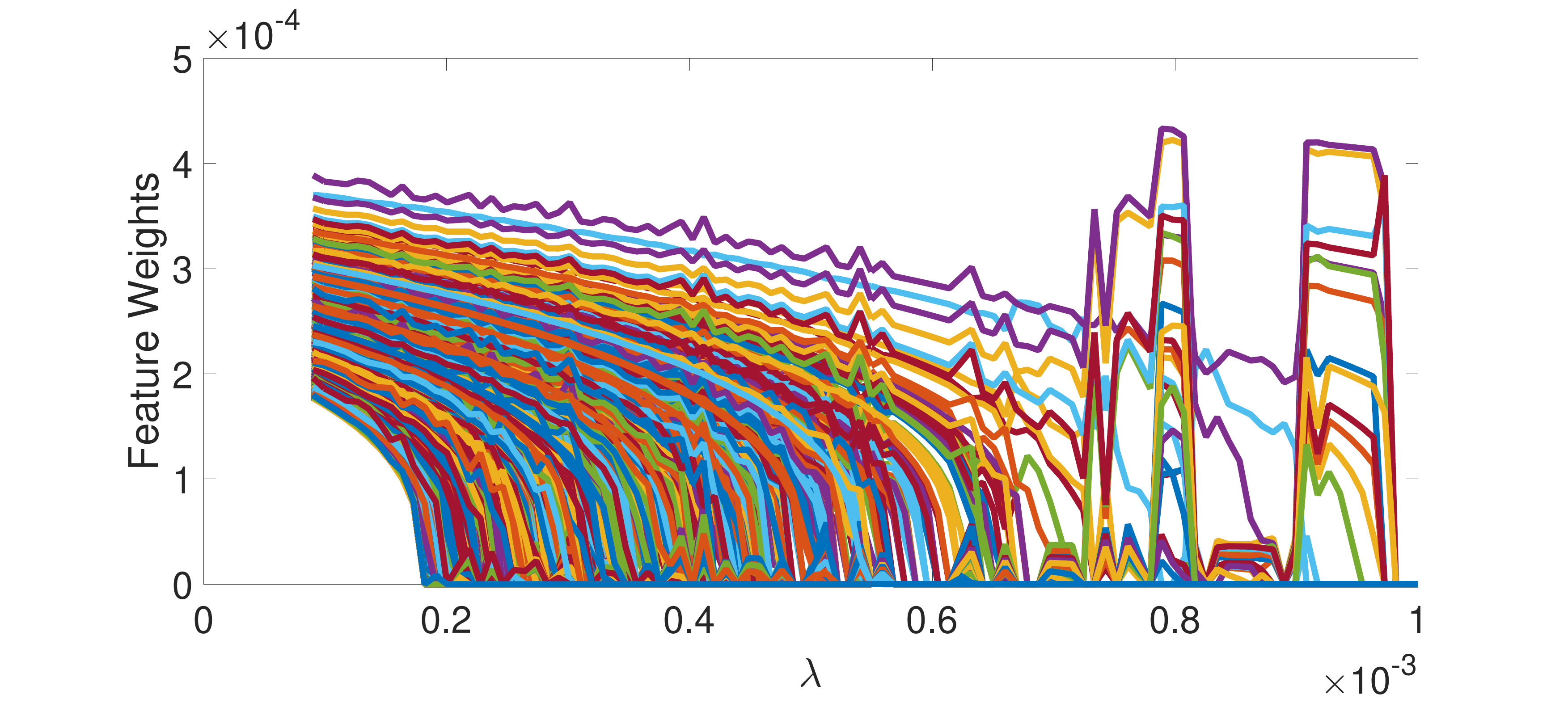}
        \caption{\small{ Mean Regularization Path for Lymphoma Dataset}}
        \label{fig:mean}
    \end{subfigure}
    \begin{subfigure}[b]{0.48\textwidth}
        \includegraphics[width=\textwidth]{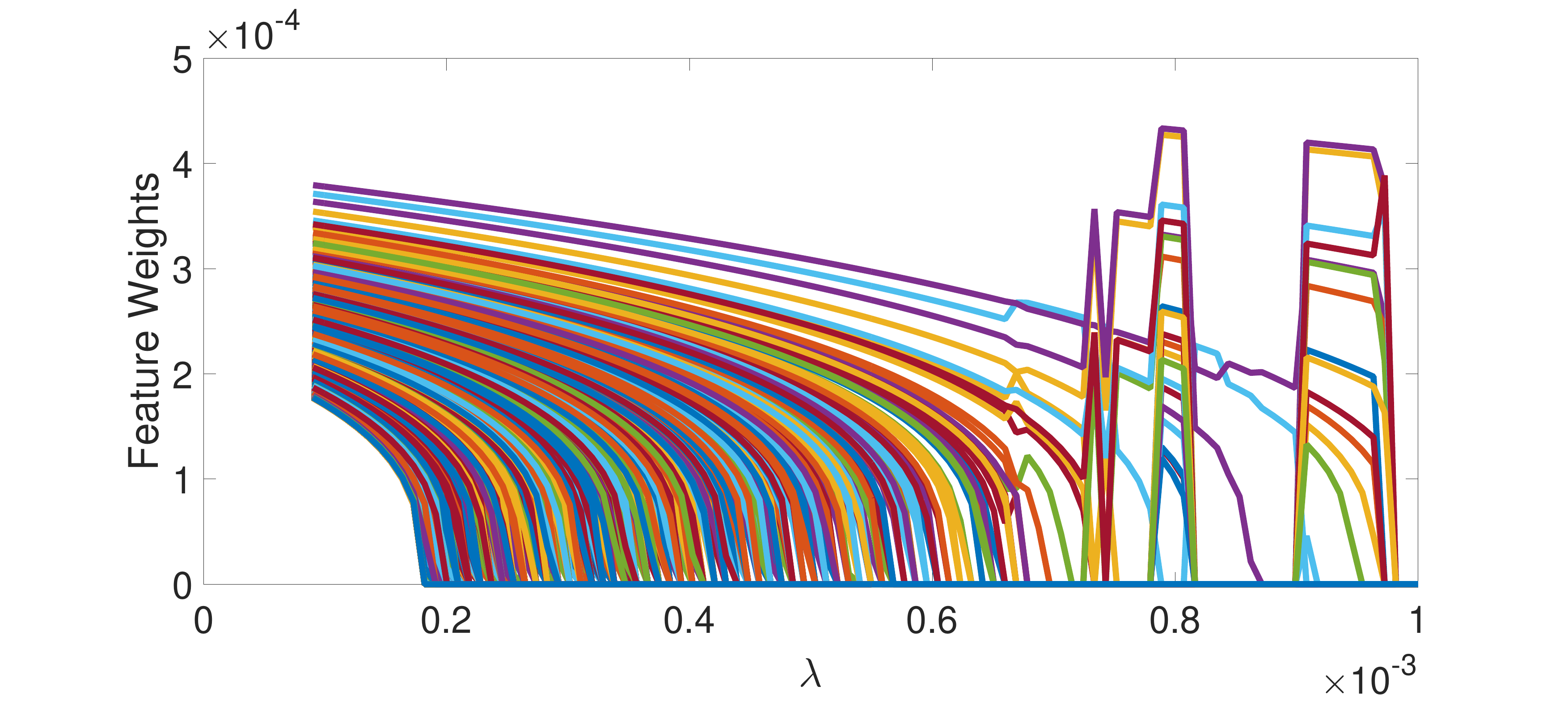}
        \caption{Median Regularization Path for Lymphoma Dataset}
        \label{fig:median}
    \end{subfigure}
    \caption{Regularization Path for Lymphoma Dataset.}\label{fig:lympho}
\end{figure}
\begin{figure}[htb]
    \centering
    \begin{subfigure}{0.49\textwidth}
        \includegraphics[width=\textwidth]{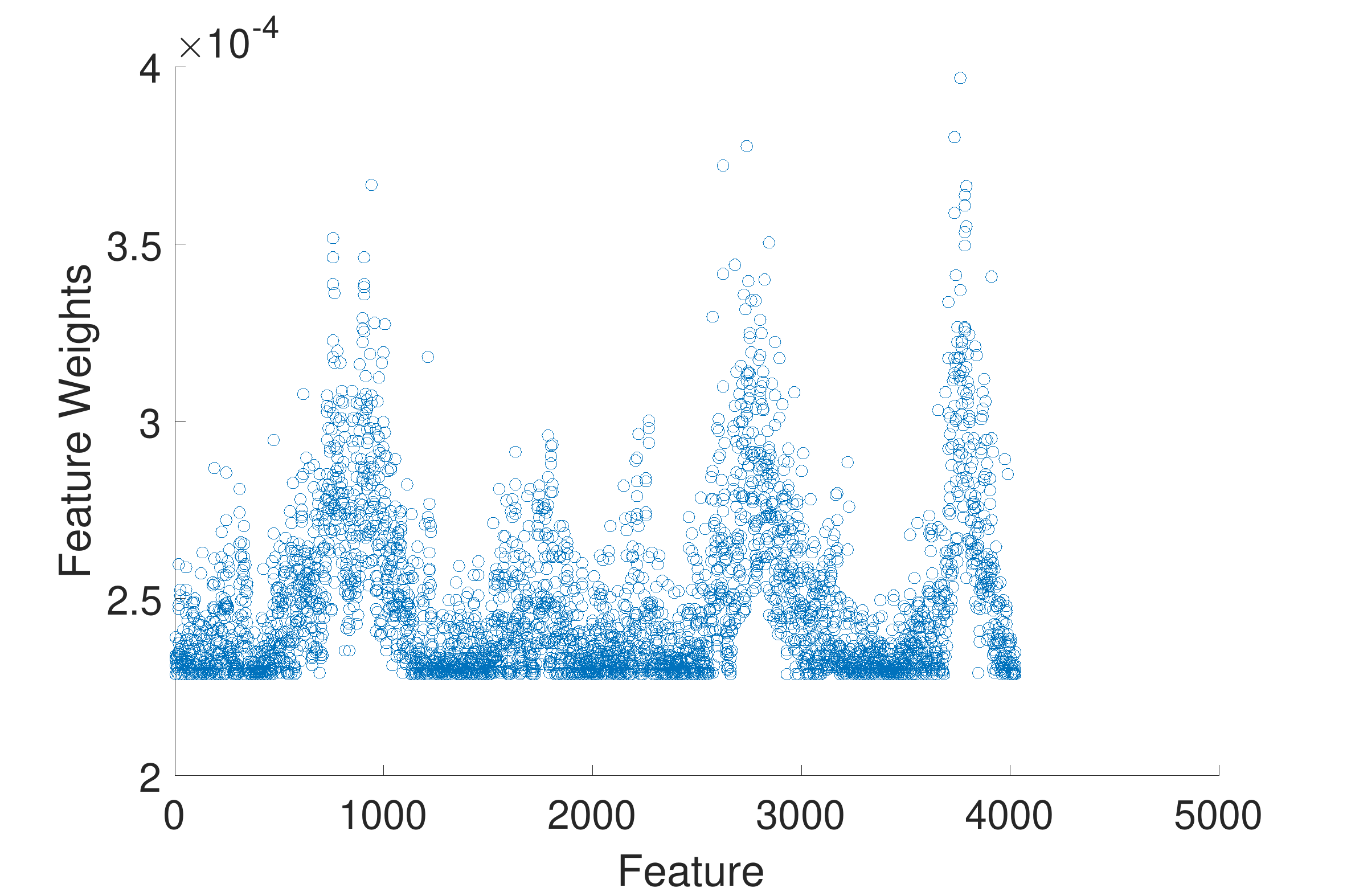}
        \caption{$WK$-means weights}
        \label{fig:lim}
    \end{subfigure}
    \begin{subfigure}{0.49\textwidth}
        \includegraphics[width=\textwidth]{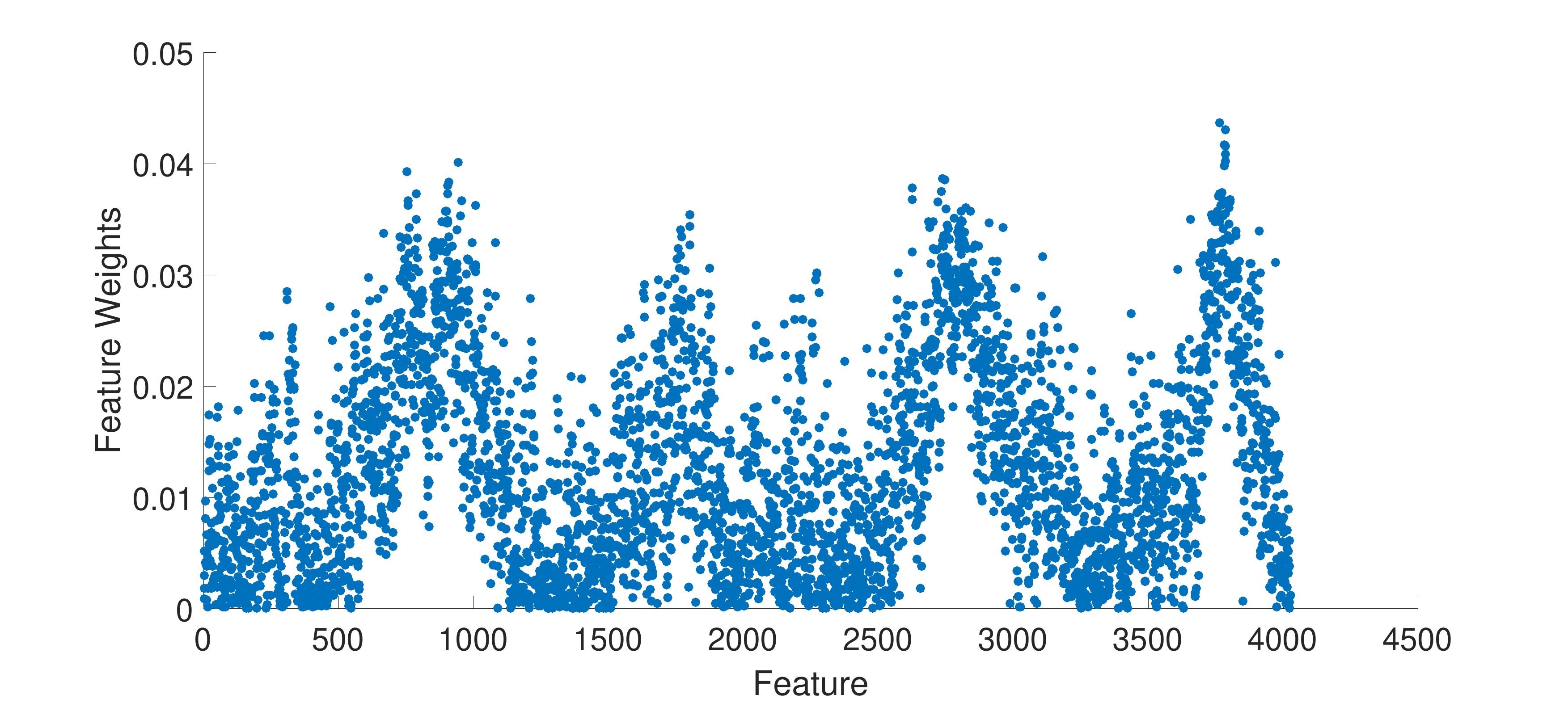}
        \caption{Sparse $k$-means}
        \label{sp1}
    \end{subfigure}
    \begin{subfigure}{0.49\textwidth}
        \includegraphics[width=\textwidth]{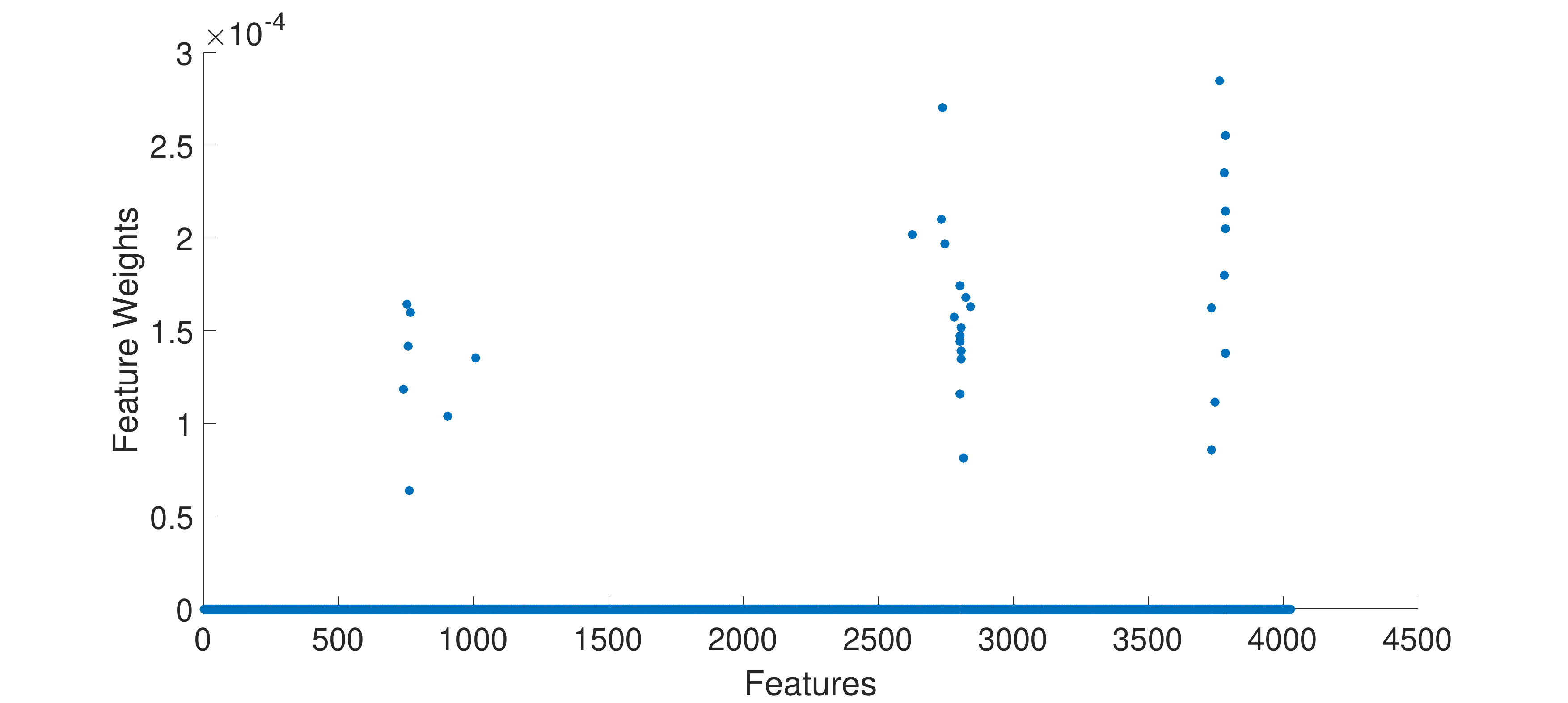}
        \caption{$LW$-$k$-means weights}
        \label{fig:lim_lwk}
    \end{subfigure}
    \caption{Average Weights assigned to different features by the $WK$-means and $LW$-$k$-means algorithm for lymphoma dataset. $LW$-$k$-means assigns zero feature weights to many of the features whereas $WK$-means and sparse $k$-means does not. Also, the features which were given more weights by $WK$-means, many of them have non-zero weights assigned by $LW$-$k$-means.}
    \label{fig:limy}
\end{figure}

In Fig. \ref{fig:lympho}, we plot the average (both mean and median) regularization paths and in Fig. \ref{fig:lympho_r}, we plot the average (both mean and median) CER for different values of $\lambda$. We observe that the median regularization path is smoother relative to the mean regularization path.  We also see from Fig. \ref{fig:lympho_r}, that the mean CER curve is less smooth than the median CER curve. The non-smooth mean regularization paths indicate a few cases where due to a bad initialization, the solutions got stuck at a local minimum instead of the global minima of the objective function. During our
experiments we observed that there were a few times when we got a bad initialization for cluster centroids, thus adversely affecting the mean regularization path and mean CER. On the other hand, the median is more robust against outliers and thus the corresponding regularization paths and CER are smoother compared to those corresponding to the mean. From Fig. \ref{fig:median_r}, we observe that there is a sudden drop in the misclassification error rate around $\lambda=6.2 \times 10^{-4}$ and it remains stable when $\lambda$ is further decreased. This might be due to the fact that when $\lambda$ is high, no features are selected and as $\lambda$ is decreased to around $\lambda=6.2 \times 10^{-4}$, the relevant features are selected. Also note that these features have higher weights than other features, even when $\lambda$ is quite small. The above facts indicate that indeed the $LW$-$k$-means detects the features which contain the cluster structure of the data.\par
We also run the $WK$-means and sparse $k$-means algorithms 100 times on the Lymphoma dataset and compute the median of the weights for different features. In figures \ref{fig:lim} and \ref{sp1}, we plot these feature weights against the corresponding features for $WK$-means and sparse $k$-means respectively. It can be easily seen that $WK$-means and sparse $k$-means do not assign zero feature weights and thus in effect do not perform a feature selection. In Fig. \ref{fig:lim_lwk}, we plot the corresponding average (median) feature weights assigned by the $LW$-$k$-means algorithm. It is easily observed that $LW$-$k$-means assigns zero feature weights to many of the features. 

\begin{figure}
    \centering
    \begin{subfigure}{0.47\textwidth}
        \includegraphics[width=\textwidth]{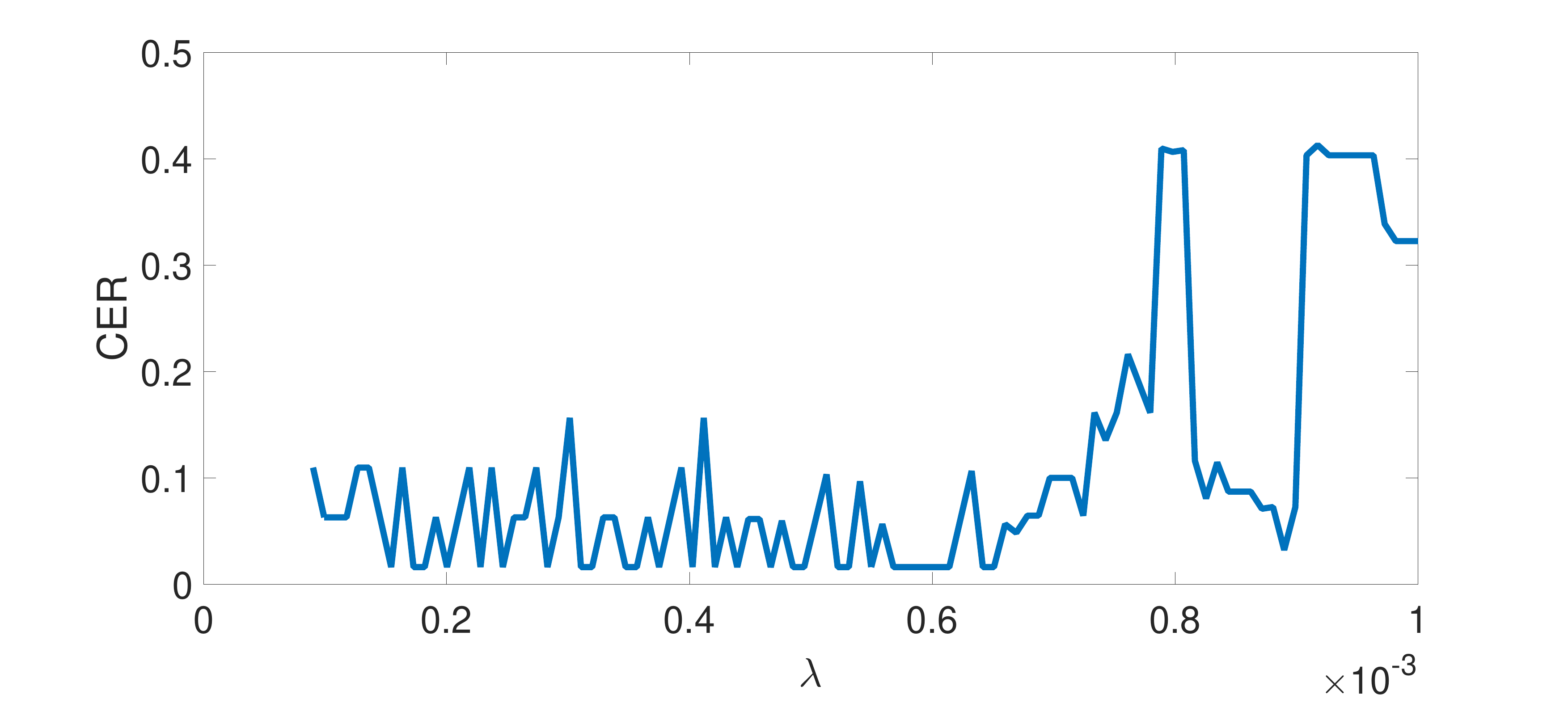}
        \caption{Mean CER for different values of $\lambda$}
        \label{fig:mean_r}
    \end{subfigure}
    ~ 
    \begin{subfigure}{.47\textwidth}
        \includegraphics[width=\textwidth]{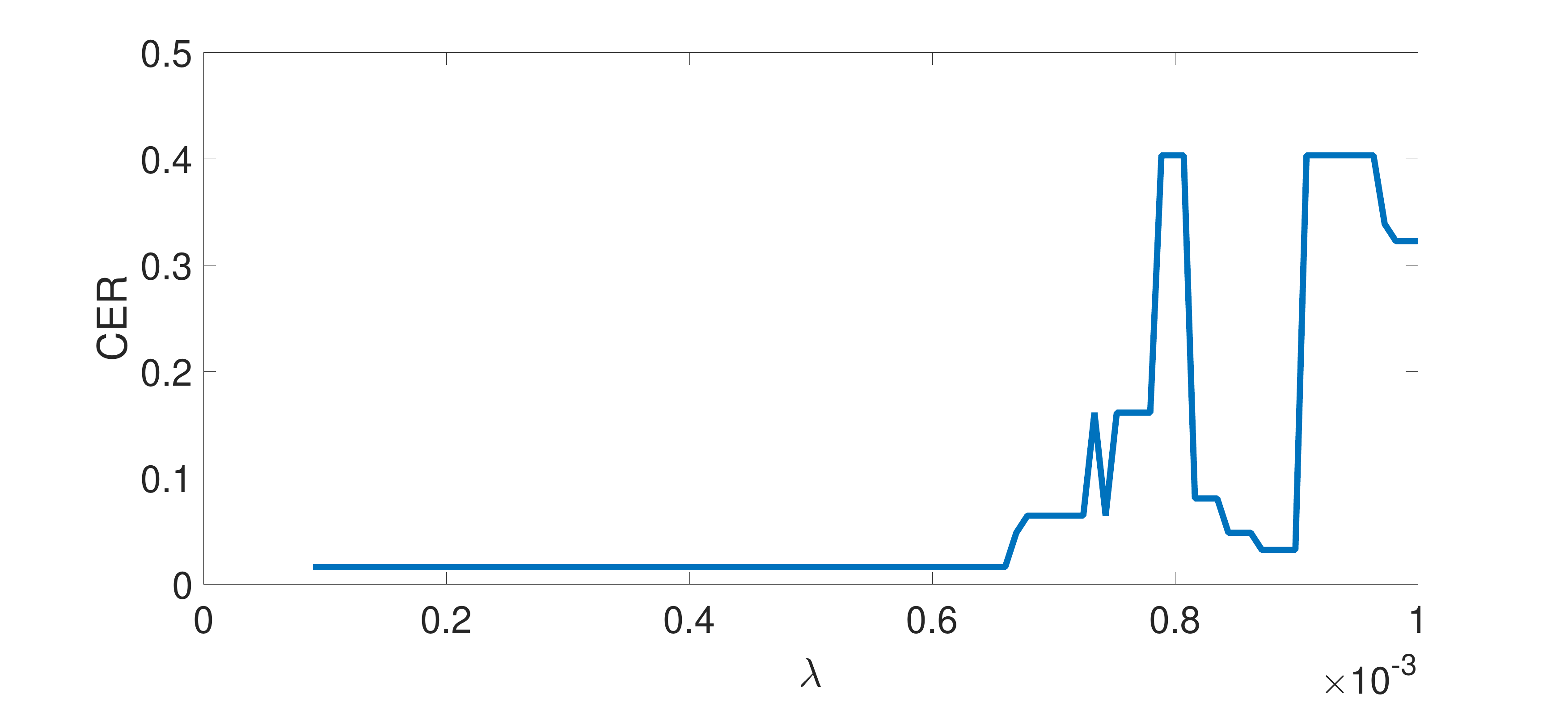}
        \caption{Median of the CER for different values of $\lambda$}
        \label{fig:median_r}
    \end{subfigure}
    \caption{Average CER for different values of $\lambda$ for lymphoma dataset. The median regularization path is much smoother than the mean regularization path because the median is not adversety affected by the $k$-means initialization of the $LW$-$k$-means algorithm.}\label{fig:lympho_r}
\end{figure}
\subsection{Choice of $\lambda$}
Let us illustrate with the example of the synthetic toy1 dataset (generated by us) which has $10$ features of which only the first $4$ are the distinguishing ones. The dataset is available from \url{https://github.com/SaptarshiC98/lwk-means}.  We take different values of $\lambda$ and iterate the $LW$-$k$-means algorithm $20$ times and take the average value of the weights assigned to different features by the algorithm. Fig. \ref{fig:toy1} shows the average value of the feature weights for different values of $\lambda$. This figure is similar to the regularization paths for the lasso \cite{tibshirani1996regression}. \par
Here the key observation is that as $\lambda$ increases, the weights decrease on an average and eventually becomes $0$. From Fig. \ref{fig:toy1}, it is evident that the $LW$-$k$-means correctly identifies that the first $4$ features are important for revealing the cluster structure for the dataset. Here, an appropriate guess for $\lambda$ might be any value between $0.1$ and $0.5$. 
It is clear from this toy example, if the dataset has a proper cluster structure, after a threshold, increasing $\lambda$ slightly does not reduce the number of feature selected.  
\begin{figure}[htb]
    \centering
    \includegraphics[width=0.4\textwidth]{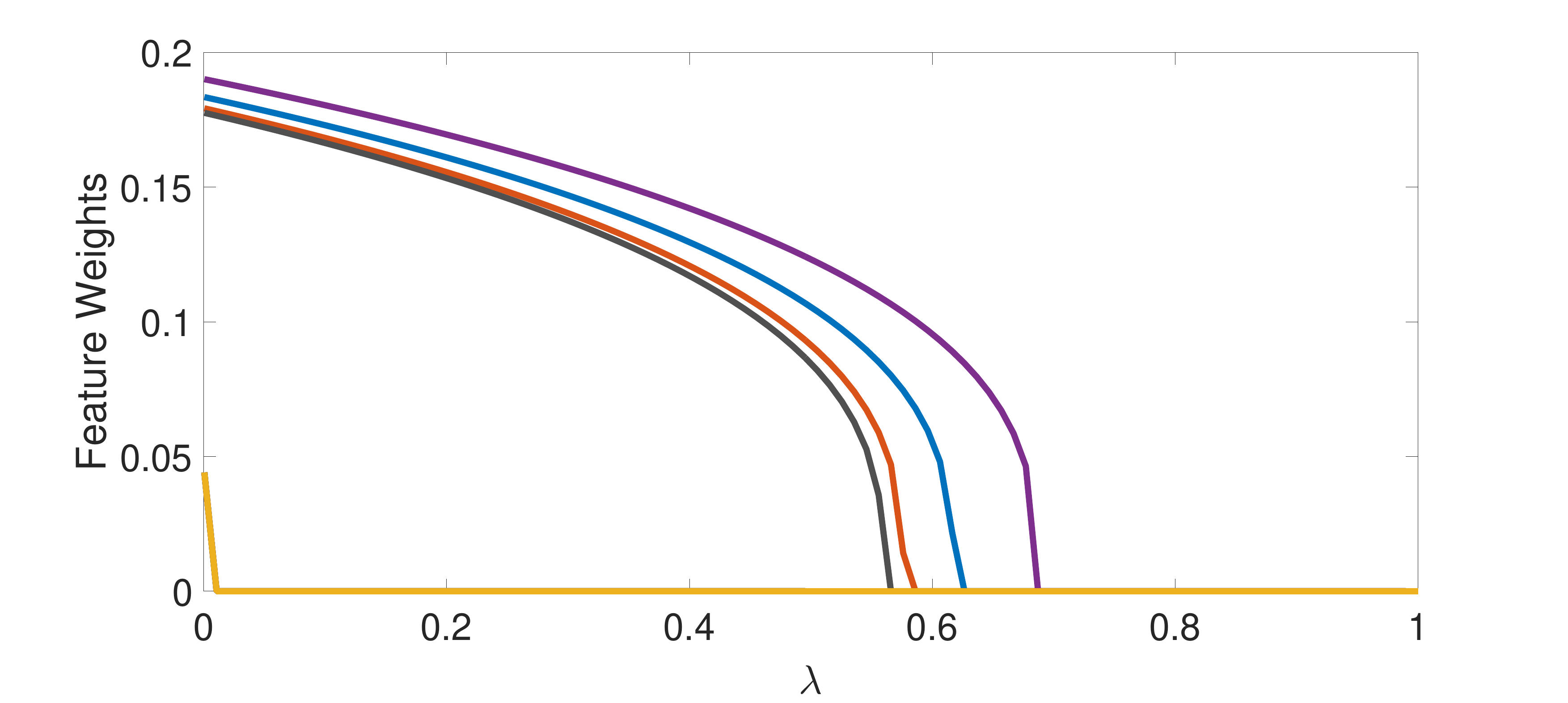}
    \caption{Mean regularization paths for dataset toy1}
    \label{fig:toy1}
\end{figure}
\subsection{Experimental Results on Real-life Datasets}
\subsubsection{Description of the Datasets}
The datasets are collected from the Arizona State University (ASU) Repository (\url{http://featureselection.asu.edu/datasets.php}), Keel Repository \cite{alcala2010keel}, and The UCI Machine Learning Repository \cite{Lichman:2013}. In Table \ref{tab des}, a summary description of the datasets is provided. The $COIL_2$, $ORL_2$, $YALE_2$ datasets are constructed by taking the first  144, 20 and 22 instances from the COIL20, ORL, and Yale image datasets respectively. The Breast Cancer and Lung Cancer datasets were analyzed and grouped into two classes in \cite{yousefi2009reporting}. A description of all the genomic datasets can be found in \cite{jin2016influential}.  
\begin{table}[htb]
\centering
\caption{Description of the Real-life Datasets}
\label{tab des}
\begin{tabular}{ccccc}
\hline

Dataname  & Source & $k$ & $n$ & $p$  \\
\hline
Brain Cancer &    Pomeroy \cite{jin2016influential}&    5 &    42 & 5597\\
Leukemia &    Gordon et al. \cite{gordon2002translation} &    2 &    72 &    3571 \\
Lung Cancer & Bhattacharjee et al.  \cite{bhattacharjee2001classification}   & 2 & 203 &    12,600\\
Lymphoma &    Alizadeh et al. \cite{alizadeh2000distinct} & 3 & 62 &    4026\\
SuCancer & Su et al. \cite{jin2016influential} & 2 & 174 & 7909\\
Wine &  Keel &    3    & 178 & 13\\
$COIL_5$ & ASU &    5 &    360 &    1024\\
$ORL_2$ &    ASU &    2 &    20 &    1024\\
$YALE_2$ & ASU &    2 &    22 &    1024\\
ALLAML & ASU    & 2 &    72 & 7129\\
Appendicitis & Keel & 2 & 106 & 7\\ 
WDBC & Keel & 2 & 569 & 30\\
GLIOMA & ASU & 4 & 50 & 4434

\\ \hline
\end{tabular}
\end{table}

\subsubsection{Performance Index}
For comparing the performance of various algorithms on the same dataset, we use the {\it Classification Error Rate (CER)} \cite{friedman2001elements} between two partitions $\mathcal{T}_1$ and $\mathcal{T}_2$ of the patterns as the cluster validation index. This index measures the mismatch between two partitions of a given set of patterns with a value 0 representing no mismatch and a value 1 representing complete mismatch.
\subsubsection{Computational Protocols}
The following computational protocols were followed during the experiment.\par 
\textit{Algorithms under consideration}: The $LW$-$k$-means algorithm, the $k$-means algorithm \cite{Mac}, $WK$-means algorithm \cite{huang2005automated}, the IF-HCT-PCA algorithm \cite{jin2016influential} and the sparse $k$-means algorithm \cite{witten2010framework}. \par
We set the value of $\beta$ to 4 for both the $LW$-$k$-means and $WK$-means algorithms throughout the experiments. The value of $\lambda$ was chosen by performing some hand-tuned experiments. To choose the value of $\alpha$, we first run the $k$-means algorithm until convergence. Then use the value of $\alpha=\frac{1}{\Bigg[\sum_{l=1}^p\frac{1}{[\beta D_l]^{\frac{1}{\beta-1}}}\Bigg]^{\beta-1}}.$ Here $D_l=\sum_{i=1}^{n}\sum_{j=1}^{k}u_{i,j} d(x_{i,l},z_{j,l})$. This is the value of the Lagrange multiplier in $WK$-means algorithm \cite{huang2005automated}.\par 
\textit{Performance comparison}: For each of the last three algorithms, we start with a set of randomly chosen centroids and iterate until convergence. We run each algorithm independently $20$ times on each of the datasets and calculate the CER. We standardized the datasets prior to applying the algorithms for all the five algorithms. 

\begin{table*}[htb]
\centering
\caption{CER for Synthetic Datasets}
\label{tab synth}
\begin{tabular}{crrrrr}
\hline

Datasets & \multicolumn{1}{c}{$LW$-$k$-means($\lambda$)} & \multicolumn{1}{c}{$WK$-means} & \multicolumn{1}{c}{$k$-means} & \multicolumn{1}{c}{IF-HCT-PCA} & \multicolumn{1}{c}{Sparse $k$-means} \\
\hline
s1 &     \textbf{0} (0.04) &    0.0642 &     0.2012    & 0.3333&\textbf{0}\\
s2    &   \textbf{0}(0.02)&    0.1507&    0.1398&    0.34&0\\
s3&    \textbf{0}(0.02)    &0.0401&    0.0865    &0.6667& \textbf{0}\\
s4&    \textbf{0}(0.007)    &0.087    &0.2000&    0.3167&\textbf{0}\\
s5&\textbf{    0}(0.007)    &0.1065&    0.1172&    0.3267&\textbf{0}\\
s6 &\textbf{0} (0.002)   & 0.0465 &    0.1537&    0.3567&\textbf{0}\\
s8 &     \textbf{0}(0.0005) &    0.1272 &    0.0653    &0.3067&\textbf{0}\\
hd6&    \textbf{0} (0.005)   & 0.2567 &    0.3062 &    0.34&\textbf{0}\\
sim1 &    \textbf{0}(0.1)    & 0.0203&    0.0452&    0.333&\textbf{0}\\
f1 & \textbf{0.0267}(0.0019) & 0.6158 & 0.6138 &0.3700& 0.5938333\\
f5  & \textbf{0.0100}(0.0006) & 0.6337 & 0.6260 & 0.2767 & 0.5328333

\\ \hline
\end{tabular}
\end{table*}

\begin{table*}[htb]
\centering
\caption{CER for Real-life Datasets}
\label{tab real}
\begin{tabular}{crrrrr}
\hline
Datasets & \multicolumn{1}{c}{$LW$-$k$-means($\lambda$)} & \multicolumn{1}{c}{$WK$-means} & \multicolumn{1}{c}{$k$-means} & \multicolumn{1}{c}{IF-HCT-PCA} & \multicolumn{1}{c}{Sparse $k$-means}  \\
\hline
Brain &    \textbf{0.2381}(0.0005) &    0.4452 &    0.2865 &    0.2624& 0.2857\\
Leukemia &    \textbf{0.0278}(0.0005) &    0.2419 &    0.2789 &    0.0695& 0.2778\\
Lung Cancer &    \textbf{0.2167}(0.000162) &    0.4672 &    0.4361 &    \textbf{0.2172} & 0.3300\\
Lymphoma &    \textbf{0.0161}(0.0006) &    0.3266 &    0.3877 &    0.0657&0.2741 \\
SuCancer & \textbf{0.4770}(0.0003) &    0.4822 &    0.4772&    0.5000&\textbf{0.4770}\\
Wine &    \textbf{0.0506}(1) &    0.0896    & 0.3047 &    0.1404&\textbf{0.0506}  \\
$COIL_5$ &    0.4031(0.001) &    0.4365 &    0.4261 &    0.4889 & \textbf{0.3639} \\
$ORL_2$ &    \textbf{0.0500}(0.005) &    0.1053 &    0.1351 &    0.3015  & 0.0512\\
$YALE_2$ &    \textbf{0.1364}(0.002) &    0.1523 &    0.1364 &    0.4545 & \textbf{0.1364}\\
ALLAML &    \textbf{0.2500}(0.0002) &    0.3486 &    0.2562 &    0.2693 & 0.2546\\
Appendicitis & 0.1981(0.17) & 0.3642 &    0.3156 &    \textbf{0.1509} & 0.1905\\
WDBC & \textbf{0.0756}(0.0001) &    0.0758 & 0.0901 & 0.1494 & 0.0810\\
GLIOMA & \textbf{0.4}(0.00051)   &0.424 &    0.442 &    0.6    & \textbf{0.4}

\\ \hline
\end{tabular}
\end{table*}
\subsection{Discussions}

In this section, we discuss some of the results obtained by using $LW$-$k$-means algorithm for clustering various datasets. In Tables \ref{tab synth} and \ref{tab real}, we report the mean CER obtained by  $LW$-$k$-means, $WK$-means, $k$-means, IF-HCT-PCA and sparse $k$-means. The values of $\lambda$ for $LW$-$k$-means are  also mentioned in both Tables \ref{tab synth} and \ref{tab real}.\par 
In Table \ref{tab synth}, we report the mean CER obtained by  $LW$-$k$-means, $WK$-means, $k$-means, IF-HCT-PCA, and sparse $k$-means. It is evident from Table \ref{tab synth} that the $LW$-$k$-means outperforms three of the state of the art algorithms (except sparse $k$-means) in all the synthetic datasets. 
Though the sparse $k$-means and $LW$-$k$-means give the same CER for the synthetic datasets, the time taken by sparse $k$-means is much more compared to $LW$-$k$-means. Also for some of the synthetic datasets, sparse $k$-means fails to identify all the relevant feature as discussed in section \ref{fs}.  \par
As revealed from Table \ref{tab real}, the $LW$-$k$-means outperforms the IF-HCT-PCA in 11 of the 13 real-life datasets. In Table \ref{tab real compar}, we note the average time taken by each of the $LW$-$k$-means, IF-HCT-PCA and sparse $k$-means. Computation of the threshold by Higher Criticism thresholding increases the runtime of the IF-HCT-PCA algorithm. Also the computation of the tuning parameter via the gap statistics increases the runtime of the sparse $k$-means algorithm. We also note the average number of selected features for the three algorithms in Table \ref{tab real compar}. It is clear from Table \ref{tab real compar}, $LW$-$k$-means also achieves better results in much lesser time compared to that of IF-HCT-PCA. \par
From Table \ref{tab real}, it can be seen that the $LW$-$k$-means outperforms the sparse $k$-means in all the 6 microarray datasets. For the other datasets, $LW$-$k$-means and sparse $k$-means give copmarable results. Also, it is clear from Table \ref{tab real compar}, $LW$-$k$-means achieves it in much lesser time compared to sparse $k$-means. Also note from Table \ref{tab real compar}, the sparse $k$-means gives non-zero weights to all the features except for $YALE_2$ and $ORL_2$ datasets. Thus, in effect, for all the other datasets, sparse $k$-means {\it does not} perform feature selection. It can also be seen that $LW$-$k$-means achieves almost the same level of accuracy using much smaller number of features for the two aforementioned datasets. \par

\subsection{Discussions on Feature Selection}
\label{fs}
In this section, we compare the feature selection aspects between $LW$-$k$-means, IF-HCT-PCA and sparse $k$-means algorithms. We only discuss compare the three algorithms for synthetic datasets, since the importance of each feature is known beforehand. \par
Before we proceed, we define a new concept called the {\it ground truth relevance vector} of a dataset. The ground truth relevance vector of a dataset $\mathcal{D}$ is defined as, $\mathcal{T_l}=(t_1,\dots ,t_p)$, where $t_i=1$ if $i^{th}$ feature is important in revealing the cluster structure of the dataset, $t_i=0$, otherwise. In general, this vector is not known beforehand. The objective of any feature selection algorithm is to estimate it.\par

\begin{table*}
\centering
\caption{Comparison between $LW$-$k$-means and IF-HCT-PCA}
\label{tab real compar}
\begin{tabular}{|c|rrr|rrr|}
\hline
Datasets & \multicolumn{3}{c|}{Number of Selected Features} & \multicolumn{3}{c|}{Time (in seconds)}  \\
\hline
 & $LW$-$k$-means & IF-HCT-PCA & Sparse $k$-means & $LW$-$k$-means & IF-HCT-PCA & Sparse $k$-means\\
 \hline
Brain &    14 &    429 & 5597 &    2.407632 &    186.951822 &  324.26  \\
Leukemia &    28 &        213 & 3571 &  1.008672 &    48.983883 & 159.44\\
Lung Cancer &    148 & 418    & 12600 &    1.542459  &    229.079416 &  2225.28 \\
Lymphoma &    32 &44    & 4026 &    1.542459 &    60.122838 &  184.23 \\
SuCancer & 7909 & 6 &7909 & 236.310317 & 805.546843 &964.39 \\
Wine &   13  & 4    &13    &0.219742  &273.263245 &4.49      \\
$COIL_5$ &    332.2 &    441 & 1024 &    4.661402 & 205.827235 &    480.38  \\
$ORL_2$ &92     &324     &148 &    0.156323  & 14.038397 &  43.41\\
$YALE_2$ &    33 &    31 &159 &0.204668 &    229.513561  &43.45\\
ALLAML &    357 &    213 & 7129&    1.008672 & 48.983883 & 423.25\\
GLIOMA & 77 & 50 & 4358 & 2.15 & 164.14 & 199.04\\
Appendicitis & 5 & 7 & 7 & 2.421305 & 110.572437 &  2.87 \\
WDBC & 30 & 13 & 30 & 0.510246 & 118.152659 & 21.05
\\ \hline
\end{tabular}
\end{table*}
Similarly we define {\it relevance vector} of a feature selection algorithm $\mathcal{A}$ and a dataset $\mathcal{D}$. It is a binary vector assigned by feature selection algorithm $\mathcal{A}$ to the dataset $\mathcal{D}$ and is defined by,
$\mathcal{T_l^A}=(t_1,\dots ,t_p)$, where $t_i=1$ if $i^{th}$ feature is selected by algorithm $\mathcal{A}$, $t_i=0$, otherwise.\par
For the synthetic datasets, we already know the ground truth relevance vector for these datasets. We use Matthews Correlation Coefficient (MCC) \cite{matthews1975comparison} to compare between the ground truth relavance vector and the relevance vector assigned by the algorithms $LW$-$k$-means, IF-HCT-PCA and sparse $k$-means.  MCC lies between $−1$ and $+1$. A coefficient of $+1$ represents a perfect agreement between the ground truth and the algorithm with respect to feature selection, $−1$ indicates total disagreement between the same and 0 denotes no better than random feature selection.
The MCC between the ground truth relevance vector and the relevance vector assigned by the algorithms $LW$-$k$-means, IF-HCT-PCA, and sparse $k$-means is shown in Table \ref{mathew}. From Table \ref{mathew}, it is clear that $LW$-$k$-means correctly identifies all the relevant features and thus leads to an MCC of +1 for each of the synthetic datasets, whereas, IF-HCT-PCA performs no better than a random feature selection. For the sparse $k$-means algorithm, it identifies only a subset of the relevant features as important for datasets s2, s3, s4, s5, s6, s7 and correctly identifies all of the features in only datasets s1, hd1, and sim1. Also for datasets f1 and f5, the sparse $k$-means algorithm performs no better than random selection of the features.

\begin{table}[htb]
\centering
\caption{Matthews Correlation Coefficient For Synthetic Datasets}
\label{mathew}
\begin{tabular}{crrr}
\hline
Datasets & \multicolumn{1}{c}{$LW$-$k$-means} & \multicolumn{1}{c}{IF-HCT-PCA}&\multicolumn{1}{c}{Sparse $k$-means}  \\
\hline
s1 & 1 & 0.0870&1\\ 
s2 & 1 & 0.0380&0.7535922\\ 
s3 & 1 & -0.0611&0.9594972\\
s4 & 1 &  0.0072 &0.5016978\\
s5 & 1 & -6.3668e-04&0.6276459\\
s6 & 1 & 0.0547&0.6813851\\
s7 & 1 & 0.0345&0.6707212\\
hd6 & 1 & 0.0048&1\\
sim1& 1&0.1186&1\\
f1&1&0.2638& 0.01549587\\
f5 & 1 & 0.3413 & 0.02240979\\
\hline
\end{tabular}

\end{table}

\section{Simulation Study}\label{simulation studies}
In the following example, we compare the $WK$-means estimate of weights with those of the $LW$-$k$-means estimates.

\subsection{Example 1}
We simulated $50$ datasets each of which have $4$ clusters consisting of $100$ points each. Let $\mathbf{X}_i$ be a random point from the $i^{th}$ cluster, where $i \in \{1,2,3,4\}$. Let $\mathbf{X}_i=(X^{(i)}_1,X^{(i)}_{2})'$. The dataset is simulated as follows.
\begin{itemize}
    \item $X^{(1)}_1$ are i.i.d from $\mathcal{N}(0,1)$.
    \item $X^{(1)}_2$ are i.i.d from $\mathcal{N}(0,1)$.
    \item $X^{(2)}_1$ are i.i.d from $\mathcal{N}(7,1)$.
    \item $X^{(2)}_2$ are i.i.d from $\mathcal{N}(2,1)$.
    \item $X^{(3)}_1$ are i.i.d from $\mathcal{N}(13,1)$.
    \item $X^{(3)}_2$ are i.i.d from $\mathcal{N}(-2,1)$.
    \item $X^{(4)}_1$ are i.i.d from $\mathcal{N}(19,1)$.
    \item $X^{(4)}_2$ are i.i.d from $Unif(-10,10)$.
\end{itemize}
We run the sparse $k$-means and $LW$-$K$-means algorithms 10 times on each dataset and noted the average of the feature weights. We do this procedure on each of the 50 datasets.  In Fig. \ref{fig:hist}, we plot the histogram for the features $x_1$ and $x_2$. From Fig. \ref{fig:hist}, it is clear that feature $x_1$  has a clusture structure and feature $x_2$ does not. In Fig. \ref{fig:eg1}, we plot the boxplot of the average weights assigned by the $LW$-$k$-means and sparse $k$-means algorithms to features $x_1$ and $x_2$ for all the 50 datasets. Fig. \ref{fig:eg1} shows that sparse $k$-means assigns a feature weight of 0.32 to the unimportant feature $x_2$, whereas $LW$-$k$means assigns $x_2$, zero feature weight and hence is capable of proper feature selection.
\begin{figure}[htb]
    \centering
    \begin{subfigure}{0.47\textwidth}
        \includegraphics[width=\textwidth]{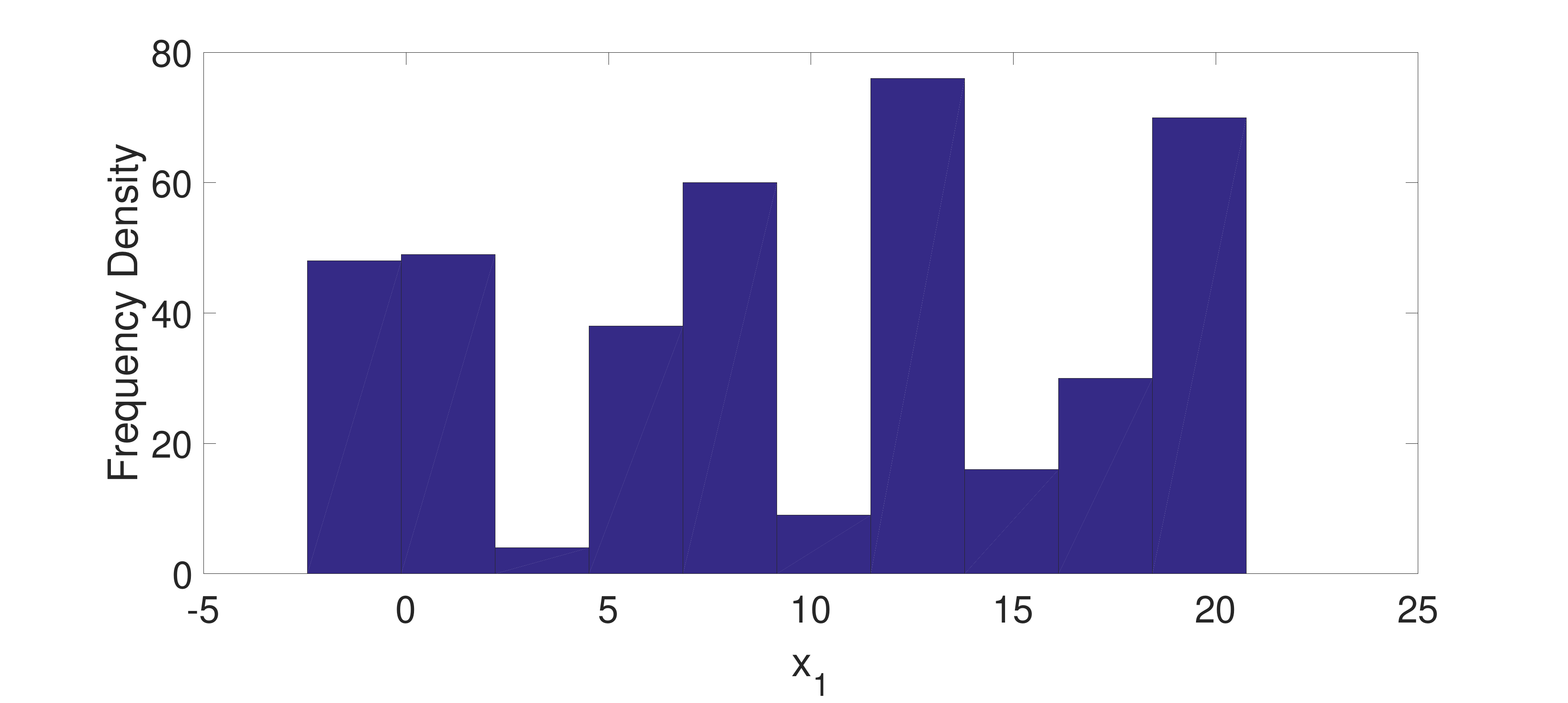}
        \caption{Feature $x_1$}
        \label{fig:hist1}
    \end{subfigure}
    \begin{subfigure}{0.47\textwidth}
        \includegraphics[width=\textwidth]{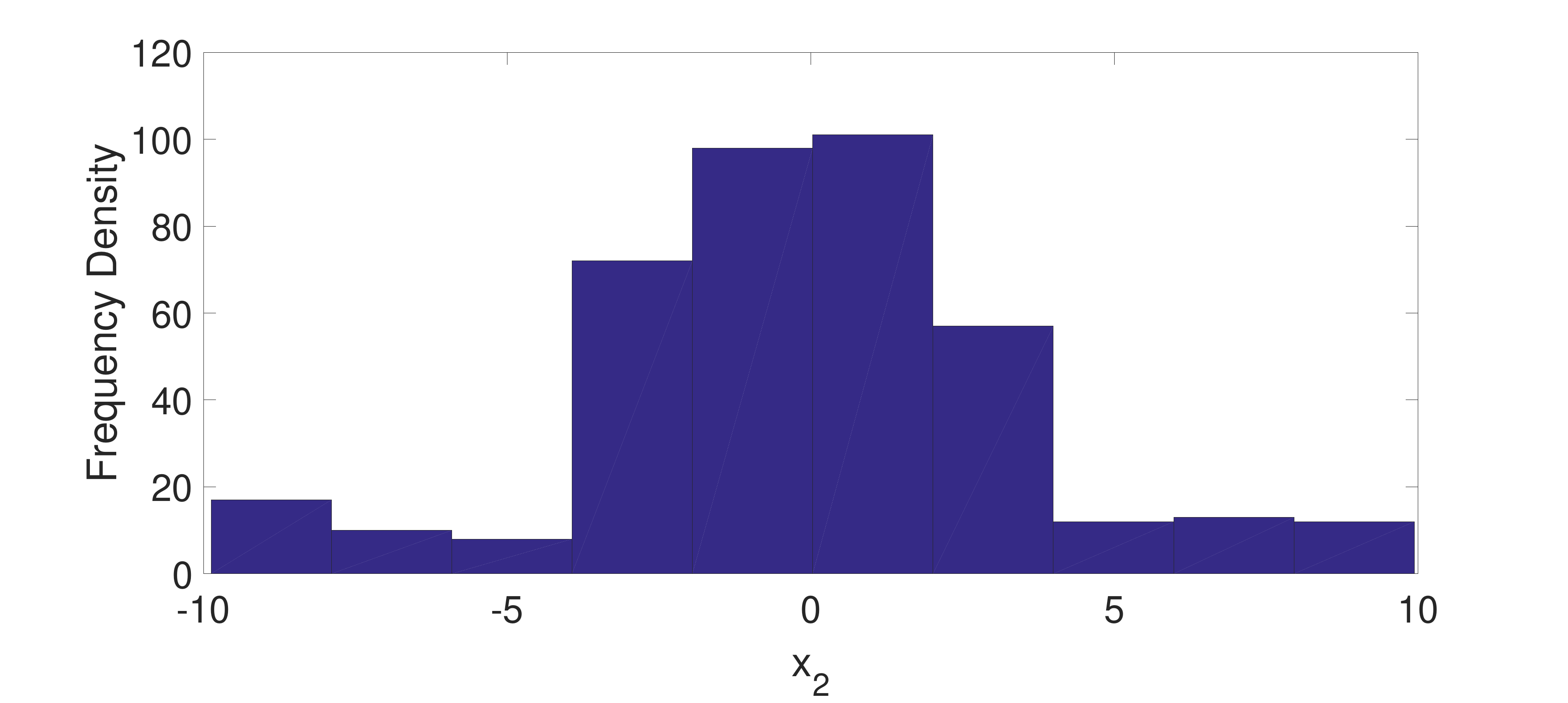}
        \caption{Feature $x_2$}
        \label{fig:hist1}
    \end{subfigure}
    \caption{Histogram of features $x_1$ and $x_2$ of the $W2$ dataset. Clearly Feature $x_1$ has a clusture structure and feature $x_2$ doesn't.}\label{fig:hist}
\end{figure}
\begin{figure}[htb]
    \centering
    \begin{subfigure}{0.47\textwidth}
        \includegraphics[width=\textwidth]{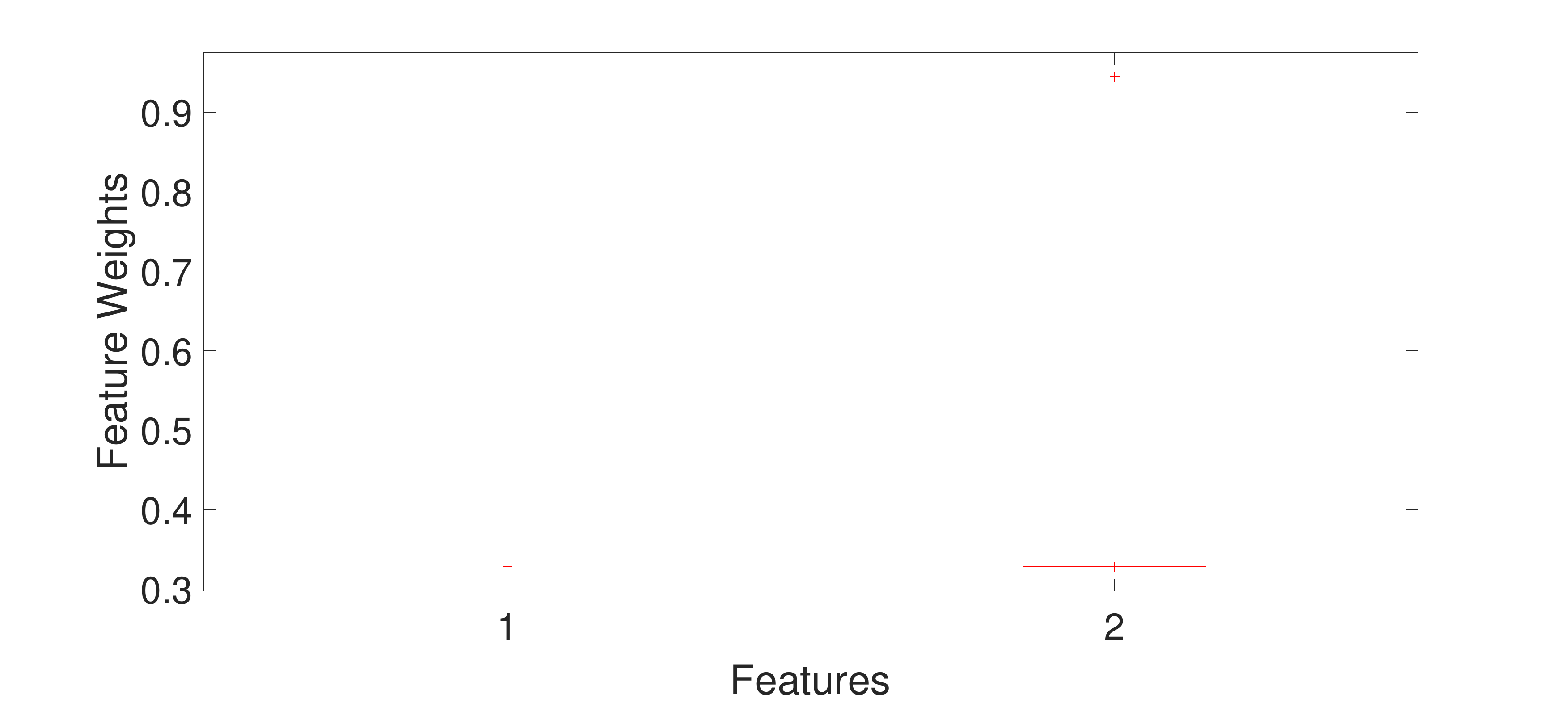}
        \caption{Sparse $k$-means}
        \label{}
    \end{subfigure}
    \begin{subfigure}{0.47\textwidth}
        \includegraphics[width=\textwidth]{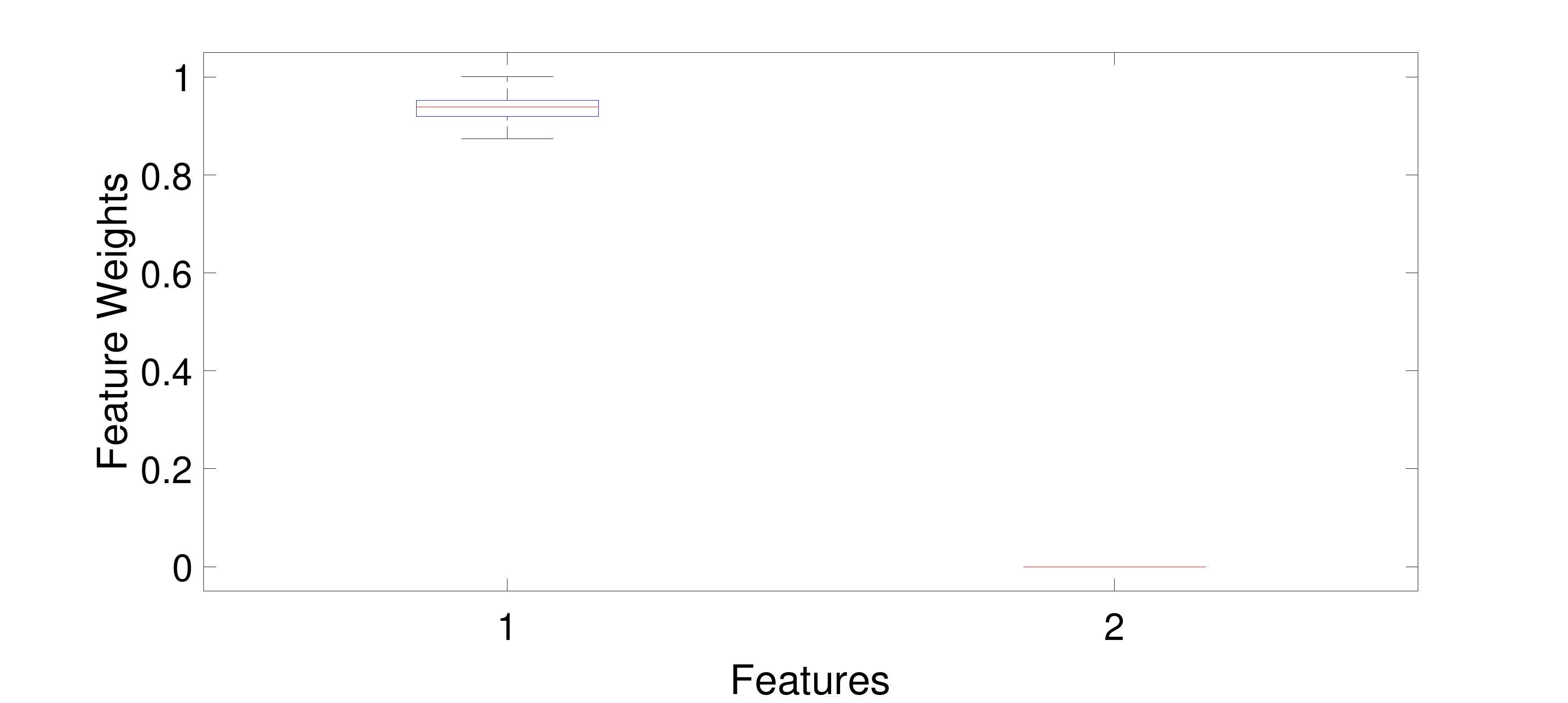}
        \caption{$LW$-$k$-means}
        \label{fig:hist2}
    \end{subfigure}
    \caption{Boxplot of the average weights assigned by the $LW$-$k$-means and sparse $k$-means algorithms to features 1 and 2 for all the 50 datasets. The boxplot shows that sparse $k$-means assigns a feature weight of 0.32 to the unimportant feature $x_2$, whereas, $LW$-$k$means assigns zero feature weight to $x_2$ and hence is capable of proper feature selection. }\label{fig:eg1}
\end{figure}
\begin{figure}
    \centering
    \begin{subfigure}{.47\textwidth}
        \includegraphics[width=\linewidth]{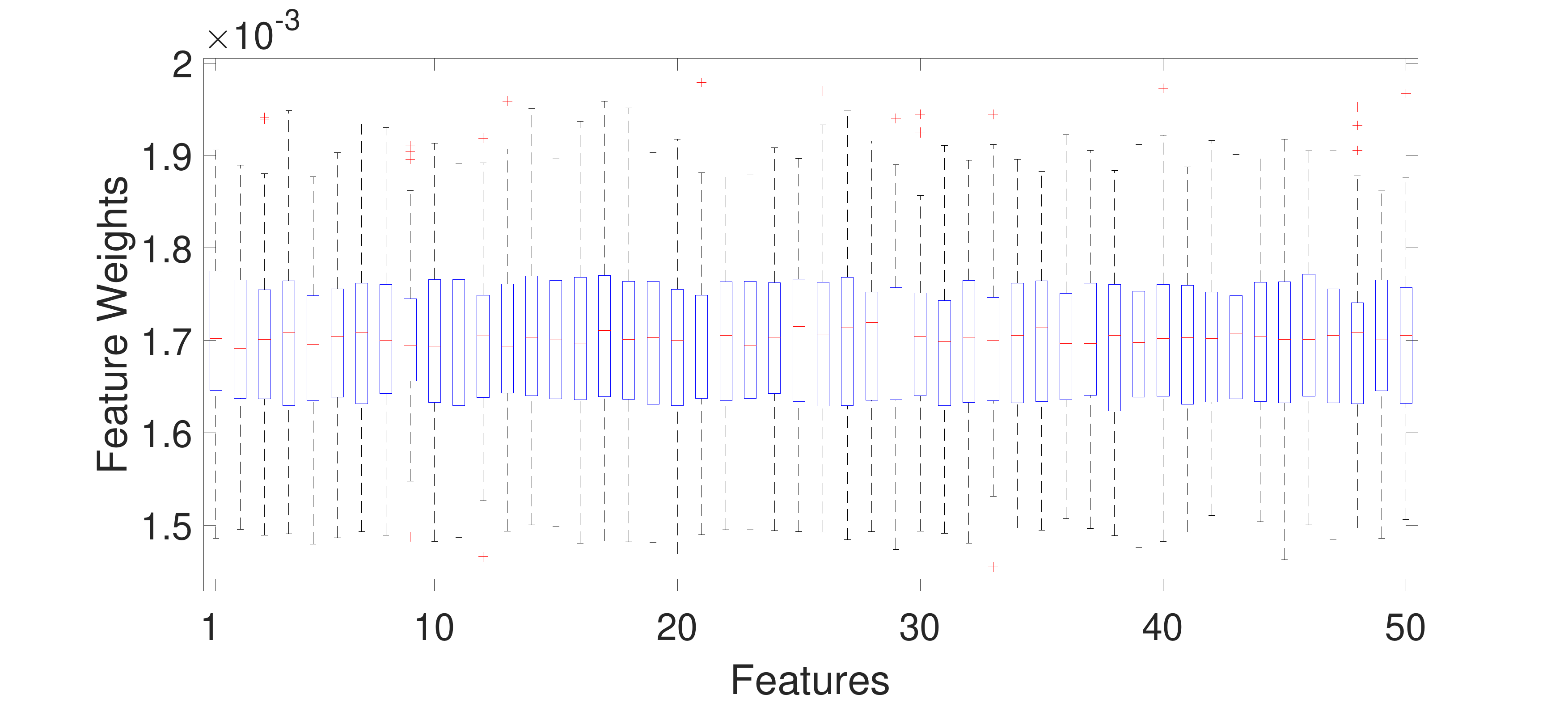}
        \caption{Sparse $k$-means}
        \label{fig:huang_mean}
    \end{subfigure}
     add desired spacing between images, e. g. ~, \quad, \qquad, \hfill etc. 
      (or a blank line to force the subfigure onto a new line)
    \begin{subfigure}{.47\textwidth}
        \includegraphics[width=\linewidth]{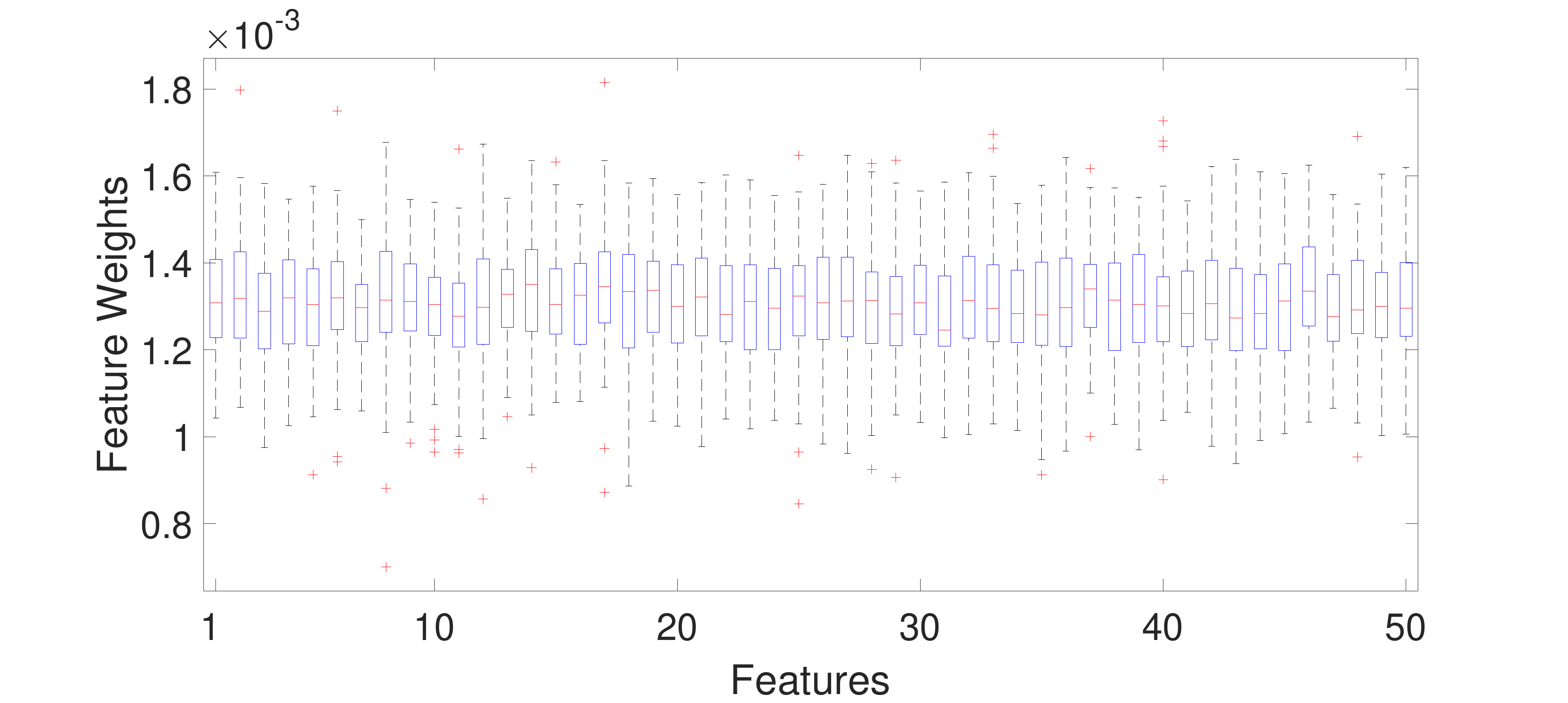}
        \caption{$LW$-$k$-means}
        \label{fig:lwk_mean}
    \end{subfigure}
    \caption{Boxplot of the average weights assigned by the $LW$-$k$-means and $WK$-means algorithms to features 1 and 2 for all the 50 datasets. The boxplot shows that sparse $k$-means gives  }\label{fig:sim1}
\end{figure}
\begin{figure}
    \centering
    \begin{subfigure}{0.47\textwidth}
        \includegraphics[width=\textwidth]{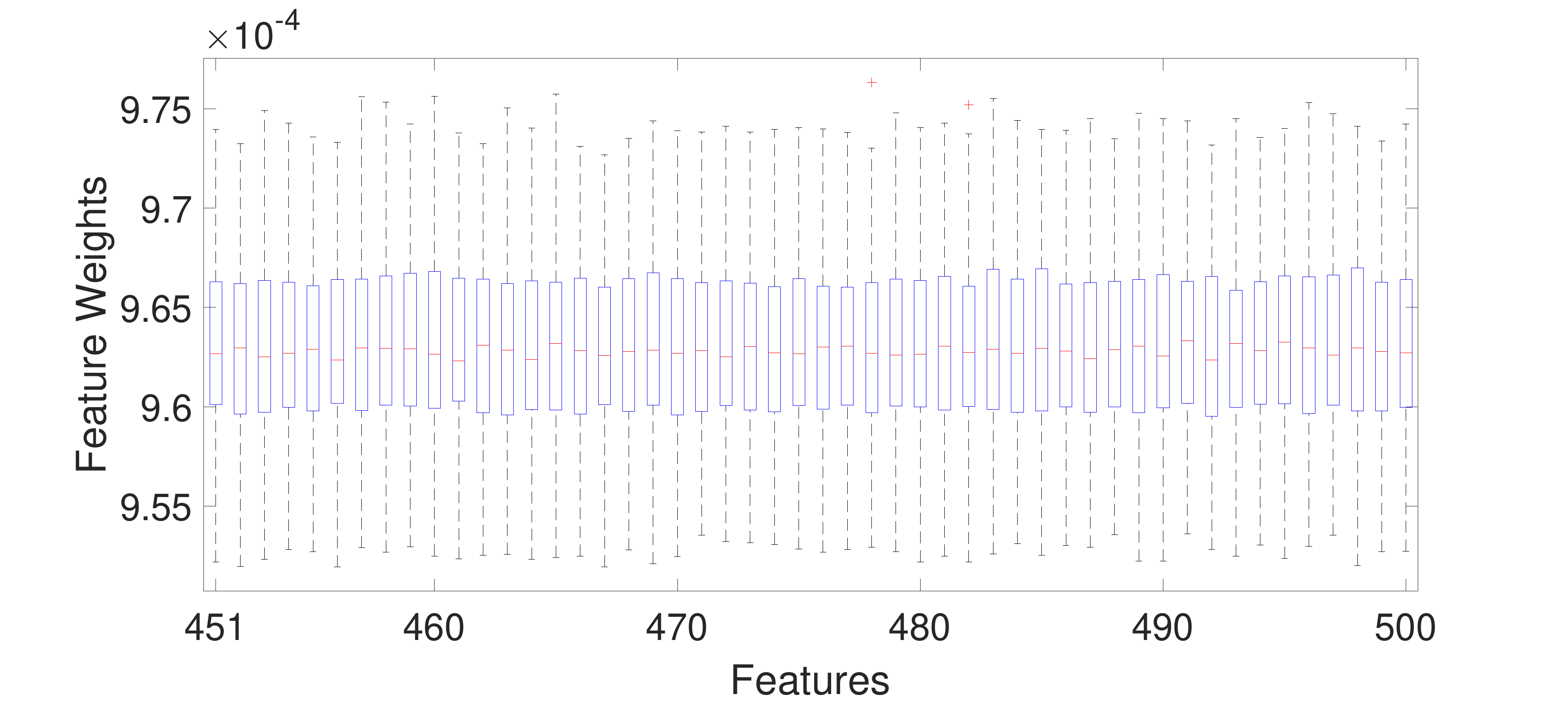}
        \caption{Boxplot of the average weights assigned by the $WK$-means algorithm to features 500 to 550 for all the 70 datasets.}
        \label{fig:huang_mean_re}
    \end{subfigure}
     add desired spacing between images, e. g. ~, \quad, \qquad, \hfill etc. 
      (or a blank line to force the subfigure onto a new line)
    \begin{subfigure}{0.47\textwidth}
        \includegraphics[width=\textwidth]{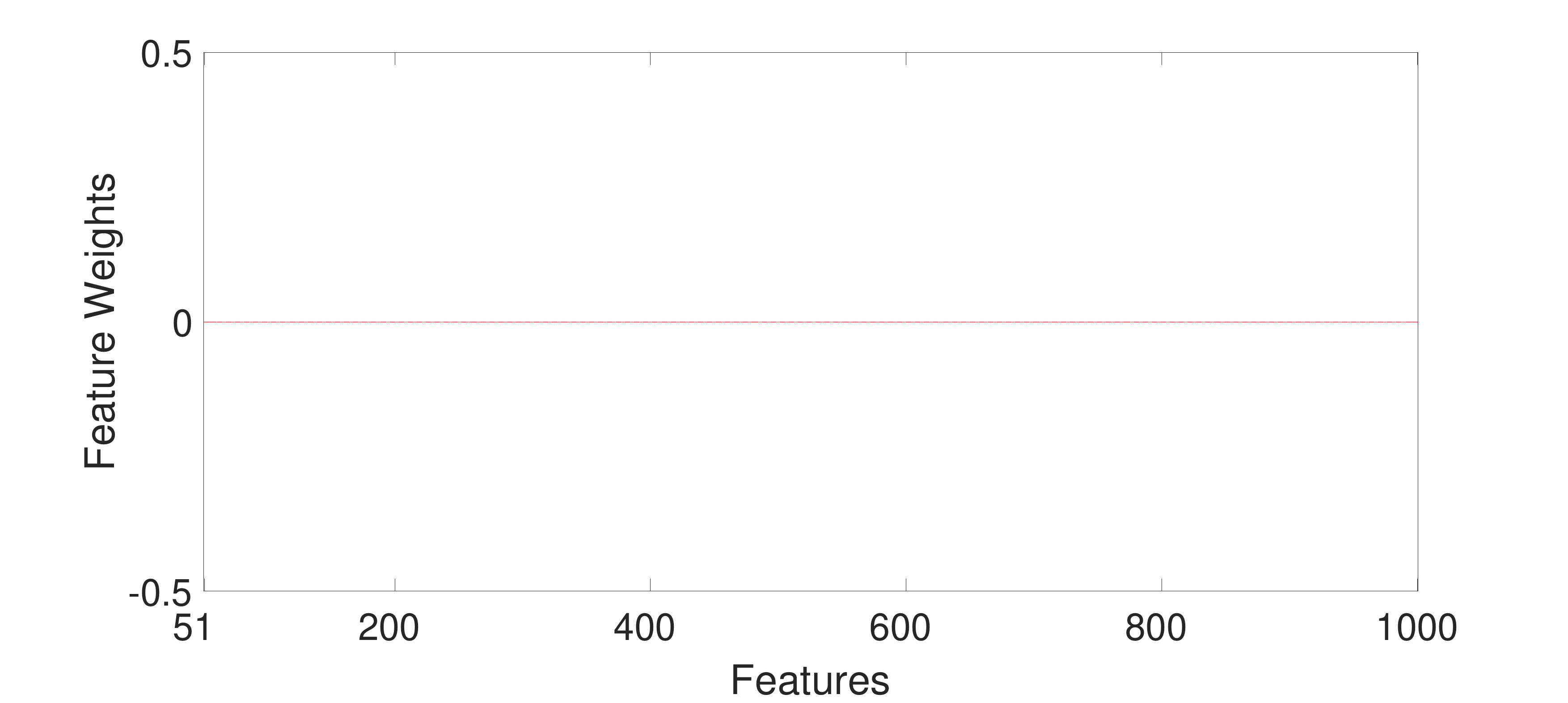}
        \caption{Boxplot of the average weights assigned by the $WK$-means algorithm to features 51 to 1000 for all the 70 datasets.}
        \label{fig:lwk_mean_re}
    \end{subfigure}
    \caption{Boxplot of the average weights assigned by the $LW$-$k$-means and $WK$-means algorithms to features 1 to 50 for all the 70 datasets. The boxplot show a lesser variability for the $LW$-$k$-means weights compared to the $WK$-means weights. }\label{fig:sim1_re}
\end{figure}
\subsection{Example 2}
 We simulated $70$ datasets each of which have $3$ clusters consisting of $100$ points each. Let $\mathbf{X}_i$ be a random point from the $i^{th}$ cluster, where $i \in \{1,2,3\}$. Let $\mathbf{X}_i=(X^{(i)}_1,\dots ,X^{(i)}_{1000})$. 
The datasets are simulated as follows.
\begin{itemize}
    \item $X^{(1)}_j$ are i.i.d from $\mathcal{N}(0,1)$ $\forall j \in \{1,\dots ,50\}$.
\item $X^{(1)}_j$ are i.i.d from $\mathcal{N}(5,1)$ $\forall j \in \{1,\dots ,50\}$.
\item $X^{(1)}_j$ are i.i.d from $\mathcal{N}(10,1)$ $\forall j \in \{1,\dots ,50\}$.
\item $X^{(i)}_j$ are i.i.d from $\chi^2_{(5)}$ $\forall j \in \{51,\dots ,1000\}$.
\item $X^{(i)}_j$ is independent of $X^{(h)}_k$ $\forall i,h \in \{1,2,3\} $ and $\forall j,k \in \{1,\dots ,1000\} $ such that $(i,j) \neq (h,k)$.
\end{itemize}

 Thus each of the datasets has only the first $50$ features relevant and the other features irrelevant.
For each of these $70 $ datasets, we run $LW$-$k$-means (with $\lambda=0.005$) and $WK$-means $40$ times and note the average (mean) weights  
assigned to different features by both the $LW$-$k$-means and $Wk$-means algorithms.\par
In Fig. \ref{fig:huang_mean}, we show the boxplot of the average weights assigned by the $Wk$-means algorithm to feature $1$ to $50$ for all the $70$ datasets. In Fig. \ref{fig:lwk_mean}, we show the corresponding boxplot for the $LW$-$k$-means algorithm. Fig. \ref{fig:sim1} clearly show a lesser variability for the weights assigned by $LW$-$k$-means compared to that of $WK$-means. In Fig. \ref{fig:sim1_re}, we plot the corresponding boxplot for the rest of the features. For space constraints, we only plotted the boxplots corresponding to features $500$ to $550$ for the $WK$-means algorithm. From Fig. \ref{fig:huang_mean_re} it is clear that the average weights assigned by $WK$-means for the irrelevant features are somewhat close to zero but not exactly zero. On the other hand, the average weights assigned by $LW$-$k$-means for the irrelevant features are exactly equal to zero as shown in fig. \ref{fig:lwk_mean_re}. 
\section{Conclusion and Future Works}
In this paper, we introduced a alternative sparse $k$-means algorithm based on the Lasso penalization of feature weighting. 
We derived the expression of the solution to the $LW$-$k$-means objective, theoretically, using KKT conditions of optimality. We also proved the convergence of the proposed algorithm. Since $LW$-$k$-means does not make any distributional assumptions of the given data, it works well even when the irrelevant features does not follow a normal distribution. We validated our claim by performing detailed experiments on 9 synthetic and 13 real-life datasets. We also undertook a simulation study to find out the variability of the feature weights assigned by the $LW$-$k$-means and $WK$-means and found that $LW$-$k$-means always assigns zero weight to the irrelevant features for the appropriate value of $\lambda$.  We also proposed an objective method to choose the value of the tuning  parameter $\alpha$ in the algorithm.\par
Some possible extension of the proposed method might be to extend it to fuzzy clustering, to give a probabilistic interpretation of the feature weights assigned by the proposed algorithm and also to use different divergence measures to enhance the performance of the algorithm. One can also explore the possibility to prove the strong consistency of the proposed algorithm for different divergence measures, prove the local optimality of the obtained partial optimal solutions and also to choose the value of $\lambda$ in an user independent fashion.

 \appendices
 \section{Proofs of Various Theorems an Lemmas of the paper}
\subsection{Proof of Theorem \ref{convex}}
\label{convex proof}
\begin{proof}
Clearly,
$P(\mathcal{W})=h(\mathcal{W})+g(\mathcal{W})$, where
$$h(\mathcal{W})=\frac{1}{n}\sum_{l=1}^p w_l^\beta D_l^0-\alpha\sum_{l=1}^pw_l$$
and 
$$g(\mathcal{W})=\frac{\lambda}{n p^2}\sum_{l=1}^p  |w_l|D_l^0.$$
Now, $\frac{\partial^2 h}{\partial w_l^2}=\beta (\beta-1)w_l^{\beta-2} D_l \geq 0$. Hence $h(\mathcal{W})$ is convex. It is also easy to see that $g(\mathcal{W})$ is convex. Hence $P(\mathcal{W})$ being the sum of two convex functions is convex.
\end{proof}
\subsection{Proof of Theorem \ref{reformulation}}
\label{reformulation proof}
\begin{proof}
By Theorem \ref{convex}, the objective function in \ref{eq1} is convex. Let $(w_1,t_1)$ and $(w_2,t_2)$ satisfy constraints \ref{eq3} and \ref{eq4}. Let $\gamma \in (0,1)$, $t=\gamma t_1+(1-\gamma)t_2$ and $w=\gamma w_1+(1-\gamma)w_2$. Then, $t-w=\gamma (t_1-w_1)+(1-\gamma)(t_2-w_2) \geq 0$ and $t+w=\gamma (t_1+w_1)+(1-\gamma)(t_2+w_2) \geq 0. $ Hence $(w,t)$ satisfy constraints \ref{eq3} and \ref{eq4} and the constraint set of problem $P_2^*$ is convex. The Hessian of the objective function in \ref{eq2} is $H(w,t)=\frac{1}{n}\begin{bmatrix} 
\beta(\beta -1)w^{\beta-2}D & 0 \\
0 & 0
\end{bmatrix}$
which is clearly positive semi-definite. Hence the objective function of problem $P_2^*$ is convex. Thus, any local minimizer of problem $P_2^*$ is also a global minimizer.\par
Since $(w^*,t^*)$ is a local (hence global) minimizer of problem $P_2^*$, for all $(w,t)$ which satisfy Eqn \ref{eq3} and \ref{eq4}, 
\begin{equation}
    \label{eq5}
    \frac{1}{n} {w^*}^{\beta} D-\alpha w^* +\frac{\lambda}{np^2} t^* D \leq  \frac{1}{n}w^{\beta} D-\alpha w +\frac{\lambda}{np^2} t D.
\end{equation}
Taking $w=w_1^*$ and $t=|w_1^*|$ in Eqn \ref{eq5}, we get,
\begin{equation}
\label{eq6}
    \frac{1}{n}{w^*}^{\beta} D-\alpha w^* +\frac{\lambda}{np^2} t^* D \leq  \frac{1}{n}{{w_1}^*}^{\beta} D-\alpha w_1^* +\frac{\lambda}{np^2} |w_1^*| D.
\end{equation}
Again, since $w_1^*$ is a solution to problem $P_1^*$,
\begin{equation}
\label{eq7}
   \frac{1}{n} {w_1^*}^{\beta} D-\alpha w_1^* +\frac{\lambda}{np^2} |w_1^*| D \leq  \frac{1}{n}{w^*}^{\beta} D-\alpha w^* +\frac{\lambda}{np^2} |w^*| D.
\end{equation}
Adding Eqn \ref{eq6} and \ref{eq7}, we get
\begin{equation}
\label{eq8}
    t^* \leq |w^*|.
\end{equation}
Again, from constraints \ref{eq3} and \ref{eq4}, we get
\begin{equation}
\label{eq9}
    t^* \geq |w^*|.
\end{equation}
Hence, from Eqn \ref{eq8} and \ref{eq9}, we get
\begin{equation}
\label{eq10}
t^* = |w^*|.    
\end{equation}
Substituting Eqn \ref{eq10} in Eqn \ref{eq6}, we get
\begin{equation}
\label{eq11}
     \frac{1}{n}{w^*}^{\beta} D-\alpha w^* +\frac{\lambda}{np^2} |w^*| D \leq \frac{1}{n} {w_1^*}^{\beta} D-\alpha w_1^* +\frac{\lambda}{np^2} |w_1^*| D.
\end{equation}
Hence from Eqn \ref{eq7} and \ref{eq11}, we get 
\begin{equation}
\label{eq12}
    \frac{1}{n} {w^*}^{\beta} D-\alpha w^* +\frac{\lambda}{np^2} |w^*| D =  \frac{1}{n}{w_1^*}^{\beta} D-\alpha w_1^* +\frac{\lambda}{np^2} |w_1^*| D.
\end{equation}
Since, Eqn \ref{eq12} is true for all $\alpha \geq 0$ and $\lambda \geq 0$, $ w^*= w_1^*$
\end{proof}
\subsection{Proof of Theorem \ref{solve alt}}
\label{solve alt proof}
\begin{proof}
The Lagrangian for the single-dimensional optimization problem $P_2^*$ is given by,
$$\mathcal{L}(w,t,\lambda_1,\lambda_2)=\frac{1}{n}w^{\beta} D-\alpha w +\frac{\lambda}{np^2} t D-\lambda_1(t-w)-\lambda_2(t+w).$$
The Karush-Kuhn-Tucker (KKT) necessary conditions of optimality for $(w^*,t^*)$ is given by,
$$\frac{\partial \mathcal{L}}{\partial w}=0 $$
\begin{equation}
\label{2.3.1}
  \implies \frac{1}{n}\beta {w^*}^{\beta-1}D=\alpha -\lambda_1+\lambda_2  
\end{equation}

$$\frac{\partial \mathcal{L}}{\partial t}=0 $$
\begin{equation}
    \label{2.3.2}
    \implies \frac{\lambda}{np^2} D =\lambda_1+\lambda_2
\end{equation}
\begin{equation}
\label{2.3.3}
  t-w \geq 0.
\end{equation}
\begin{equation}
\label{2.3.4}
  t+w \geq 0  ,
\end{equation}
$$\lambda_1, \lambda_2 \geq 0.$$
\begin{equation}
\label{2.3.5}
  \lambda_1(t-w)=0 . 
\end{equation}
\begin{equation}
\label{2.3.6}
  \lambda_2(t+w)=0  .
\end{equation}
Now let us consider the following situations:\par
\textbf{Case-1} $\frac{n\alpha}{D}>\frac{\lambda}{p^2}$:\par
$$\frac{1}{n}\beta {w^*}^{\beta-1}D=\alpha -\frac{\lambda}{np^2} D+2 \lambda_2>0 \implies w>0. $$
From, Eqn \ref{2.3.3}, $t>0.$ Thus, from Eqn \ref{2.3.6}, $\lambda_2=0$. Hence, $\frac{1}{n}\beta {w^*}^{\beta-1}D=\alpha -\frac{\lambda}{np^2} D \implies w^*=\Bigg[\frac{1}{\beta}(\frac{n\alpha}{D}-\frac{\lambda}{p^2})\Bigg]^{\frac{1}{\beta-1}}.$\par
\textbf{Case-2} $\frac{n\alpha}{D} \leq \frac{\lambda}{p^2}$:\par
If $w>0$, $(t+w)>0$ which implies $\lambda_2=0.$ $\frac{1}{n}\beta {w^*}^{\beta-1}D=\alpha -\frac{\lambda}{np^2} D \leq 0 \implies w \leq 0$, which is a contradiction.\par Now if $w<0$, $(t-w)>0$ which implies $\lambda_1=0.$ From Eqn \ref{2.3.1} and \ref{2.3.2}, it is easily seen that, 
$\frac{1}{n}\beta  {{w}^*}^{\beta-1}D=\alpha +\frac{\lambda}{np^2} D \implies w \geq 0,$
which is again a contradiction. Hence the only possibility is $w=0$. Now, since $\frac{n\alpha}{D} \geq 0$, from \textbf{Case 1} and \textbf{2}, we conclude that $w^*=\Bigg[\frac{1}{\beta}S(\frac{n\alpha}{D},\frac{\lambda}{p^2})\Bigg]^{\frac{1}{\beta-1}}.$

\end{proof}
\subsection{Proof of Theorem \ref{solve p3}}
\label{solve p3 proof}
\begin{proof}
Now to solve \textit{Problem $\mathbf{P_3}$}, note that \textit{Problem $\mathbf{P_3}$} is separable in $\mathcal{W}$ i.e. we can write $P(\mathcal{W})$ as 
\begin{equation}
    \label{sepa}
    P(\mathcal{W})=\sum_{l=1}^p P_l(w_l),
\end{equation}

where $P_l(w_l)=\frac{1}{n}w_l^{\beta} D^0_l-\alpha w_l +\frac{\lambda}{np^2} |w_l| D^0_l$. Now since \textit{Problem $\mathbf{P_3}$} is separable, it is enough to solve \textit{Problem $\mathbf{P'_l}$} $\forall d \in \{1,\dots ,p\}$ and combine the solutions to solve \textit{Problem $\mathbf{P_3}$}. Here \textit{Problem $\mathbf{P'_l}$} ($d \in \{1,\dots ,p\}$) is given by, 
\begin{equation}
    \label{pd}
    minimize \text{ } P_l(w_l)=\frac{1}{n} (w_l^\beta+\frac{\lambda}{p^2} |w_l|)D_l^0-\alpha w_l  \text{ w.r.t } w_l.
\end{equation}
The Theorem follows trivially from Theorem \ref{solve 1 d}.
\end{proof}
\subsection{Proof of Theorem \ref{convergence}}
\label{convergence proof}
\begin{proof}
Let $f_m$ be the value of the objective function at the end of the $m^{th}$ iteration of the algorithm. Since each step of the inner {\it while} loop of the algorithm decreases the value of the objective function, $f_t \geq f_{t+1}$ $\forall t \in \mathbb{N}$. Again note that,  $f_t \geq 0$ $\forall t \in \mathbb{N}$. Hence the sequence $\{f_m\}_{m=1}^\infty$ is a decreasing sequence of reals bounded below by $0$. Hence, by monotone convergence theorem, $\{f_m\}_{m=1}^\infty$ converges. Now since $\{f_m\}_{m=1}^\infty$ is convergent hence Cauchy and thus $\exists N_0 \in \mathbb{N}$ such that if $n \geq N_0$, $|f_{n+1}-f_n| < \epsilon$, which is the stopping criterion of the algorithm. Thus, the $LW$-$k$-means algorithm converges in a finite number of iteration.
\end{proof}
\subsection{Proof of Theorem \ref{bddwt}}
\label{bddwt proof}
\begin{proof}
Let $D_l$ denote the minimum value of the $k$-means objective function for only the $d^{th}$ feature of the the dataset i.e. $\{x_{1,l},\dots,x_{n,l}\}$.
 Let $\mathcal{U}^*_n$ denote the cluster assignment matrix corresponding to the optimal set of centroids $A_n=\{\mathbf{a}_1,\dots,\mathbf{a}_k\}$. Let $D^*_l=\sum_{i=1}^n\sum_{j=1}^k u^*_{ij}d(x_{i,l},z_{j,l})$. It is easy to see that $D^*_l \geq D_l$. Hence, 
$\frac{1}{D_l^*} \leq \frac{1}{D_l}$.
Thus, $$\frac{1}{\Bigg(\sum_{l=1}^p [\frac{n}{\beta {D^*}_l}]^{\frac{1}{\beta-1}}\Bigg)^{\beta-1}} \geq \frac{1}{\Bigg(\sum_{l=1}^p [\frac{n}{\beta D_l}]^{\frac{1}{\beta-1}}\Bigg)^{\beta-1}}=\alpha_n(P_n).$$
We know that $w^{(n)}_l=\Bigg[\frac{1}{\beta}S\Bigg(\frac{n \alpha(P_n)}{D^*_l},\frac{\lambda}{p^2}\Bigg)\Bigg]^{\frac{1}{\beta-1}}$. Thus,
$$w^{(n)}_l \leq \Bigg[\frac{1}{\beta}\frac{n\alpha(P_n)}{D^*_l}\Bigg]^{\frac{1}{\beta-1}}.$$
Thus, 
\begin{align*}
\sum_{l=1}^pw^{(n)}_l & \leq \sum_{l=1}^p\Bigg[\frac{1}{\beta}\frac{n\alpha(P_n)}{D^*_l}\Bigg]^{\frac{1}{\beta-1}}\\
& \leq (n\alpha(P_n))^{\frac{1}{\beta-1}}\sum_{l=1}^p\Bigg[\frac{1}{\beta D^*_l}\Bigg]^{\frac{1}{\beta-1}}\\
& =\alpha(P_n)^{\frac{1}{\beta-1}}\sum_{l=1}^p\Bigg[\frac{n}{\beta D^*_l}\Bigg]^{\frac{1}{\beta-1}}\\
&\leq 1.    
\end{align*}
The almost sure convergence of $\alpha(P_n)$ follows from the strong consistency of the $k$-means algorithm \cite{pollard1981strong}.
\end{proof}
\subsection{Proof of Lemma \ref{lem1}}
\label{lem1 proof}
\begin{proof}
If $A,B \in \xi_k$ such that $H(A,B)<\delta$, then for each $\mathbf{b} \in B,\exists \textbf{ } \mathbf{a}(\mathbf{b})\in A$ such that $|b_l-a(b)_l|<\delta$  $\forall d \in \{1,\dots,p\}$.
{\small
\begingroup
\allowdisplaybreaks
\begin{equation}
    \begin{aligned}
    &\Phi(\mathcal{W},A,P)-\Phi(\mathcal{W},B,P)\\
    &=\int min_{a \in A} \sum_{l=1}^p (w_l^\beta+\frac{\lambda}{p^2}|w_l|)(x_l-a_l)^2 P(dx)\\
    &-\int min_{b \in B} \sum_{l=1}^p (w_l^\beta+\frac{\lambda}{p^2}|w_l|)(x_l-b_l)^2P(dx)\\    
    & \leq \int max_{b \in B} \sum_{l=1}^p (w_l^\beta+\frac{\lambda}{p^2}|w_l|)[(x_l-b_l)^2-(x_l-a(b)_l)^2]P(dx)\\
    & \leq \int min_{b \in B} \sum_{l=1}^p (w_l^\beta+\frac{\lambda}{p^2}|w_l|)(2x_l+10M)\delta P(dx)\\
    & = \int_{\|x\| \leq R} \sum_{l=1}^p (w_l^\beta+\frac{\lambda}{p^2}|w_l|)(2x_l+10M)\delta P(dx)\\
    & +\int_{\|x\| > R} \sum_{l=1}^p (w_l^\beta+\frac{\lambda}{p^2}|w_l|)(2x_l+10M)\delta P(dx)\\
    & \leq \int_{\|x\| \leq R} \sum_{l=1}^p (w_l^\beta+\frac{\lambda}{p^2}|w_l|)(2R+10M)\delta P(dx)\\
    &  +\sum_{l=1}^p (w_l^\beta+\frac{\lambda}{p^2}|w_l|)\delta\int_{\|x\| > R}(2x_l+10M) P(dx).
    \end{aligned}
\end{equation}
\endgroup
}%
The last term can be made smaller than $\epsilon/2$ if $R$ is chosen large enough. The first term can be made less than $\epsilon/2$ if $\delta$ is chosen sufficiently small. Similarly one can show that $\Phi(\mathcal{W},B,P)-\Phi(\mathcal{W},A,P)<\epsilon$.  Hence the result.
\end{proof}
\subsection{Proof of Lemma \ref{lem2}}
\label{lem2 proof}
\begin{proof}
Let, $\mathcal{W},\mathcal{W}'\in \Gamma_k$ such that $\|\mathcal{W}-\mathcal{W}'\|<\delta$.Take $R>5M$. Thus,
{\small
\begingroup
\allowdisplaybreaks
    \begin{align*}
        & \Phi(\mathcal{W},A,P)-\Phi(\mathcal{W}',A,P)\\
        &=\int min_{a \in A} \sum_{l=1}^p (w_l^\beta+\frac{\lambda}{p^2}|w_l|)(x_l-a_l)^2 P(dx)\\
        &- \int min_{a \in A} \sum_{l=1}^p ({w'_l}^\beta+\frac{\lambda}{p^2}|w_l'|)(x_l - a_l)^2 P(dx)\\
        &\leq \int \sum_{a \in A} \sum_{l=1}^p ({w_l}^\beta-{w'_l}^\beta+\frac{\lambda}{p^2}(|w_l|-|w_l'|))(x_l-a_l)^2 P(dx)\\
        &=\int_{\|x\|\leq R} \sum_{a \in A} \sum_{l=1}^p ({w_l}^\beta-{w'_l}^\beta+\frac{\lambda}{p^2}(|w_l|-|w_l'|))(x_l-a_l)^2 P(dx)\\
        & + \int_{\|x\|> R} \sum_{a \in A} \sum_{l=1}^p (w_l^\beta-{w'_l}^\beta+\frac{\lambda}{p^2}(|w_l|-|w_l'|))(x_l-a_l)^2 P(dx)\\
        & \leq \int_{\|x\|\leq R} k \sum_{l=1}^p (w_l^\beta-{w'}^\beta_l+\frac{\lambda}{p^2}(|w_l|-|w_l'|))4R^2 P(dx)\\
        &+\int_{\|x\|> R} k \sum_{l=1}^p 2(b^\beta+\frac{\lambda}{p^2}b)(x_l-a_l)^2 P(dx).
    \end{align*}
\endgroup
}%
    The second term can be made smaller than $\epsilon/2$ if $R$ is chosen sufficiently large. Appealing to the continuity of the function $f(\mathcal{W})=\sum_{l=1}^p(w_l^\beta+\frac{\lambda}{p^2}|w_l|)$, the first term can be made smaller than $\epsilon/2$, if $\delta$ is chosen sufficiently small enough. Similarly one can show that, $\Phi(\mathcal{W}',A,P)-\Phi(\mathcal{W},A,P)<\epsilon$. Hence the result.
\end{proof}
\IEEEpeerreviewmaketitle

\ifCLASSOPTIONcaptionsoff
  \newpage
\fi

\bibliographystyle{IEEEtran}
\bibliography{bare_jrnl}
\begin{IEEEbiography}[{\includegraphics[width=1in,height=1.25in,clip,keepaspectratio]{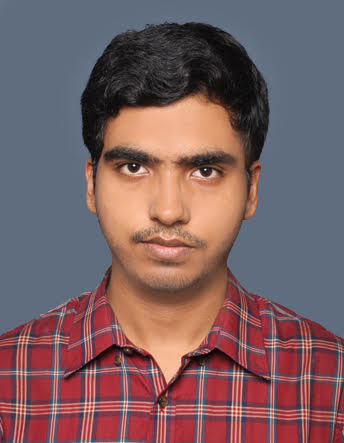}}]{Saptarshi Chakraborty}
received his B. Stat. degree in Statistics from the Indian Statistical Institute, Kolkata in 2018 and is currently pursuing his M. Stat. degree (in Statistics) at the same institute.  He was also a summer exchange student at the Big Data Summer Institute, University of Michigan, USA in 2018, where he worked on the application of Machine Learning algorithms on medical data. His current research interests are Statistical Learning (both supervised and unsupervised), Evolutionary Computing  and  Visual  Cryptography.
\end{IEEEbiography}
\begin{IEEEbiography}[{\includegraphics[width=1in,height=1.25in,clip,keepaspectratio]{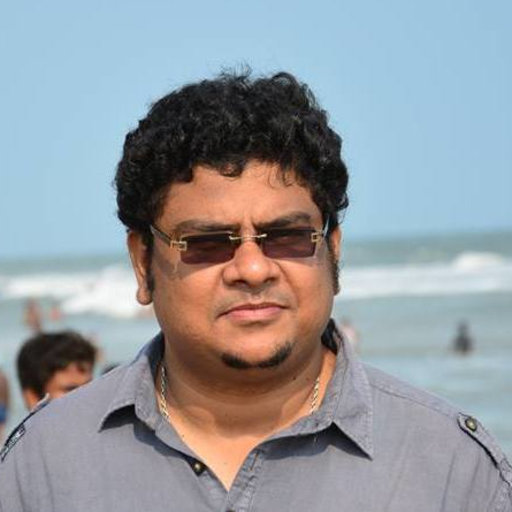}}]{Swagatam Das}
is currently serving as an associate professor at the Electronics and Communication Sciences Unit, Indian Statistical Institute,
Kolkata, India. He has published more than 250
research articles in peer-reviewed journals and
international conferences. He is the founding coeditor-in-chief of Swarm and Evolutionary Computation, an international journal from Elsevier.
Dr. Das has 16,000+ Google Scholar citations
and an H-index of 60 till date. He is also the recipient of the 2015 Thomson Reuters Research Excellence India Citation Award as the highest cited researcher from India in Engineering and Computer Science category between 2010 to 2014.
\end{IEEEbiography}




\end{document}